\documentclass[acmlarge]{acmart}
\makeatletter
\newcommand{\myconfshort}{\acmConference@shortname}
\newcommand{\myconffull}{\acmConference@name}
\newcommand{\myconfdate}{\acmConference@date}
\newcommand{\myconfloc}{\acmConference@venue}
\AtBeginDocument{
  \fancypagestyle{firstpagestyle}{
    \fancyhead{}%
    \fancyfoot[C]{}%
  }
  \fancyhf{}
  \fancyhead[LO]{\@headfootfont\shorttitle}%
  \fancyhead[RE]{\@headfootfont\@shortauthors}%
  \fancyhead[LE]{\@headfootfont\footnotesize \myconfshort, \myconfdate, \myconfloc}%
  \fancyhead[RO]{\@headfootfont\footnotesize \myconfshort, \myconfdate, \myconfloc}%
  \fancyfoot[C]{}%
}
\makeatother
\acmBooktitle{\conffull\@ (\confshort), \confdate, \confloc}

\AtBeginDocument{\settopmatter{printacmref=true}}
\AtBeginDocument{%
  }

\usepackage{amsmath,amsfonts,bm}

\def\eqref#1{equation~\ref{#1}}

\def\1{\bm{1}}

\DeclareMathAlphabet{\mathsfit}{\encodingdefault}{\sfdefault}{m}{sl}
\SetMathAlphabet{\mathsfit}{bold}{\encodingdefault}{\sfdefault}{bx}{n}

\usepackage[T1]{fontenc}
\usepackage{graphicx}
\usepackage{xcolor}
\usepackage{pifont}
\usepackage{mathtools}
\usepackage{caption}
\usepackage{subcaption}
\usepackage{dsfont} 
\usepackage{algorithm}
\usepackage{algpseudocode}
 
 \newtheorem{problem}{Problem}
 \newtheorem{claim}{Claim}
\algtext*{EndIf}
\algtext*{EndWhile}
\algtext*{EndFunction}

\usepackage{booktabs}
\usepackage[table]{xcolor}
\usepackage{colortbl}

\usepackage{amsmath}

\newcommand{\EE}{\mathbb E}
\newcommand{\NN}{\mathbb N}
\newcommand{\pNN}{\mathbb N^+}
\newcommand{\BN}{\mathbb B}
\newcommand{\PP}{\mathbb P}
\newcommand{\RR}{\mathbb R}

\newcommand{\prob}{\mathbb{P}}
\newcommand{\expe}{\mathbb{E}}

\newcommand{\indi}{\mathbf{1}}

\newcommand{\Enf}{\mathfrak{S}}

\newcommand{\RN}{\mathbb{R}}

\newcommand{\fairtarget}{\mu^*} 
\newcommand{\fairpivot}{\kappa} 

\newcommand{\probviolation}{\mathcal{P}}
\newcommand{\violation}{\mathcal{V}}
\newcommand{\expectedviolation}{\mathcal{E}}

\newcommand{\safe}{\mathcal{S}}
\newcommand{\live}{\mathcal{L}}
\newcommand{\fair}{\varphi}

\copyrightyear{2026}
\acmYear{2026}
\setcopyright{cc}
\setcctype{by}
\acmConference[FAccT '26]{The 2026 ACM Conference on Fairness, Accountability, and Transparency}{June 25--28, 2026}{Montreal, QC, Canada}
\acmBooktitle{The 2026 ACM Conference on Fairness, Accountability, and Transparency (FAccT '26), June 25--28, 2026, Montreal, QC, Canada}
\acmDOI{10.1145/3805689.3806807}
\acmISBN{979-8-4007-2596-8/2026/06}
\begin{document}

\title{Energy Shields for Fairness}


\author{Filip Cano}
\affiliation{%
  \institution{Institute of Science and Technology Austria}
  \city{Klosterneuburg}
  \country{Austria}}
\email{filip.cano@ist.ac.at}

\author{Thomas A. Henzinger}
\affiliation{%
  \institution{Institute of Science and Technology Austria}
  \city{Klosterneuburg}
  \country{Austria}}
\email{tah@ist.ac.at}

\author{Konstantin Kueffner}
\affiliation{%
  \institution{Institute of Science and Technology Austria}
  \city{Klosterneuburg}
  \country{Austria}}
\email{konstantin.kueffner@ist.ac.at}


\begin{abstract}
  Runtime fairness is not a one-time constraint but a dynamic property evaluated over a sequence of decisions.
  To ensure fairness at runtime, it is necessary to account for past decisions, information neglected by conventional, static classifiers.
  Traditional fairness shields enforce runtime fairness abruptly,
  by intervening \emph{deterministically} whenever a sequence of decisions violates the target for a running fairness measure. This motivates our \emph{main conceptual contribution: \textbf{energy shields}.}
  An energy shield is a novel, lightweight, adaptive controller that monitors a sequence of decisions and intervenes \emph{probabilistically} to ensure runtime fairness smoothly, by utilizing physics-inspired energy functions to nudge the sequence toward fairness:
  the more unfair the decisions, the stronger the nudging force becomes. This makes energy shields the \emph{\textbf{first}} fairness shields to provide both \emph{short-term safety and long-term liveness guarantees}.
  Safety ensures that the running fairness measure stays within a running target interval with high probability,
  and liveness ensures that the limit of the fairness measure lies within the limit target interval. 
  Intuitively, the short-term specifies the tolerated fairness values and the long-term specifies the desired fairness values. 
  We also provide a synthesis procedure for constructing the least intrusive energy shield for a given target specification, and demonstrate its efficiency experimentally. 
  We evaluate our energy shields against existing fairness shields through the lens of short- and long-term fairness.
\end{abstract}

\maketitle

\section{Introduction}
Algorithmic decision-making is ubiquitous in modern life, from hiring and lending to online advertising.
In these settings, binary decisions, e.g., approving or denying a loan or displaying one ad over another, are often made sequentially~\cite{liu2018delayed,d2020fairness}.
Much research has focused on designing \emph{statically fair} algorithms, which ensure fairness in expectation over a fixed distribution~\cite{caton2020fairness}.
This guarantees that in the long term the decision sequence will be fair (if the distribution remains static).
However, this perspective fails to address unfair behavior that arises in the short term from the cumulative history of decisions, which can be arbitrarily biased even for a fair classifier~\cite{cano2025algorithmic}.
We start to illustrate our work using a simplified binary online advertisement setting, which we generalize later.

\begin{example}
\label{ex:company_1}
Consider a static online advertising setting where two companies bid for advertising space on an online platform. 
The company $A$ places a bid of value $c_A$ and the company $B$ places a bid of value $c_B$. At each point in time $t\in \NN$, a user accesses the platform and will be shown an ad. The choice of ad $X_t$ is made probabilistically $X_t\sim \mathrm{Bernoulli}(p)$ by a decision maker, if the realization $x_t$ evaluates to $1$, the ad from company $A$ will be selected.  
Intuitively, the decision maker is fair if $p$ matches the bid ratio,
i.e., $p\in c_A/(c_A+c_B)\pm \varepsilon$ for $\varepsilon>0$.
In a sequential setting, a sequence $x_1, \dots, x_t$ of ad placements is fair at time $t$, if the empirical running average decision $\mu(x_1, \dots, x_t)=\frac{1}{t}\sum_{i=1}^t x_i$ lies within a target interval centered around the bid ratio $c_A/(c_A+c_B)\pm \varepsilon$ for $\varepsilon>0$. This translates into a short-term objective where the average is required to remain in the interval most of the time with high-probability, and a long-term objective where the limit average is required to converge to a point in the target interval.
\end{example}


\paragraph{Running vs.\ limit fairness: safety and liveness.}
Ex.~\ref{ex:company_1} illustrates two \emph{fairness} properties defined w.r.t.\ a running and a limit target interval.
The short-term property requires the running target interval to be met by the running average at a finite time,
and the long-term property requires the limit target to be met by the limit average of the decision sequence. Intuitively, the running target expresses what fairness values we tolerate in the short-term, and the limit target expresses what fairness values we desire in the long-term.
This distinction matches safety-liveness classification of properties studied in formal verification~\cite{baier2008principles}.
The violation of a safety property can be determined at a finite time,
thus matching our short-term requirement for fairness, called \emph{running fairness}.
The satisfaction of a liveness property can be determined in the limit only,
thus matching our long-term requirement for fairness, called \emph{limit fairness}.
Our objective is to ensure that a decision sequence meets given target intervals for both running and limit fairness.

\paragraph{Traditional vs.\ energy shielding.}
To enforce dynamic fairness requirements, \emph{runtime shielding} has emerged as a promising approach~\cite{cano2025fairness}.
Originating in formal verification, a shield is a lightweight controller that monitors a decision sequence and intervenes minimally to correct decisions in order to enforce desired fairness requirements~\cite{alshiekh2018safe}.
Unfortunately, existing fairness shields act deterministically, alternating abruptly between full intervention and no intervention,
to enforce highly restrictive short-term fairness requirements,
and they fail to provide long-term fairness guarantees altogether~\cite{cano2025fairness}.
We introduce probabilistic shielding mechanisms that have short- and long-term guarantees.

\begin{example}[Ex.~\ref{ex:company_1} cont.]
    \label{ex:company_2}
    Suppose both companies bid equally, resulting in a target ratio of $0.5$ and we allow a tolerance of $0.1$. A shield will observe the sequence of the decision maker one decision at a time. After each new observation, the shield must decide, based on the history, whether to keep or flip the original decision. For example, if the average decision risks exceeding $0.6$ and the current decision is $1$. A naive classical shield would deterministically flip the decision to remain within the target interval, thus switching from being completely idle to being maximally invasive. An energy shield would
    randomly flip the decision based on a chosen enforcement probability, thus it can increase its intervention strength with the unfairness of the past decision average.
\end{example}

\paragraph{Group fairness.}
Group fairness metrics, such as demographic parity, equal opportunity, equalized odds, and others, quantify the fairness between demographic groups in binary decision-making by comparing the group conditional acceptance probabilities~\cite{mehrabi2021survey}. For example, demographic parity ensures that the ratio of positive decisions is independent of the group.
The motivational Example~\ref{ex:company_1} considers a setting where decisions are mutually exclusive, i.e., accepting company $A$ implies rejecting company $B$. 
In this setting, the empirical demographic parity degenerates to the arithmetic mean. In Section~\ref{sec:group_fairness} we demonstrate that our energy shields generalize to the  general setting, as motivated below.
\begin{example}[Ex.~\ref{ex:company_2} cont.]
\label{ex:company_4}
As in Example~\ref{ex:company_1} the decision maker chooses between company $A$ or company $B$. Afterwards, a second decision maker places the chosen ad in a regular ($S=R$) or a premium ($S=P$) ad space.
This setting admits a second fairness measure between companies based on the premium ad placement rate, the empirical equivalent of demographic parity, i.e., 
$\prob(S=P \mid \text{Company}=A)- \prob(S=P \mid \text{Company}=B)$.
\end{example}

\paragraph{Generalizations.}
The motivational Example~\ref{ex:company_1} considers a static setting where the companies bid a constant amount at every point in time. The setting of this example is ideal for explaining the main concepts in energy shielding, and will be the main focus in Section~\ref{sec:setting} to \ref{sec:long-term}. 
In Section~\ref{sec:generalizations} we study two generalizations that may be useful for realistic settings: unknown and dynamic settings.
In the \emph{unknown} setting, the decision probability is fixed, but not communicated to the shield. The shield can estimate it from previous decisions and adapt the enforcement effect to the current parameter estimation.
In the \emph{dynamic} setting, the decision probability may be subject to change over time (for example, when the decision maker is retrained). Here, energy shields remain useful, although with weaker formal guarantees.

\paragraph{Contributions}
Our main \emph{conceptual contribution} is the introduction of \emph{\textbf{energy shields}}, a probabilistic shielding framework inspired by physics-based energy functions to ensure both running and limit fairness requirements. 
An energy shield smoothly nudges a sequence of decisions toward fairness by assigning higher ``energy'' to unfair sequences and intervening, proportional to the energy, to enter a lower-energy state. As a consequence, our \emph{energy shields are the \textbf{first} fairness shields capable of providing short- and long-term fairness guarantees}
To demonstrate that energy shields provide \textbf{short-term fairness} guarantees,
we provide exponentially decaying tail bounds on the probability and expected value of running fairness violations with respect to a given target interval. 
To demonstrate that energy shields provide \textbf{long-term fairness} guarantees,
we characterize the limit behavior of the shielded process, deriving conditions on the energy function to ensure convergence of the fairness measure to a given target. 
Additionally, we quantify the long-run expected cost of intervention and prove that steeper energy functions yield fewer violations, an important monotonicity property.
We exploit the monotonicity property of energy functions to propose a \textbf{synthesis procedure} that combines binary search and dynamic programming with tail bounds to find the least intrusive energy shield satisfying the desired running and limit fairness properties. We validate the effectiveness and efficiency of the synthesis procedure experimentally. Moreover, we benchmark our energy shields against the seminal fairness shields by \cite{cano2025fairness}, empirically supporting the claim that energy shields are the first shields to provide both short- and long-term guarantees.
As the bulk of our contribution is theoretical, we focus on presenting the results and leave the proofs in the appendix.

\section{Related Work}
Most algorithmic fairness research focuses on fairness in a static setting. This includes: measures for the fairness of a decision maker at the level of groups~\cite{feldman2015certifying,hardt2016equality} and individuals~\cite{dwork2012fairness};
pre-, in-, and post-processing techniques to synthesize fair decision maker~\cite{hardt2016equality,gordaliza2019obtaining,zafar2019fairness,agarwal2018reductions,wen2021algorithms};
verification techniques to check whether static decision makers are fair~\cite{albarghouthi2017fairsquare,bastani2019probabilistic,sun2021probabilistic,ghosh2020justicia,meyer2021certifying,li2023certifying}. 
Among the existing techniques for static fairness, our shields could be classified as a post-processing technique. The key difference is that those methods modify decision-makers \emph{once} before deployment, whereas our work addresses fairness \emph{during} deployment via runtime intervention.

We are not the first to be concerned with algorithmic fairness over a sequence of decisions~\cite{pmlr-v235-alamdari24a,cano2025algorithmic}. A large body of work focuses on detecting unfair behavior at runtime, both for individual~\cite{gupta2025monitoring} and for group fairness~\cite{albarghouthi2019fairness,henzinger2023dynamic,baumeister2025stream}. Beyond detection, \cite{cano2025fairness} is the only work 
enforcing fairness at runtime. Their shields adopt the sequential fairness definition from Parand et al.~\cite{pmlr-v235-alamdari24a}, ensuring that a sequence of decisions will be fair with probability 1 at predefined periodic intervals. 
Our shields soften this condition, providing high-probability short-term guarantees and are the first to provide limit guarantees.
Although the monitoring and enforcement of fairness has only recently emerged as a topic of interest, classical runtime monitoring and enforcement have long been studied in the runtime verification community, with monitors~\cite{stoller2011runtime,faymonville2017real,maler2004monitoring,donze2010robust,bartocci2018specification,baier2003ctmc} and shields~\cite{carr2023safe,alshiekh2018safe,cordoba2023safety} developed for Linear Temporal Logic specifications. Shielding has also been explored in probabilistic settings~\cite{jansenconcur2020,probshieldsIJCAI2023}, using probabilistic model checking techniques~\cite{katoen2016probabilistic}. 
We build on stochastic approximation results~\cite{borkar2008stochastic,karandikar2024convergence} to design probabilistic shields with almost sure convergence guarantees.

\section{Setting}
\label{sec:setting}
In this section, we formally introduce the setting, the fairness properties, and the notion of a shield.

\paragraph{Decision process.}
We model the setting in Ex.~\ref{ex:company_1} using a single coin with decision probability $p\in [0,1]$.
At each point in time $t\in \NN$ the coin is tossed, resulting in a decision $x_t\in \BN$, where $\BN=\{0,1\}$, which is the realization of the random variable $X_t\sim \mathrm{Bernoulli}(p)$. Combined they generate the decision process $X=(X_t)_{t\in \NN}$ of i.i.d. Bernoulli($p$) random variables. 
A realization $x=(x_t)_{t\in \NN}\in \{0,1\}^{\omega}$ of $X$ is an infinite sequence of binary values.

\paragraph{Fairness.}
We are interested in measuring the fairness of the process defined above. The measure we use is called average outcome fairness~\cite{cano2025algorithmic}. Formally, 
given an infinite sequence of binary decisions $x\in \{0,1\}^{\omega}$, e.g., a realization of the decision process, we measure the fairness of a finite prefix $x_{1:t} = (x_1,\dots, x_t)$ as its average decision, denoted by $\mu(x_{1:t})$, or $\mu_t$ if clear from the context.
If it exists, $\mu(x)$ denotes the fairness measure of the realization in the limit:
\begin{equation}
    \mu(x_{1:t})  = \frac{1}{t} \sum_{i=1}^t x_i, \qquad
    \text{and} \qquad \mu(x) = \lim_{t\to\infty} \mu(x_{1:t}).
\end{equation}
A \emph{fairness target} is a tuple $\fair=(\tau, \safe, \live)$ consisting of a burn-in time $\tau\in \NN$, a running target $\safe\subseteq [0,1]$, and a limit target $\live\subseteq \safe \subseteq [0,1]$. 
The fairness target specifies the acceptable fairness measure values at every finite time greater than the burn-in $\tau$ and at the limit. For the target to be satisfiable we require that the burn-in time $\tau$ is sufficiently large to account for initial high variance (see Section~\ref{sec:conclusion} for further discussion on how to choose the burn-in time parameter). 
Given a fairness target, an infinite sequence $x\in \BN^{\omega}$ satisfies:
\begin{align*}
    &\text{- \emph{point fairness} at time $t\in\NN$, if the fairness measure is in the running target at $t$, i.e., $\mu(x_{1:t})\in \safe $};\\
    &\text{- \emph{running fairness}, if point fairness is always satisfied after the burn-in $\tau$, i.e., $\forall t\geq \tau \colon \mu(x_{1:t})\in \safe $;}\\
    &\text{- \emph{limit fairness}, if the fairness measure is in the limit target in the limit, i.e.,  $\lim\nolimits_{t\to \infty} \mu(x_{1:t}) \in \live$;}\\
    &\text{- \emph{fairness}, if both running fairness and limit fairness are satisfied.}
\end{align*}
Intuitively, the running target expresses what fairness values are tolerated in the short-term, and the limit target expresses what fairness values are desired in the long-term. 

\begin{example}[Ex.~\ref{ex:company_1} cont.]
    \label{ex:company_fairness_target1}
    Assume the bid ratio is $0.5$. 
    In the long-term we require a fairness measure of $0.5$, and in the short-term we accept a tolerance of $0.1$ after some burn-in $\tau$. 
    The corresponding fairness target is $\fair=(\tau, [0.4,0.6], \{0.5\})$.
    Note, the decision $X_t$ follows a Bernoulli distribution, thus the fairness measure  $\mu(X_{1:t})$ follows a binomial distribution scaled by $1/t$. 
    Since $(1/t)Bin(t,p) \xrightarrow{t\to\infty}p$, 
    limit fairness requires $p=0.5$. The probability of satisfying point fairness at $t$ is $ \prob[\mu(X_{1:t}) \in \safe] = \prob[Bin(t,0.5)\in t(0.5\pm 0.1)]\approx\sum_{i =\lfloor 0.4t\rfloor}^{\lceil 0.6t\rceil} \binom{t}{i}p^{i}(1-p)^{t-i}$.
\end{example}

\paragraph{Shielding.}
As illustrated in Ex.~\ref{ex:company_fairness_target1}, without control of $p$, the only possible intervention on the process 
is to overwrite individual decisions. This is called shielding. 
A deterministic shield for a decision process is a program with the power to flip the decisions made at runtime. 
Formally, a deterministic shield $\Enf\colon (\BN\times \BN)^* \times \BN \to  \BN$ uses the history of decisions $x_1, \dots,  x_{t} \in \BN^{t}$ at time $t\in \NN$ and the history of intervention $y_1, \dots, y_{t-1}\in \BN^{t-1}$ 
to compute the next intervention $y_t$.
The intervention indicates whether the decision $x_t$ is flipped, i.e., if $y_t=1$ then the shielded decision is $z_t=1-x_t$, otherwise the shielded decision is $z_t=x_t$.
The shield should aid the satisfaction of the fairness target, evaluated over the sequence shielded decisions $z=(z_t)_{t\in \NN}$ with as little interference as possible. This is quantified by the average interference cost $\nu_t\coloneqq \sum_{i=1}^t y_i/t$ over a intervention sequence $y_1, \dots, y_t$.
\begin{example}[Ex.~\ref{ex:company_fairness_target1} cont.]
    \label{ex:shield}
    A trivial shield guaranteeing running and limit fairness, ensures that the company ads alternate, i.e., the shielded decision sequence is $(01)^{\omega}$, thus $\mu(z_{1:2t})=0.5$ for all $t$.
    A less restrictive shield that satisfies running fairness, but not the limit fairness, is one that only interferes if the point fairness is about to be violated. We call this shield the \emph{naive shield}.
    For example, assume at $t=100$ we have $\mu(z_{1:t})=40/100$, if $x_{t+1}=0$, then the fairness measure would be $40/101 < 0.4$, so the shield enforces, i.e., $z_{t+1}=1$.
\end{example}

\paragraph{Probabilistic shields.}
Deterministic shields act aggressively and abruptly once the fairness measure is at risk of leaving the target and remain idle most of the time. 
This leads to two regimes, one where the shield has full control over
the decision and one where it has none. 
We introduce probabilistic shields, i.e., shields in which interventions are executed probabilistically, allowing for a much more gentle approach to shielding. Formally, a probabilistic shield is a function $\Enf\colon (\BN\times \BN)^* \times \BN \to  [0,1]$ mapping
a history of decisions $x_1, \dots,  x_{t} \in \BN^{t}$ at time $t\in \NN$ and the history of interventions $(y_1, \dots, y_{t-1})$ into an nudging probability $q_t=\Enf(x_1,y_1, \dots, y_{t-1}, x_t)$, defining the distribution from which the next intervention is sampled, i.e., $Y_{t}\sim \mathrm{Bernoulli}(q_t)$.

At each time step $t$, we distinguish three binary random variables: 
$X_t$, the original decision produced by the environment; 
$Y_t$, the intervention indicator chosen by the shield;
and $Z_t$, the released shielded decision. 
$Y_t=0$ means that the shield keeps the original decision, while $Y_t=1$ means that it flips it. Therefore, it holds that
\[
Z_t= (1-X_t)\cdot Y_t + X_t\cdot (1-Y_t).
\]

\paragraph{Problem definition.}
We consider four processes: the decision process $X=(X_t)_{t\in \NN}$,
the intervention process $Y=(Y_t)_{t\in \NN}$,
the shielded decision process $Z=(Z_t)_{t\in \NN}$, and the shielded fairness process $M=(M_t)_{t\in \NN}$.
The processes are defined at every time $t\in \NN$ as follows:
\begin{align*}
    &X_t\sim \mathrm{Bernoulli}(p), \quad Y_t\sim  \mathrm{Bernoulli}(\Enf(X_1,Y_1, \dots, Y_{t-1}, X_t)),\;  \\
    & Z_t = (1-X_t)\cdot Y_t + X_t\cdot (1-Y_t),  \quad \text{and } \quad M_t = \mu(Z_1, \dots, Z_t).
\end{align*}
Intuitively, the shielded fairness process, i.e., the fairness measure evaluated over the shielded decision process, should satisfy a given fairness target $\fair=(\tau, \safe, \live) $ with an error probability less than $\delta\in (0,1)$, i.e., 
\begin{align*}
    \prob(\forall t\geq \tau \colon M_t\in \safe)\geq 1-\delta, \quad\text{and}\quad  \prob\left( \lim_{t\to \infty} M_t \in \live\right)\geq 1-\delta \quad \text{for $\delta\in (0,1)$}.
\end{align*}
Because both the environment and the shield are stochastic, the short-term specification can only be guaranteed with high probability. We use $\delta\in (0,1)$ to denote the tolerated failure probability, equivalently a confidence level of $1-\delta$. 
Smaller $\delta$ yields stronger guarantees, but typically leads to more conservative shields. Like the burn-in time $\tau$, the failure tolerance $\delta$ is an application-dependent design parameter.

\section{Energy-based Shields}
\label{sec:energy_shields}
We introduce energy shields, a family of physics inspired probabilistic shields. The shields acts by computing an energy state for the process and nudging the system toward lower-energy configurations.
The energy state is given by an energy function and the current fairness measure.

\paragraph{Energy function.}
An energy function is bowl-shaped with minimum at its pivot point.
\begin{definition}
\label{def:energyfunc}
    An \emph{energy function} with \emph{pivoting point} $\fairpivot\in [0,1]$ is a function $\zeta\colon [0,1]\to[0,1]$ satisfying: (i) $\zeta$ is continuously differentiable, i.e.,
        $\zeta$ is differentiable, and $\zeta'$ is continuous; (ii) $\zeta'(c) \leq 0$ for $ c < \fairpivot$, and $\zeta'(c)\geq 0$ for $c>\fairpivot$; (iii) $\zeta(\fairpivot) = \zeta'(\fairpivot) = 0$, and $\zeta(c) > 0$ for $c\in \{0,1\}$.
\end{definition}
\begin{example}
    \label{ex:energy_functions}
    Two families of energy functions are even-polynomials and exponential energy functions (see Fig.~\ref{fig:behavior}) defined, as  $\zeta^{\mathrm{Pol}}_{\kappa, \alpha,\beta} (x) = \alpha |x-\kappa|^{\beta}$ and  $\zeta^{\mathrm{Exp}}_{\kappa, \rho,\sigma}(x) = 
\rho( 1-e^{-\sigma (x-\kappa)^2} )$, where  $\kappa\in (0,1)$, 
    $\beta\in (1, \infty)$, $\alpha\in (0,1/\max(\kappa, 1-\kappa)^{\beta}$,
    $\sigma\in(0, \infty)$, and $\rho\in (0, 1/(1-e^{-\sigma(\min\{\kappa, 1-\kappa\})^2}))$.
    The parameter ranges ensure that the energy function does not exceed $1$ without clipping.
\end{example}

\paragraph{Energy shield.}
An energy shield is defined w.r.t.\ an energy function $\zeta\colon [0,1]\to [0,1]$ with pivoting point $\fairpivot\in [0,1]$.
The pivot point determines the favored decision, while the energy function determines the nudging probability. Formally, assume we have observed the decisions $x_1, \dots, x_t$ and accumulated the interventions $y_1, \dots, y_{t-1}$, which determined the shielded decisions $z_1, \dots, z_{t-1}$ at time $t\in \NN$. 
Then if $\mu_{t-1} \leq \fairpivot$,
the shield accepts $x_t=1$
and flips $x_t=0$ with probability $\zeta(\mu_{t-1})$,  and if $\mu_{t-1} > \fairpivot$, the shield accepts $x_t=0$, and flips $x_t=1$ with probability $\zeta(\mu_{t-1})$.
This determines the distribution of the next shielded decision $Z_t$ and fairness value $M_t$.

\begin{claim}[Shielded decision process]
    \label{claim:process}
    A decision process $Z$ shielded by an energy shield forms a sequence of Bernoulli random variables with evolving bias, i.e, $Z_t\sim \mathrm{Bernoulli}(p_t)$. 
The biases are defined recursively as $p_1=p$ and $p_{t+1} = f(\mu_t)$ for a given history $z_1, \dots, z_t$, where 
\begin{equation}
\label{eq:def_of_f}
f(\mu) = \begin{cases}
 p + (1-p)\zeta(\mu) & 
 \mbox{ if }\: \mu \leq \fairpivot, \\
 p\cdot \left(1-\zeta(\mu)\right) & \mbox{ if }\: \mu > \fairpivot
    \end{cases} .
\end{equation}
Moreover, the resulting shielded fairness process update can be written as 
\begin{equation}
\label{eq:process-update}
   M_{t} = \mu_{t-1} + \frac{1}{t}(Z_{t} - \mu_{t-1}) \quad \text{(with $\mu_{0}=0$)}.
\end{equation}
\end{claim}

Intuitively, Eq.~\ref{eq:def_of_f} shows that the decision maker pulls the fairness value $\mu$ toward $p$, while the shield exerts an opposing pull toward $\fairpivot$. Stronger energy values amplify this effect, shifting the process further to $\fairpivot$.
Eq.~\ref{eq:process-update} makes this dynamic explicit: the sequence of $\mu_t$'s evolves as a stochastic approximation process drifting in the direction of $\EE[Z_{t+1}|\mu_t] - \mu_t = f(\mu_t) - \mu_t$. At a convergence point, this drift should vanish, which occurs at a fixpoint of $f$, i.e., a value of $\mu$ satisfying $f(\mu) = \mu$.

\section{Short-term: Safety Guarantees}
\label{sec:short-term}
For the short-term running fairness property we require the shielded fairness process to stay within the running target $\safe\subseteq [0,1]$ at all finite times after the burn-in $\tau$ (see Ex.~\ref{ex:company_fairness_target1}). 
To prove that our shield satisfies running fairness with high-probability we develop upper bounds on the probability and the expected number of point fairness violations over an interval.
For a target $\safe = [L,U]$, we set the burn-in to
$\tau_\safe =  4/\min(|L-\fairtarget|, |U-\fairtarget|)$.

\begin{definition}
\label{def:fairness}
Let $M$ be the shielded fairness process generated by the decision probability $p$ and an energy shield with energy function $\zeta$. For a running target interval $\safe\subseteq [0,1]$, we define two violation measures, the probability of violating point fairness $\probviolation_{\safe}$ and the expected number of point fairness violations $\expectedviolation_{\safe}$ over a time interval $[T, T']\subseteq \NN\cup \{\infty\}$ as
\begin{equation}
    \probviolation_{\safe}(M_{T:T'}) =
    \prob\left[\exists t\in [T,T'] \colon M_t \not\in \safe\right]  \quad \text{and} \quad  \expectedviolation_{\safe}(M_{T:T'})= \sum_{t=T}^{T'}\expe\left[ \indi[ M_t \notin \safe] \right].
\end{equation}
\end{definition}

\subsection{Upper-bounds}
\label{sec:absolute-bounds}

Our main result is a bound for the probability of the shielded fairness process violating point fairness at time $T$.
This result follows from constructing a martingale that upper-bounds the distance of the process to its converging point $\fairtarget$, and using Azuma-Hoeffding's inequality on said martingale.
\begin{theorem}
    \label{thm:tail_bound}
    Let $\safe=[L,U]$ such that $\fairpivot, p\in \safe$, and $K= 1/32$.
    Then for every $ t\geq \tau_\safe$ 
    \begin{equation}
       \prob[M_t\notin \safe]\ \leq  \exp\left(-Kt|L-\fairtarget|^2\right) + \exp\left(-Kt|U-\fairtarget|^2\right).
    \end{equation}
\end{theorem}

\paragraph{Property bounds.}
Utilizing Theorem~\ref{thm:tail_bound}, the expected number of violations and the probability of a single violation 
in an interval, follows from Boole's inequality.
  \begin{corollary}
    \label{cor:tail_bound}
   Let $r_L=e^{-K|L-\fairtarget|^2}$, $r_U = e^{-K|U-\fairtarget|^2}$. 
   For every $T,T'\in \NN$ s.t. $\tau_\safe \leq T < T'$, then
    \begin{equation*}
        \expectedviolation_{\safe}(M_{T:T'}) \leq \sum_{t=T}^{T'}( r_L^t + r_U^t),
          \qquad \mbox{ and }
         \qquad 
         \expectedviolation_{\safe}(M_{T:\infty}) \leq \frac{r_L^{T}}{1-r_L} + \frac{r_U^{T}}{1-r_U}.
    \end{equation*}
    Hence, it follows that $\probviolation_{\safe}(M_{T:T'})  \leq \expectedviolation_{\safe}(M_{T:T'}) $ for $T'\in \NN\cup\{\infty\}$ s.t.\ $T'\geq T$.
\end{corollary}

\subsection{Monotonicity with Respect to the Energy Function}
\label{sec:monotonicity}
In this section, we define a partial order among energy functions given by their steepness, and show how our violation properties are monotone with respect to steepness.
\begin{definition}
    Let $\zeta_1,\zeta_2$ be two energy functions with a common pivot point $\kappa$.
    We say that $\zeta_1$ is \emph{steeper} than $\zeta_2$, and denote it by $\zeta_1 \succeq \zeta_2$ if $\zeta_1(y)\geq \zeta_2(y)$ for all $y\in[0,1]$. 
\end{definition}
Intuitively, a steeper energy function (see Fig.~\ref{fig:behavior}) constitutes a more ``aggressive'' shield, leading to fewer safety violations, both in probability and in expectation.
\begin{theorem}
\label{thm:monotoniciy}
    Let $\zeta_1\succeq \zeta_2$ be two energy functions, and let
    $M^{\zeta_1}$ and $M^{\zeta_2}$ be the shielded fairness process 
    generated by enforcing the decision process with 
    $\zeta_1$ and $\zeta_2$, respectively.
    For all $T,T'\in\mathbb N$ such that $\tau_\safe\leq T<T'$ we have 
    \[
    \expectedviolation_{\safe}\left(M^{\zeta_1}_{T:T'}\right) \leq
    \expectedviolation_{\safe}\left(M^{\zeta_2}_{T:T'}\right), \qquad
    \mbox{ and } \qquad
        \probviolation_{\safe}\left(M^{\zeta_1}_{T:T'}\right) \leq
    \probviolation_{\safe}\left(M^{\zeta_2}_{T:T'}\right).
    \]
\end{theorem}

\section{Long-term: Liveness Guarantees}
\label{sec:long-term}
For the long-term limit fairness property, we require the shielded fairness process to converge to a point in the limit target $\live\subseteq [0,1]$. 
To prove that our shield satisfies limit fairness, we identify the conditions on the energy function to ensure that the shielded fairness process converges to a target value $\fairtarget$ under the decision probability $p$ with an expected average interference cost of $|p - \fairtarget|$.

\paragraph{Main result.}
It is remarkable that both the convergence of $M_t$ and the expected cost depend \emph{only} on $p$ and the value of $\zeta$ at the target point $\fairtarget$ and not on the pivot point $\fairpivot$. 

\begin{theorem}
\label{thm:mainconvergence}
    Let $\fairtarget\in [0,1]$.
   Given the shielded decision process from Eq.~\ref{eq:process-update}, with bias 
    $p$
    and energy function 
    $\zeta$,
    the shielded fairness process $(M_t)_{t\in\NN}$ converges almost surely (a.s.) to $\fairtarget$ if and only if 
    \begin{equation}
    \label{eq:mainconvergenceeq}
    \zeta(\fairtarget) = 
    \begin{cases}
        (\fairtarget-p)/(1-p) & \mbox{ if } p < \fairtarget, \\
        (p-\fairtarget)/p & \mbox{ otherwise }
    \end{cases}.
    \end{equation}
    Furthermore, the expected interference cost $(\EE[\nu_t])_{t\in \NN}$ converges almost surely to $|\fairtarget-p|$.
\end{theorem}

\paragraph{Proof intuition.} We establish that the shielded fairness process $M=(M_t)_{t\in \NN}$ converges at the fixpoint of $f$.
First, we show $f$ has indeed a unique fixpoint located between $p$ and the pivot $\fairpivot$.

\begin{lemma}
\label{lem:unique-fixpoint-1d}
    The function $f\colon [0,1]\to [0,1]$ defined as in Eq.~\ref{eq:def_of_f} is continuously differentiable, and has a unique point $\fairtarget\in [0,1]$ such that $f(\fairtarget) = \fairtarget$.
    Furthermore, $\fairtarget$ sits between $p$ and $\fairpivot$.
\end{lemma}
Once we know $f$ has a unique fixpoint, we use stochastic approximation theory to prove that the shielded fairness process $M$ converges almost surely to the fixpoint.
\begin{lemma}
\label{lem:convergence_outcome_1d}
    The fairness process, defined in Eq.~\ref{eq:process-update}, converges a.s. to the unique fixpoint $\fairtarget$ of $f$, as defined in Eq.~\ref{eq:def_of_f}. The error $(M_t-\fairtarget)^2$ converges a.s. at the rate of $o(1/t^\lambda)$ for all $\lambda\in (0,1)$.
\end{lemma}
\begin{proof}[Proof sketch.]
    Let $g(x) = f(x)-x$, then we rewrite the update rule in Eq.~\ref{eq:process-update} as 
    \begin{equation}
    \label{eq:robbins-monro-form}
        M_{t+1} = M_t + \gamma_t\big( g(M_t) + \xi_{t+1} \big),
    \end{equation}
    where $\xi_t = Z_{t} - f(M_{t-1})$, and $\gamma_t = 1/t$.
    This is a classical Robbins-Monro form for stochastic approximation~\cite{borkar2008stochastic}. From stochastic approximation theory we know that, under certain regularity conditions, $M_t$ from Eq.~\ref{eq:robbins-monro-form} approximates the zero value of $g$, which is fixpoint of $f$.
    We use~\cite{karandikar2024convergence} to bound the convergence rate.
\end{proof}
Equation~\ref{eq:mainconvergenceeq} in Theorem~\ref{thm:mainconvergence} follows from Lemma~\ref{lem:convergence_outcome_1d} by noticing that the fixpoint $\fairtarget$ has to satisfy Eq.~\ref{eq:aux1}.
Since the fixpoint lies between $p$ and $\fairpivot$, 
the branch in Eq.~\ref{eq:aux1} is chosen based on $\fairtarget\leq p$.
\begin{equation}
\label{eq:aux1}
    \fairtarget = \begin{cases}
        p + (1-p)\cdot \zeta(\fairtarget) & 
 \mbox{ if }\: p \leq \fairpivot, \\
 p\cdot \left(1-\zeta(\fairtarget)\right) & \mbox{ if }\: p > \fairpivot
    \end{cases}.
\end{equation}
The expected cost of intervention at step $t+1$ is
the probability of 
not seeing the favorable decision ($1$ if the current average is below $\fairpivot$ and $0$ otherwise) and having an energy high enough to intervene. The expected intervention cost $h(\mu_t)$ converges $h(\mu_t)\to h(\fairtarget)$, because the fairness value converges $\mu_t \to \fairtarget$, where $h$ is defined as
\begin{equation}
\label{eq:def_of_h}
h(\mu) = \begin{cases}
    (1-p)\cdot \zeta(\mu) & \mbox{ if } \mu \leq \fairpivot\\
    p\cdot \zeta(\mu) & 
    \mbox{ otherwise }
\end{cases}    .
\end{equation}

\begin{lemma}
    \label{lem:cost}
    For the process described in Eq.~\ref{eq:process-update} ,
    the corresponding sequence of average interference $(\nu_t)_{t\in\NN}$ converges to $h(\fairtarget)$, 
    where $\fairtarget$ is the fixpoint of $f$ (Eq.~\ref{eq:def_of_f}) and 
    $h$ is as defined in Eq.~\ref{eq:def_of_h}.
\end{lemma}
Finally, Eq.~\ref{eq:mainconvergenceeq} and Lemma~\ref{lem:cost} imply that the expected intervention cost converges to $|p-\fairtarget|$.

\section{Energy Shields for Group Fairness}
\label{sec:group_fairness}
We apply our fairness shields to enforce the difference between two decision probabilities, which is necessary to enforce common group fairness metrics.
Group fairness metrics quantify the fairness between demographic groups in binary decision-making by comparing the group conditional acceptance probabilities~\cite{mehrabi2021survey}. For example, demographic parity ensures that the ratio of positive decisions is independent of the group.
\begin{equation}
    \label{eq:dem-fairness-def}
    \prob(X=1 \mid \text{Group}=A)- \prob(X=1 \mid \text{Group}=B) \in [-\varepsilon, +\varepsilon] \quad \text{with tolerance of $\varepsilon >0$}.
\end{equation}
The energy shields introduced so far focus on a setting where the decision maker must accept exactly one of the groups at each point in time (see  Sec.~\ref{sec:setting}). We generalize this in the following.

\paragraph{Two-group setting.}
We assume that at each point in time $t\in \pNN$ only one group $g_t\in \mathbb{G}=\{A,B\}$ is presented to the decision maker with probability $r_{g_t}\in(0,1)$, where $r_{A}+r_B=1$, which subsequently makes the decision $x_t\in \{0,1\}$ with group conditional decision probability $p_{g_t}\in (0,1)$. 
We use $d=p_A-p_B$ to denote the parameter that will have an analogous role as $p$ in the single-group setting.
The objective is to control the empirical group conditional acceptance rate, which is defined for every sequence of group-decision pairs $w_{1:t}=(g_1,x_1), \dots, (g_t,x_t)$ of length $t\in \pNN$ as
\begin{align*}
    \mu_A(w_{1:t})-\mu_B(w_{1:t}) \quad \text{where}\quad \mu_g(w_{1:t})\coloneqq \frac{\sum_{i=1}^t x_i \cdot \indi[g_i=g]}{\sum_{i=1}^t \indi[g_i=g]}\quad \text{for $g\in \mathbb{G}$}.
\end{align*}
In this setting a shield is a function 
$\Enf\colon (\mathbb G\times \mathbb B\times \mathbb B)^*\times (\mathbb G \times \mathbb B) \to [0,1]$ is a mapping from a group-decision-intervention history $(g_1,x_1, y_1),\dots, (g_{t-1},x_{t-1}, y_{t-1})$ and an input $(g_t,x_t)\in\mathbb G\times \mathbb B$ to a nudging probability $q_t$, defining the distribution from which the next intervention $y_t$ is sampled. 
Formally, we consider the group process $G = (G_t)_{t\in \mathbb N}$, the decision process $X = (X_t)_{t\in \mathbb N}$ the intervention process $Y = (Y_t)_{t\in\mathbb N}$, the shielded input process $Z = (Z_t)_{t\in \mathbb N}$, and the shielded fairness process $M=(M_t)_{t\in\mathbb N}$ defined as
\begin{align*}
&     
G_t \sim \mathrm{Bernoulli}(r_A), \quad 
X_t \sim \mathrm{Bernoulli}(p_{G_t}),
Y_t = \mathrm{Bernoulli}(\Enf(G_1, X_1,Y_1, \dots, Y_{t-1}, G_t, X_t)) \\
& Z_t = (1-X_t)\cdot Y_t + X_t\cdot (1-Y_t), \quad
M_t =  \mu_A(G_1, Z_1, \dots, G_t, Z_t) - \mu_B(G_1, Z_1, \dots, G_t, Z_t)
\end{align*}
Given a probability $\delta\in [0,1]$ and a fairness target $\fair=(\tau, \safe, \live) $, with $\tau\in \pNN$ and $\live\subseteq\safe\subseteq [-1,+1]$, the objective is to find a shield that guarantees:
$\prob(\forall t\geq \tau \colon M_t\in \safe)\geq 1-\delta$, and $ \prob( \lim_{t\to \infty} M_t \in \live)\geq 1-\delta$.

\paragraph{Implementation of the energy shield.}
We extend the domain of the energy function in Def.~\ref{def:energyfunc}  from $[0,1]$ to $[-1,+1]$, i.e.,  $\zeta\colon [-1,+1]\to[0,1]$. 
The shield monitors the evolution of the shielded fairness process $M$. At each time $t$, the current fairness value is $M_t$. As before, the shield favors the decision that reduces the distance between $M_t$ and $\kappa$. This is always possible because when the group is $A$ accepting increases and rejecting decreases the fairness measure. This is inverted for the group $B$. Moreover, the shield flips the original decision with probability $\zeta(M_t)$, if this decision increases the distance to $\kappa$.
This results in similar long-term guarantees as in Thm.~\ref{thm:mainconvergence}.

\begin{theorem}
    \label{thm:convergence2g}
    Let $\fairtarget\in [-1,+1]$. 
    The two-group shielded fairness process $(M_t)_{t\in\mathbb N}$ converges a.s. to $\fairtarget$ \emph{iff}
    \begin{equation}
    \label{eq:mainconvergenceeq2group}
    \zeta(\fairtarget) = 
    \begin{cases}
        (\fairtarget-(p_A-p_B))/(1-(p_A-p_B)) & \mbox{ if } p_A-p_B < \fairtarget, \\
        (p_A-p_B-\fairtarget)/(1+p_A-p_B) & \mbox{ otherwise. }
    \end{cases}
    \end{equation}
    The expected interference cost $(\EE[\nu_t])_{t\in \NN}$ converges to a value smaller than
    $|\fairtarget - (p_A-p_B)|$.
\end{theorem}

For the short-term guarantees, the results depend on a burn-in condition for both groups.
Using $N_{g,t} := \sum_{i=1}^t \indi[G_i=g]$,
we express the condition with the following event:
\[
E_T := \Big\{\forall t\ge T:\ N_{A,t}\ge \tfrac{r_A}{2}t\ \ \text{and}\ \ N_{B,t}\ge \tfrac{r_B}{2}t\Big\}.
\]

\begin{lemma}
\label{lem:group-counts}
For every $\eta\in(0,1)$, let $\tau_G(\eta):= \frac{8}{r_{\min}}\log\frac{4}{\eta}$. Then for all $T\geq \tau_G(\eta)$, it holds that $\PP(E_T)\ge 1-\eta$.
\end{lemma}

\begin{theorem}[2-group analogue of Thm.~\ref{thm:tail_bound}]
\label{thm:tail_bound_2g}
Fix an interval $I=[L,U]\subseteq[-1,1]$ such that $\kappa,d\in I$ and let $\eta\in(0,1)$.
Let $K=\min\{r_A,r_B\}/32$, then for all $t\geq \tau_G(\eta)$
\begin{equation}
\label{eq:tail_2g}
\PP(M_t\notin I)\ \le\ \eta\ +\ 
\exp\!\bigl(-K\,t\,|L-\fairtarget|^2\bigr)\ +\ 
\exp\!\bigl(-K\,t\,|U-\fairtarget|^2\bigr).
\end{equation}
\end{theorem}

\begin{theorem}[2-group analogue of Thm.~\ref{thm:monotoniciy}]
\label{thm:monotonicity_2g}
Let $\zeta_1\succeq \zeta_2$ be two energy functions on $[-1,1]$ with common minimum at $\kappa$.
Fix $I=[L,U]\subseteq[-1,1]$ and let $M^{\zeta_1}$, $M^{\zeta_2}$ be the corresponding 2-group shielded fairness processes.
Then for every $\eta\in(0,1)$, for all $T<T'$ with $T\ge \tau_G(\eta)$,
\[
\expectedviolation_I\!\left(M^{\zeta_1}_{T:T'}\mid E_T\right)\ \le\
\expectedviolation_I\!\left(M^{\zeta_2}_{T:T'}\mid E_T\right), \quad \mbox{and}\quad
\probviolation_I\!\left(M^{\zeta_1}_{T:T'}\mid E_T\right)\ \le\
\probviolation_I\!\left(M^{\zeta_2}_{T:T'}\mid E_T\right).
\]
\end{theorem}

\section{Shield Synthesis}
\label{sec:synthesis}
In this section, we state the energy shield synthesis problem and propose Alg.~\ref{alg:synthesis} as a solution.
The synthesis procedure is virtually identical for the simple and general group fairness setting introduced in Section~\ref{sec:setting} and Section~\ref{sec:group_fairness} respectively.

\paragraph{Problem instance.}
A problem instance $(\violation, \Xi, \varphi, \delta,\varepsilon, \bm{p})$ consists of: 
a violation measure $\violation\in\{\probviolation, \expectedviolation\}$,
a totally ordered set of energy functions $\Xi= \{\zeta_k\}_{k\in R}$ indexed by some interval $R\subset \RN$ where for all $i\leq j\in R$ we have $\zeta_i\preceq \zeta_j$, a fairness target $\fair=(\tau, \safe, \live)$, a violation threshold $\delta>0$, and an approximation tolerance $\varepsilon>0$; 
a tuple $\bm{p}$ of setting parameters.
The simple setting in Section~\ref{sec:setting} is described by decision probability $\bm{p}=(p)$ and requires energy functions with domain $[0,1]$.
The general setting in Section~\ref{sec:group_fairness} is described by the group appearance probability and the group specific decision probabilities $\bm{p}=(r_A, p_A, p_B)$ and requires energy function with domain $[-1,1]$. 
\begin{problem}
\label{problem:synthesis}
Given $(\violation, \Xi, \varphi, \delta,\varepsilon, \bm{p})$ find the least invasive energy function $\zeta \in \Xi$ satisfying the fairness constraints:\\ $(i)$ the shielded fairness process $M^{\zeta}$ with parameter $\bm{p}$ satisfies $\violation_{\safe}(M_{\tau:\infty}^{\zeta}) \leq \delta$ and $\lim_{t\to \infty} M_t^{\zeta} \in \live $ a.s., and\\
$(ii)$ $\zeta$ is $\epsilon$-minimal, i.e., if $\zeta^*\in \Xi$ is the smallest valid element, then 
$|\violation_{\safe}(M_{\tau:\infty}^{\zeta^*}) - \violation_\safe(M_{\tau:\infty}^{\zeta})|\leq \epsilon$.  
\end{problem}
We remark that if the violation measure is $\probviolation$, then running fairness should be satisfied with probability greater than $1-\delta$, and if it is $\expectedviolation$, then the total number of point fairness violations after $\tau$ should be bounded by $\delta$.
In Alg.~\ref{alg:synthesis} we show a synthesis procedure based on having a family of energy functions $\Xi = (\zeta_r)_{r\in R}$ that is indexed by some interval $R\subset\mathbb R$ and is monotonic with respect to the index.
In App.~\ref{sec:families_of_energy_funcs} we describe such a family $(\zeta^{\mathrm{Mon}}_{r; \bm{p}, \safe, \live})_{r\in (0,1)}$, which is defined w.r.t.\ $\bm{p}$ and a specification $\fair$, as a piecewise exponential and polynomial function.

\begin{algorithm}
    \caption{Shield synthesis}
    \label{alg:synthesis}
    \begin{algorithmic}[1]
        \Require problem instance $(\violation, \Xi, p, \varphi, \delta,\varepsilon)$.
          \State $l \gets \inf R; u \gets \sup R$ \Comment{ Set lower and upper bound for energy function w.r.t. $\prec$.}
        \State $T_{\mathrm{DP}} \gets \min \{ t \in \NN \mid \textsc{Bound}(t)\leq \varepsilon \}$ 
          \Comment{smallest $t$ s.t. prob. bound (Thrm.~\ref{thm:tail_bound}\&\ref{thm:tail_bound_2g}) satisfies tolerance}
          \If{$\textsc{Condition}(\zeta_u, T_{\mathrm{DP}})> \delta$}
            \Return \textsc{Fail}
            \Comment{The most strict energy function is not enough}
          \EndIf
        \While{$l \neq u$} 
            \State $m \gets  (l + u) / 2$  ; \;\; 
            $d \gets \textsc{Condition}(\zeta_{m},T_{\mathrm{DP}})$
            \If{$|d-\delta| < \epsilon$} \Return $\zeta_m$ \EndIf
            \If{$d\leq \delta$} $l \gets m  $
            \textbf{ else }
            $u \gets m $
            \EndIf
        \EndWhile
    \end{algorithmic}
\end{algorithm}

\paragraph{Algorithm.}
The algorithm exploits monotonicity of the energy-function family to identify the least steep function that satisfies the violation condition, yielding the least invasive shield. Inside the binary search, it approximates the violation measure $\violation$ via \textsc{Condition}, which combines dynamic programming with the tail bounds. Concretely, we split $[\tau,\infty)$ into a prefix $[\tau,T_{\mathrm{DP}}]$ and a tail $(T_{\mathrm{DP}},\infty)$. The threshold is chosen using probability bounds from Thm.~\ref{thm:tail_bound} or \ref{thm:tail_bound_2g}, for the simple and general setting, respectively.
For the simple setting, Cor.~\ref{cor:tail_bound} gives directly $\textsc{Bound}(t)\coloneqq\frac{r_-^{t}}{1-r_-} + \frac{r_+^{t}}{1-r_+}$.
For the general setting, in line with Thm.~\ref{thm:tail_bound_2g}, we need to add a term $\eta$ to the bound, indicating that $\prob(E_T)$ is large enough. A standard choice would be to set $\eta = \varepsilon/2$, and the rest of the bound to be smaller than $\varepsilon/2$.
For the prefix, we compute the exact violation measure with standard dynamic programming techniques.
For the tail term, we use again setting specific tail bounds from Thrm.~\ref{thm:tail_bound} or \ref{thm:tail_bound_2g}.
Since the bounds for a violation in the tail can be made arbitrarily small for large enough $T_{\mathrm{DP}}$, we use them to approximate the exact violation value with as much precision as required---a time-precision trade-off.

\section{Generalizations}
\label{sec:generalizations}
The assumption of a known and stationary decision probability $p$ made in Section~\ref{sec:setting}, might not hold in reality. We generalize to setting with unknown and/or changing decision probability. Specifically, we consider a setting where (i) ``$p$ is fixed and unknown'' and (ii) ``$(p_t)_{t\in\NN}$ changes and is unknown''. 
The setting  ``$(p_t)_{t\in\NN}$ changes and is known'' is less interesting as it reduces to the shield setting the decision bias to a point within a target.  

\paragraph{Setting (i): $p$ is fixed and unknown.}
This setting requires the estimation of the unknown bias $p$ using a consistent estimate $\hat{P}_t=\mu_t(X_1,\dots, X_t)$ at each point in time $t\in \pNN$, which is used to update the energy function instead of $p$. 
Because $\hat{P_t}$ converges almost surely to $p$, the energy function converges and enforces the long-term target.
\begin{theorem}
    Let $\fairtarget\in [0,1]$ be a fairness target.
    Let $(\zeta_q)_{q\in [0,1]}$ be a family of energy functions s.t. 
    \begin{equation}
    \label{eq:mainconvergenceeq-unknown}
    \zeta_q(\fairtarget) = 
    \begin{cases}
        (\fairtarget-q)/(1-q) & \mbox{ if } q < \fairtarget, \\
        (q-\fairtarget)/q & \mbox{ otherwise. }
    \end{cases}
    \end{equation}
    Then the shielded process that takes at each step the energy function $\zeta_{\hat p_t}$ converges a.s. to $\fairtarget$.
\end{theorem}

\paragraph{Setting (ii): $(p_t)_{t\in\NN}$ changes and is unknown.}
This setting makes it impossible for a non-trivial shield to enforce the convergence to a target point. However, an energy shield with sufficiently steep energy function ensures that the shielded process remains within a target interval almost surely. 
\begin{theorem}
    Let $\zeta\colon [0,1]\to[0,1]$ be an energy function, and $L,R$ 
    be points satisfying: $ L < \kappa$, $\zeta(L) = L$, $R > \kappa$, and $\zeta(R) = 1-R$.
    Then, for the shielded process $(M_t)_{t\in\mathbb N}$ it holds
    a.s.\ that: $\liminf_{t\to \infty} M_t \geq L$ and $ \limsup_{t\to\infty}M_t \leq R$.
\end{theorem}

\section{Experiments}
\label{sec:experiments}
We evaluate the impact of energy shields on the fairness across all settings, explore the time-precision trade-off in Alg.~\ref{alg:synthesis}, and compare against the existing fairness shields from Cano et al.~\cite{cano2025fairness} using MacBook with the M2 
chip. 

\begin{figure}[b]
        \centering
        \includegraphics[width=1\linewidth]{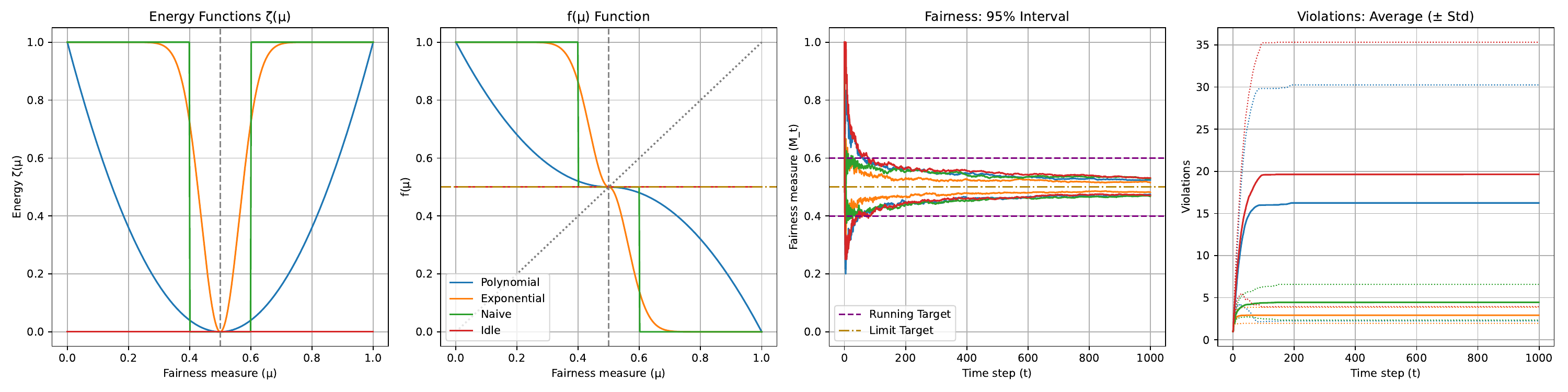}
        \includegraphics[width=1\linewidth]{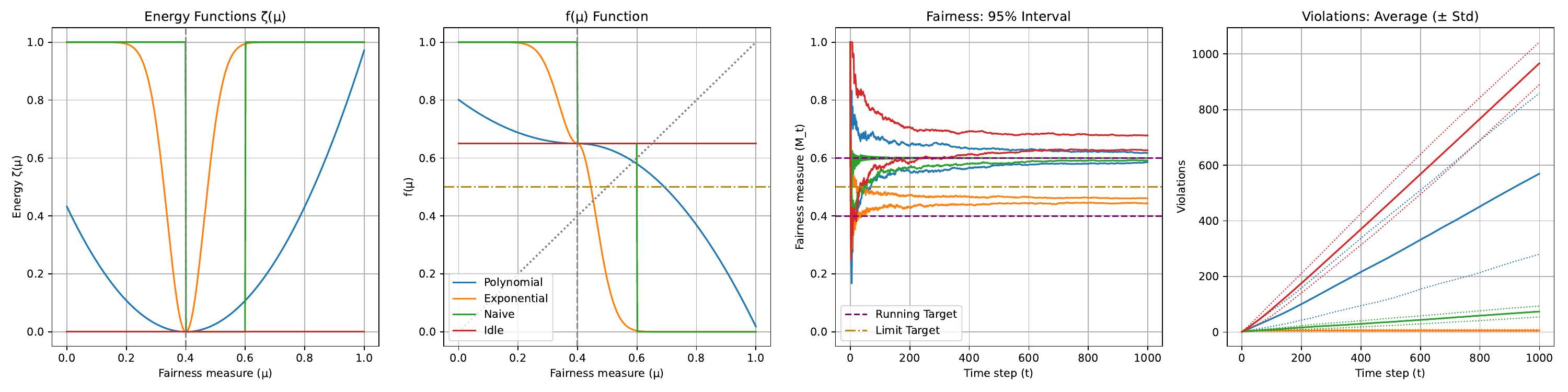}
        \caption{Simulated impact of the energy functions $\zeta$ for fairness target $\fair=(0,[0.4,0.6], \{0.5\})$, decision probability $p$.
        \textbf{Rows (R):} (R1) $p=0.5$ and $\zeta\in\{\zeta^{\mathrm{Naive}}, \zeta^{\mathrm{Idle}},  \zeta_{0.5, 4,2}^{\mathrm{Pol}}, \zeta_{0.5, 1,128}^{\mathrm{Exp}}\}$;
        (R2) considers $p=0.65$ and $\zeta\in \{\zeta^{\mathrm{Naive}}, \zeta^{\mathrm{Idle}},  \zeta_{0.4, 2.7,2}^{\mathrm{Pol}}, \zeta_{0.4, 1,128}^{\mathrm{Exp}}\}$.
        \textbf{Columns (C):}
        (C1) the behavior of the energy function for each fairness value.
        (C2) illustration of the characteristic functions $f$ and their respective fixpoints $f(\mu)=\mu$.
        (C3) 95\% confidence interval of fairness values at each time.
        (C4) Average point fairness violations with standard deviation at each time.
        The simulation results are averaged over $1000$ simulations for each $p$ and $\zeta$.}
        \label{fig:behavior}
    \end{figure}

\subsection{Impact of Energy Functions}
Here we investigate how different energy functions affect running and limit fairness, and what happens if we go beyond the assumption of a known and stationary decision probability $p$ (see Section~\ref{sec:generalizations}).

\paragraph{Energy functions.}
In Fig.~\ref{fig:behavior} we consider a fair decision maker with $p=0.5$ in R1, and a biased decision maker $p=0.65$ in R2. 
We set our fairness target $\fair=(0, [0.4, 0.6], \{0.5\})$, as indicated by the dotted lines in C3.
We deploy the shield with a polynomial $\zeta^{\mathrm{Pol}}$ and exponential $\zeta^{\mathrm{Exp}}$ energy functions (see Ex.~\ref{ex:energy_functions}), as well as an energy function that induces an idle shield $\zeta^{\mathrm{Idle}}$, which is always $0$, and an energy function that deterministically intervenes $\zeta^{\mathrm{Naive}}$, which is $0$ in the interior of the running target and $1$ elsewhere.
The energy function parametrization depends on the decision maker and is depicted in C1.
In C2 we can observe the effect of the energy function on the decision probability, where the fixpoint of the functions is indicated by the $45^\circ$ line. 
In C3 we depict a lower and upper bound containing $95\%$ of the simulated runs, and in C4 we depict the accumulated average violations and standard deviation.
We observe in C3 and C4 that if the fixpoint is contained within the running target, the number of violations decays quickly;
if the decision maker is biased, a pivot within a given target interval does not guarantee convergence of the fairness value to a point in the target;
and if the decision maker is fair, our shields remain useful by further reducing violations.

\paragraph{Generalizations.}
In Fig.~\ref{fig:generalizations} we present one experiment for each generalization for the energy functions $\zeta\in\{\zeta^{\mathrm{Naive}}, \zeta^{\mathrm{Idle}},  \zeta_{0.5, 4,2}^{\mathrm{Pol}}, \zeta_{0.5, 1,128}^{\mathrm{Exp}}\}$ and depict the $95\%$ interval of the simulated runs (similar to C3 in Fig.~\ref{fig:behavior}).
(C1) We consider the general group fairness setting with $\bm{p}=(r_A=0.8, p_A=0.7, p_B=0.4)$, and choose $\kappa$ such that $\mu^*=0.5$. We observe that both the polynomial and exponential energy shield satisfy the running and limit fairness guarantees in the general group setting, while the naive shield only satisfies the running fairness guarantees.
(C2) We consider the unknown decision probability setting with $p=0.65$. Here the $\kappa$ adapts to the estimated decision probability. We observe a slightly slower convergence for shields using the estimated decision probability compared to shields having access to the ground truth.
(C3) We consider the unknown and dynamic decision probability setting with $p_t$ following a sine curve centered around $0.65$ and $\kappa$ is set such that $\mu^*=0.5$ w.r.t. $0.65$. In this setting, our energy shields empirically aid in the satisfaction of the running and limit fairness goals.

\begin{figure}
    \centering
     \begin{subfigure}[b]{0.33\textwidth}
         \centering
         \includegraphics[width=\textwidth]{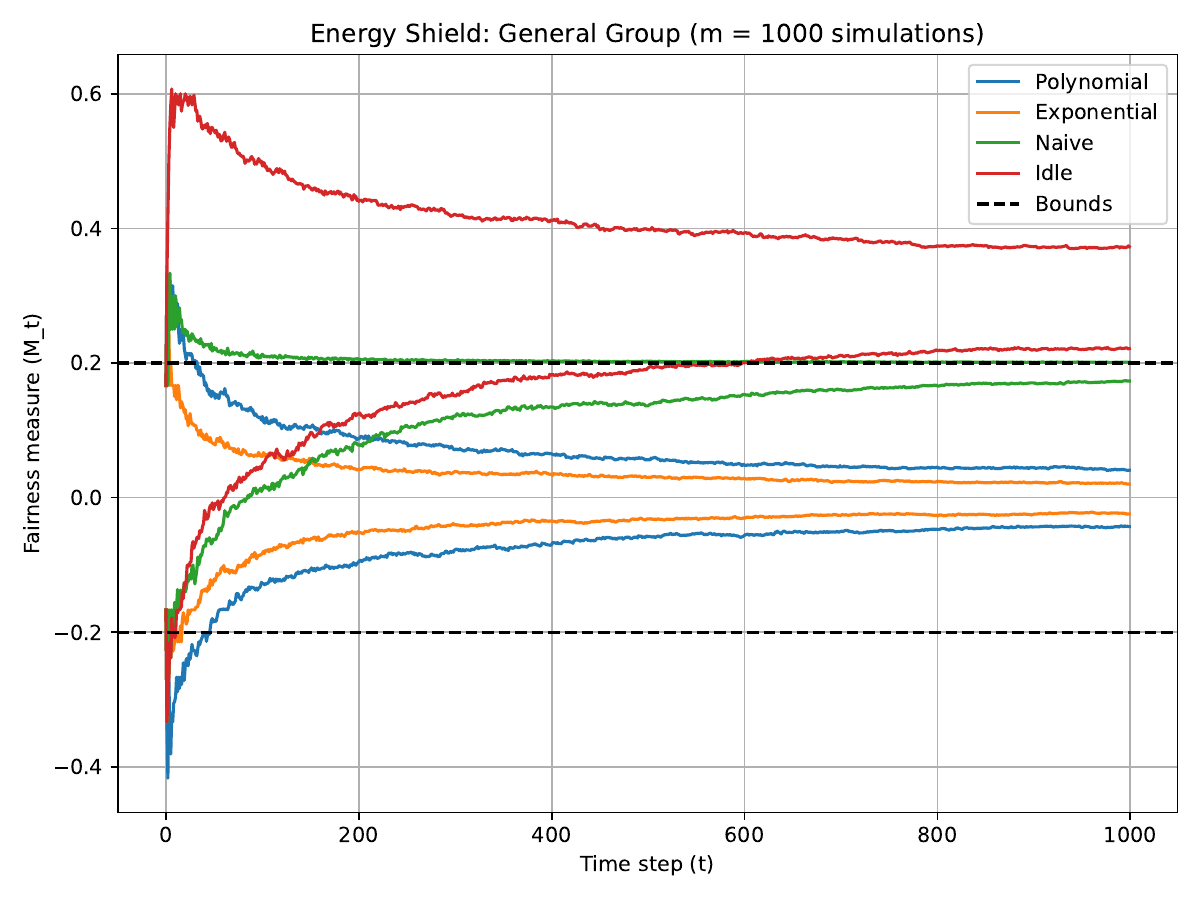}
     \end{subfigure}
     \hfill
     \begin{subfigure}[b]{0.33\textwidth}
         \centering
         \includegraphics[width=\textwidth]{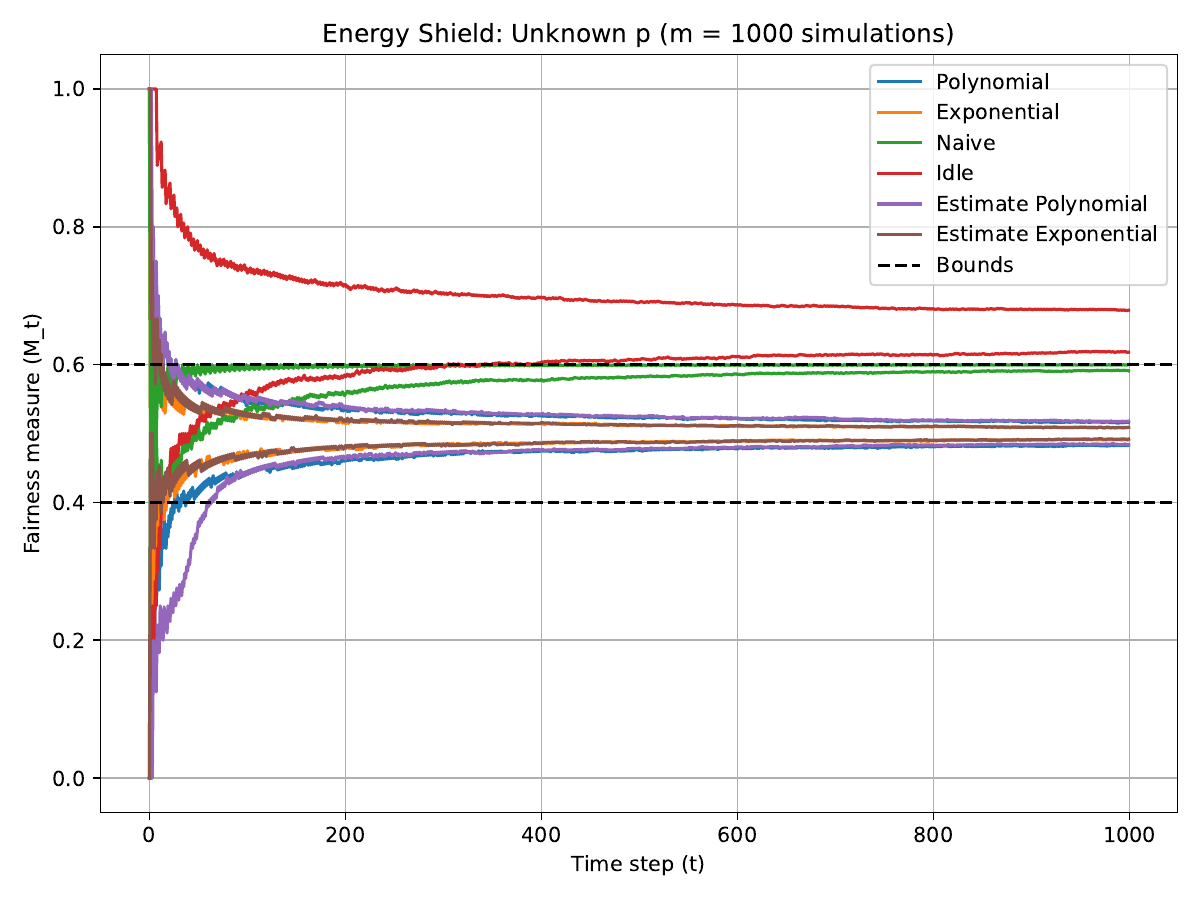}
     \end{subfigure}
     \hfill
     \begin{subfigure}[b]{0.33\textwidth}
         \centering
         \includegraphics[width=\textwidth]{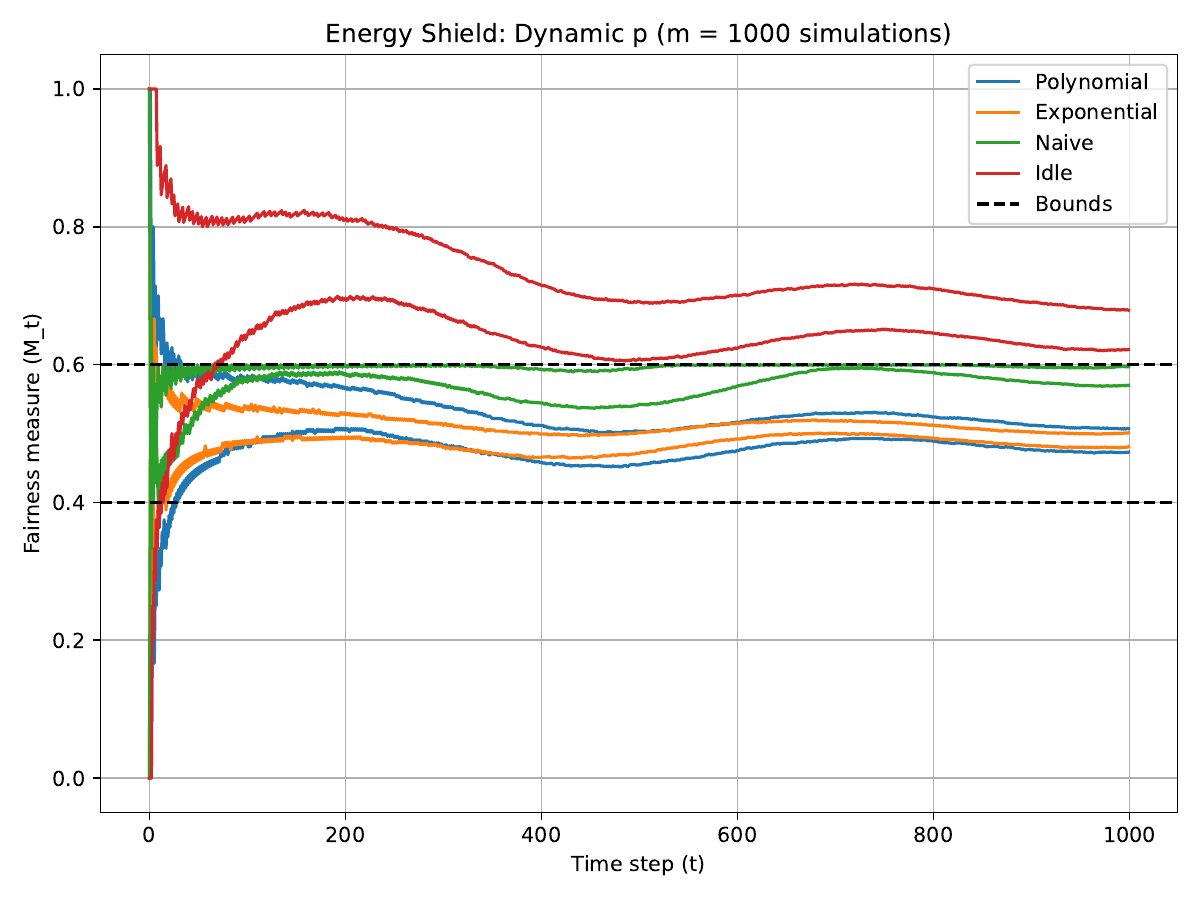}
     \end{subfigure}
    \caption{95\% confidence interval of fairness values at each time for the generalized settings. \textbf{Columns (C):} (C1) two-group group-fairness setting with $\bm{p}=(r_A=0.8,p_A=0.7,p_B=0.4)$ and $\kappa$ chosen such that $\mu^*=0.5$. (C2) unknown fixed decision probability ($p=0.65$) using an online estimate $\hat p_t$ compared to known $p$. (C3) unknown time-varying decision probability $p_t$ (sinusoidal around $0.65$).}
    \label{fig:generalizations}
\end{figure}

\subsection{Shield Synthesis}
Here we investigate the accuracy and efficiency of the dynamic-programming algorithm and the tail-bound approximation used in synthesis?
 
\paragraph{Approximation precision.}
The synthesis Algorithm~\ref{alg:synthesis} performs a violation probability approximation (VPA) by running dynamic programming (DP) up to a horizon $T_{\mathrm{DP}}$ and then utilizing the infinite horizon tail-bound from Sec.~\ref{sec:short-term}. Because the tail-bound is conservative this induces a precision-time trade-off, which is investigated in Fig.~\ref{fig:precision}. 
In C1, we increase the DP threshold $T_{\mathrm{DP}}$ from $0$ to $15000$ and record the precision of the obtained tail bound. 
In C2, we demonstrate the increase in computation time as the DP threshold $T_{\mathrm{DP}}$ increases.
In C3, we plot the error that is obtained between a direct computation and the tail bound. 
In C4, we demonstrate the time-precision trade-off directly, i.e., as we decrease the precision how much does the computation time decrease.
We observe that an extension of the DP horizon $T_{\mathrm{DP}}$ leads to an exponential gain in VPA precision, at the cost of a quadratic increase in computation time. This impacts the runtime of Alg.~\ref{alg:synthesis}, as higher precision requirements demand longer DP horizons.
The culprit is the looseness of the tail-bounds over the short horizon, as demonstrated by the gain in precision, if larger burn-ins are considered.

    \begin{figure}
  \centering
     \begin{subfigure}[b]{0.24\textwidth}
         \centering
         \includegraphics[width=\textwidth]{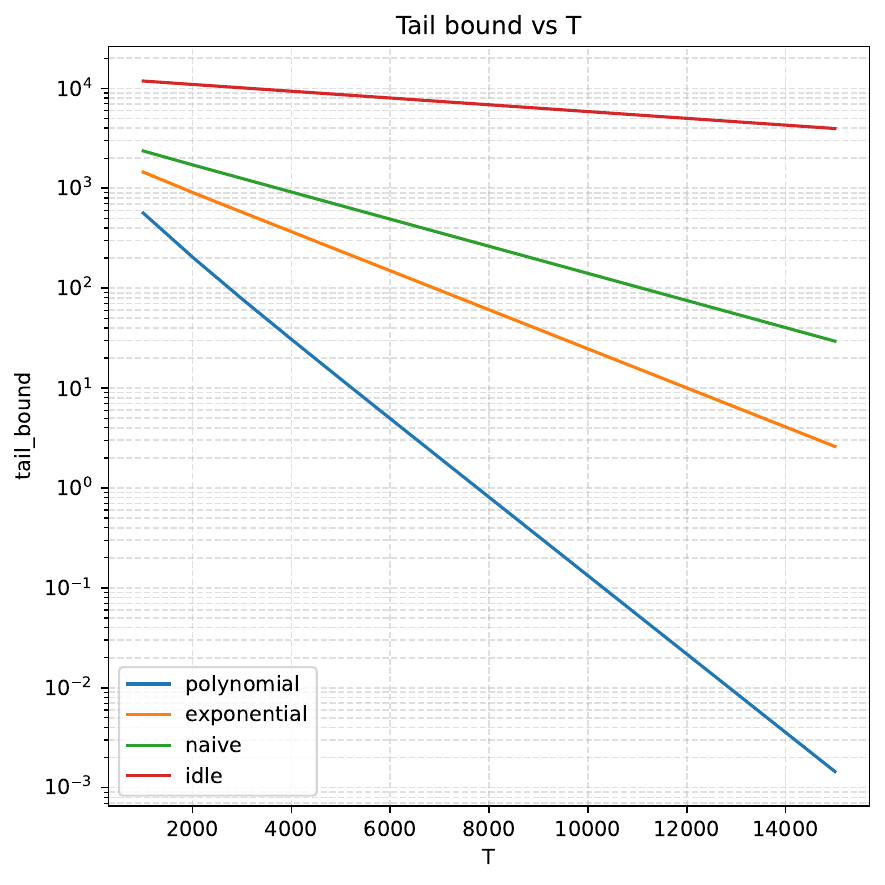}
     \end{subfigure}
     \hfill
    \begin{subfigure}[b]{0.24\textwidth}
         \centering
         \includegraphics[width=\textwidth]{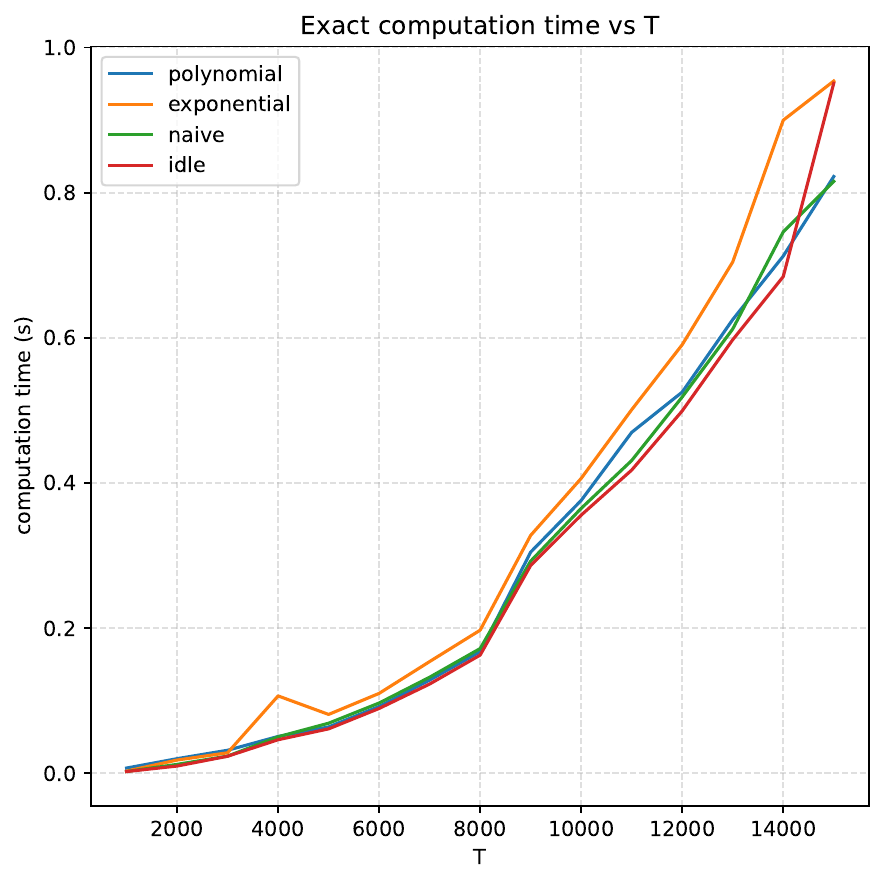}
     \end{subfigure}
     \hfill
    \begin{subfigure}[b]{0.24\textwidth}
         \centering
         \includegraphics[width=\textwidth]{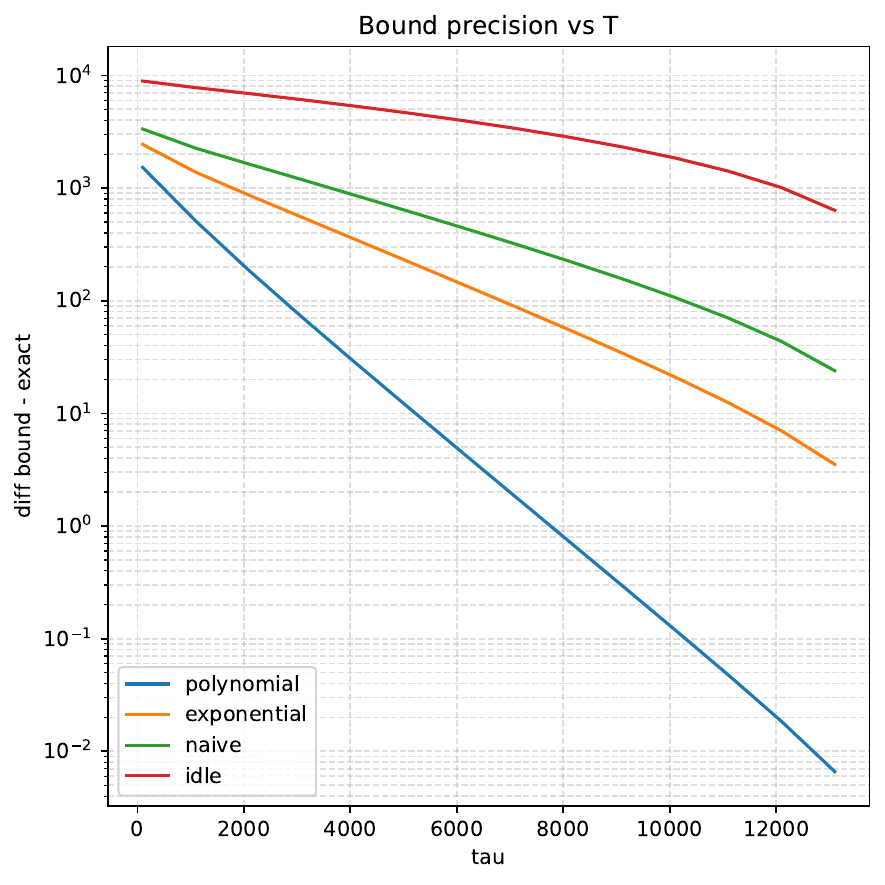}
     \end{subfigure}
     \hfill
     \begin{subfigure}[b]{0.24\textwidth}
         \centering
         \includegraphics[width=\textwidth]{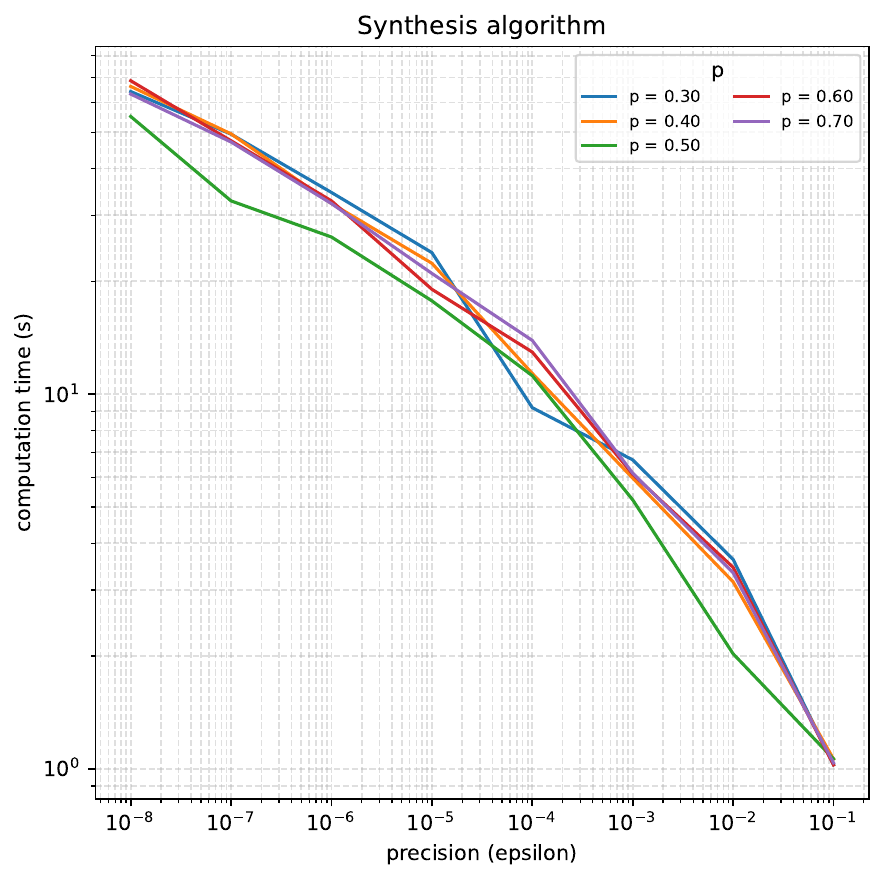}
     \end{subfigure}
        \caption{Time-precision trade-off in the violation probability approximation (VPA) of Alg.~\ref{alg:synthesis} with fairness target $\fair=(100,[0.3,0.7], \{0.45, 0.55\})$, decision probability $p=0.65$, and energy function $\zeta\in \{\zeta^{\mathrm{Naive}}, \zeta^{\mathrm{Idle}},  \zeta_{0.4, 2.7,2}^{\mathrm{Pol}}, \zeta_{0.4, 1,128}^{\mathrm{Exp}}\}$. 
        \textbf{Columns (C):}
        (C1) VPA with increasing dynamic programming (DP) threshold $T_{\mathrm{DP}}$.
        (C2) Computation time as $T_{\mathrm{DP}}$ increases.
        (C3) VPA precision, i.e., VPA with $T_{\mathrm{DP}}=0$ compared to $T_{\mathrm{DP}}=15000$, for increasing burn-ins $\tau$.
        (C4) Runtime of Alg.~\ref{alg:synthesis} as precision $\varepsilon$ increases, for different $p$ and $\delta= 0.1$.
        }
        \label{fig:precision}
    \end{figure}

\begin{figure}
          \centering
     \begin{subfigure}[b]{0.24\textwidth}
         \centering
         \includegraphics[width=\textwidth]{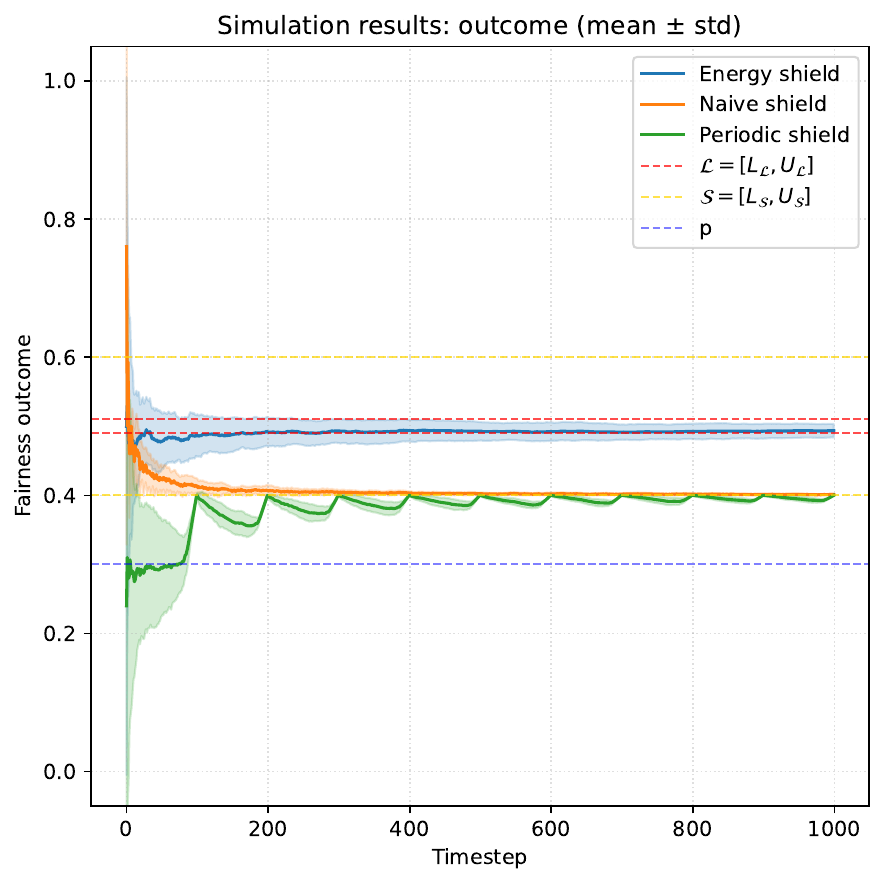}
     \end{subfigure}
     \hfill
     \begin{subfigure}[b]{0.24\textwidth}
         \centering
         \includegraphics[width=\textwidth]{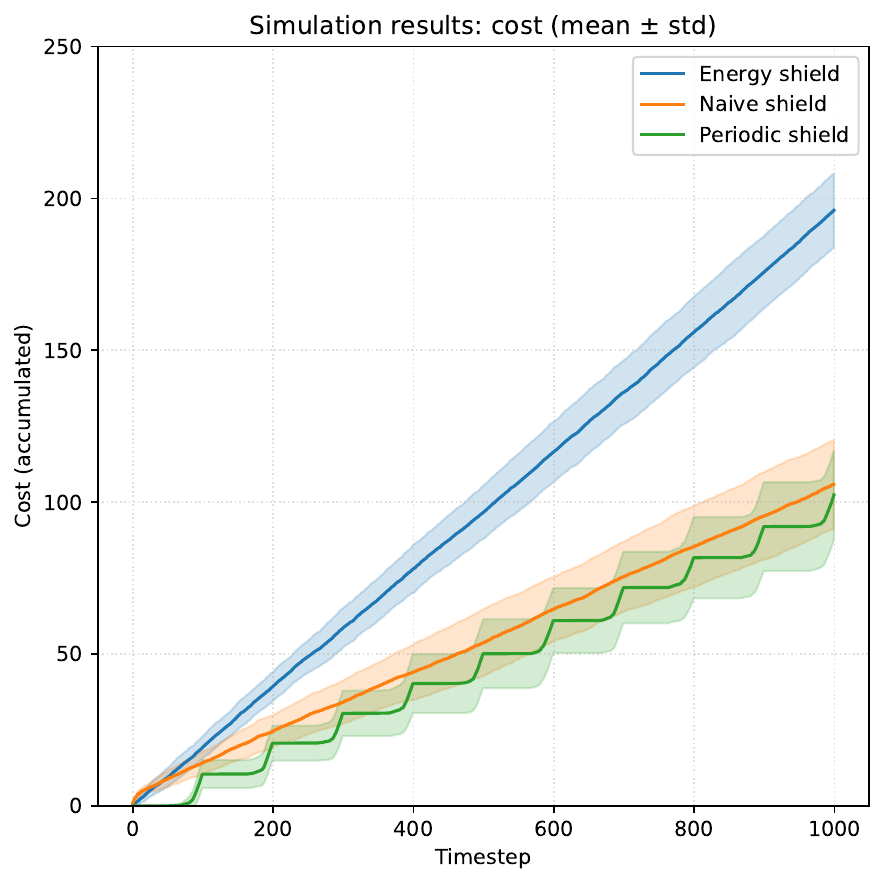}
     \end{subfigure}
     \hfill
     \begin{subfigure}[b]{0.24\textwidth}
         \centering
         \includegraphics[width=\textwidth]{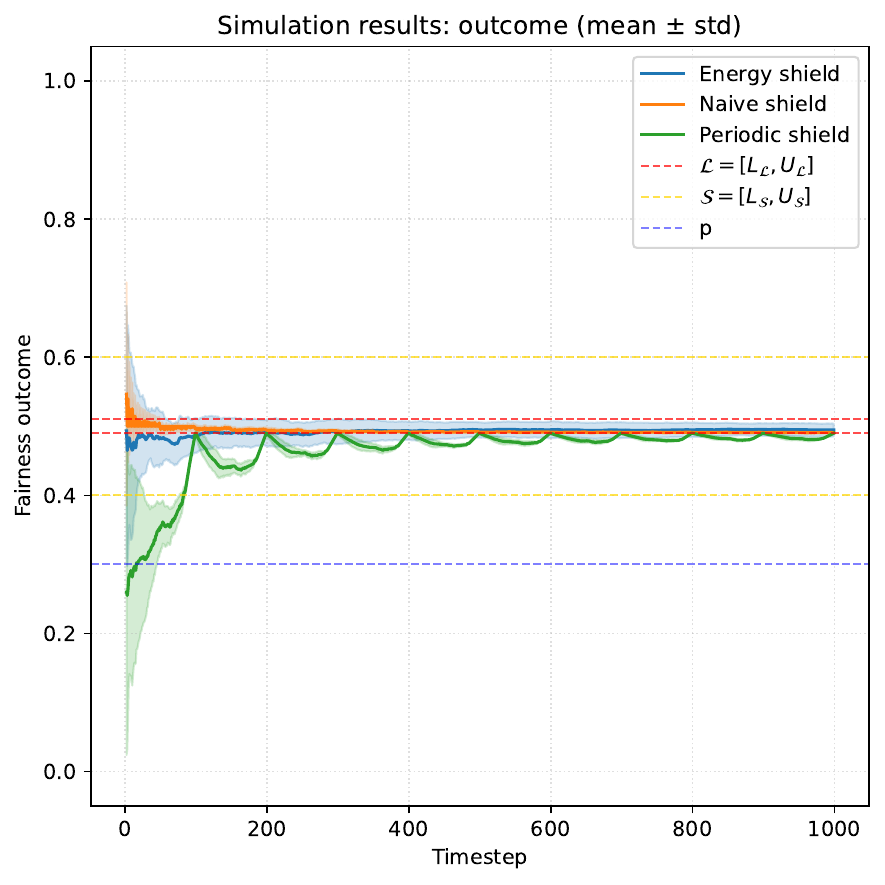}
     \end{subfigure}
     \hfill
     \begin{subfigure}[b]{0.24\textwidth}
         \centering
         \includegraphics[width=\textwidth]{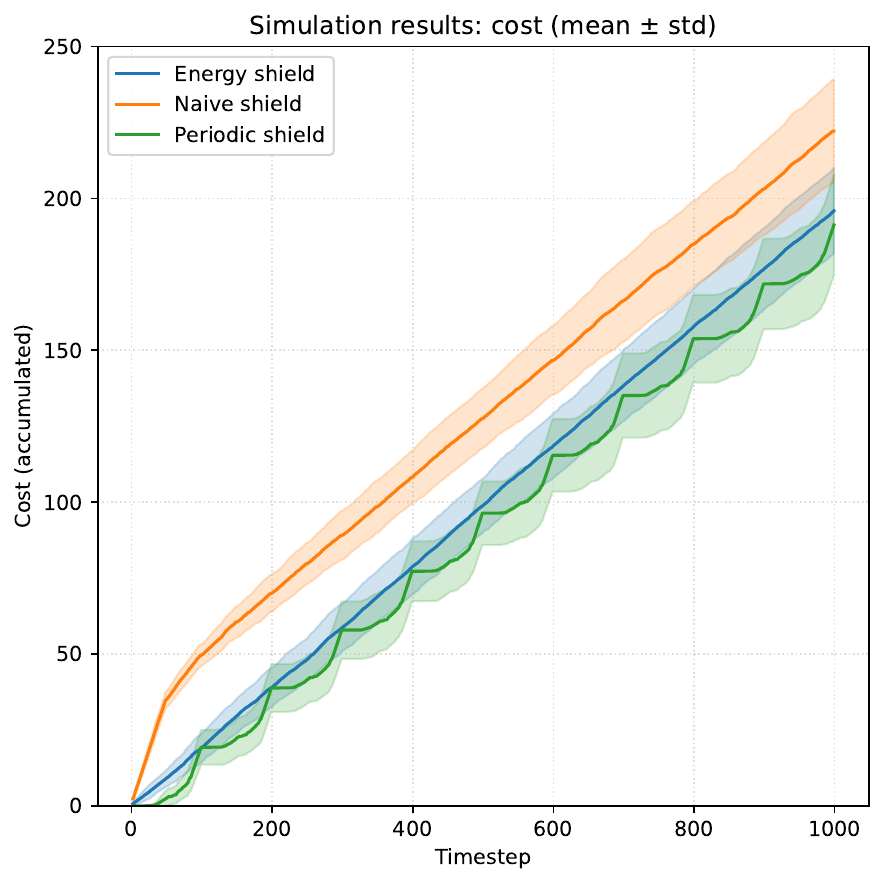}
     \end{subfigure}
        \caption{Energy shields compared against naive and periodic shields, for fairness target $\fair=(\tau=100, \safe =[0.4,0.6],\live =[0.49,0.51])$ and decision probability $p=0.3$.
        As energy function, we use $\zeta_{0.1;p,\safe,\live}^{\mathrm{Mon}}$. 
        \textbf{Columns (C):} 
        The naive and periodic shields are tuned for the: 
        (C1\&2) running target $[0.4,0.6]$;
        (C3\&4) limit target $[0.49,0.51]$.
        (C1\&C3) depict the fairness measure and (C2\&C4) the accumulated number of interventions. }
        \label{fig:experiments_simulation}
    \end{figure}

\subsection{Comparison with existing Fairness Shields}
 Here we investigate how energy shields compare with existing fairness shields in the simple one group and the general two group setting?

\paragraph{Comparison: Simple.}
In Fig.~\ref{fig:experiments_simulation}, we benchmark our energy shields against the naive shield from Ex.~\ref{ex:shield} and the periodic shields from~\cite{cano2025fairness} in the simple one-group setting of Section~\ref{sec:setting}.
The naive shield enforces running fairness by intervening whenever the next decision would leave the running target, while the periodic shield enforces point fairness at fixed periodic times with optimal expected cost, but requires heavier runtime computation.
We consider the fairness target
$\fair=(\tau=100,\safe=[0.4,0.6],\live=[0.49,0.51])$
and decision probability $p=0.3$.
As energy function, we use $\zeta_{0.1;p,\safe,\live}^{\mathrm{Mon}}$, tuned to push the fairness value toward $\fairtarget=0.5$.
The existing shields do not natively distinguish between running and limit fairness. 
Therefore, we evaluate them in two configurations: once with respect to the running target $\safe=[0.4,0.6]$, and once with respect to the stricter limit target $\live=[0.49,0.51]$.
We observe that, when tuned to the running target, the baseline shields keep the fairness value inside $\safe$, but tend to remain close to the boundary and do not converge to the limit target.
When tuned to the limit target, the baselines satisfy the stricter fairness requirement only by intervening much more aggressively; in particular, the naive shield becomes almost maximally invasive.
By contrast, the energy shield balances the two requirements: it keeps the fairness value within the running target while gradually steering it toward the limit target.
This is reflected in the intervention costs: running-target baselines intervene less, limit-target baselines intervene more, and the energy shield lies between these two extremes while achieving the desired long-term behaviour.

\begin{table}[t]
     \centering
     \begin{subfigure}[b]{0.35\textwidth}
         \centering
         \includegraphics[height=8.8cm, trim=0.26cm 0.3cm 0.3cm 0.3cm, clip]
    {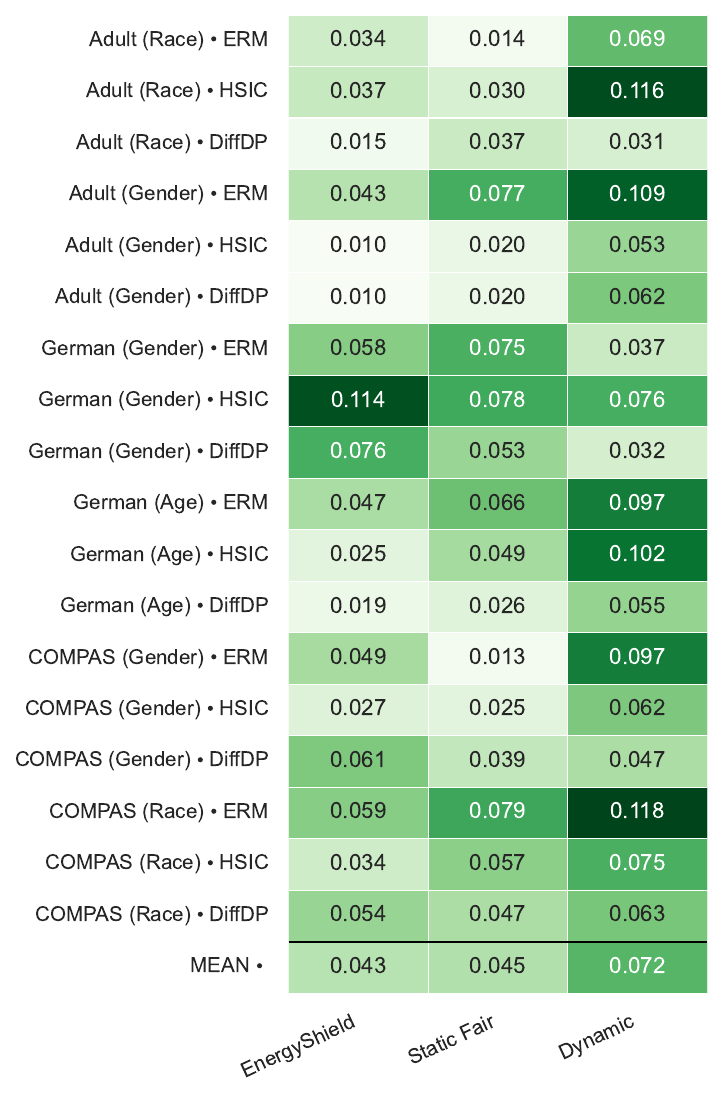}
         \caption{Fairness at 1/3 of the run}
         \label{fig:dp33-2g}
     \end{subfigure}
      \hspace{-0.024\textwidth}
     \begin{subfigure}[b]{0.35\textwidth}
         \centering
         \includegraphics[height=8.8cm, trim=0.3cm 0.3cm 0.3cm 0.3cm, clip]
         {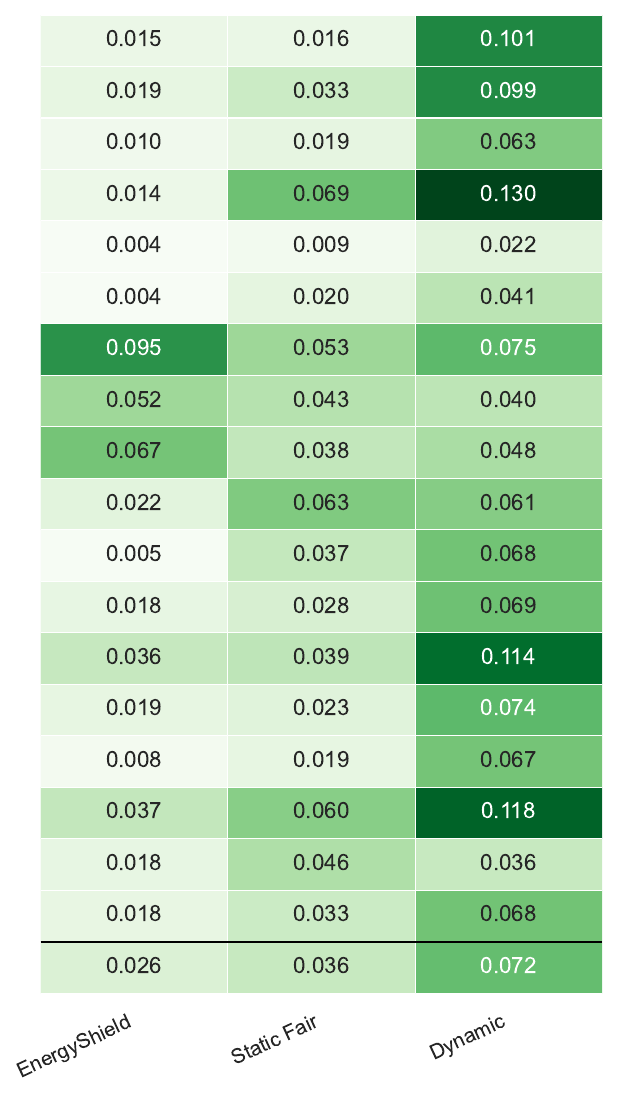}
         \caption{Fairness at the end of the run}
         \label{fig:dp-2g}
     \end{subfigure}
     \hspace{-0.065\textwidth}
     \begin{subfigure}[b]{0.35\textwidth}
         \centering
          \includegraphics[height=8.8cm, trim=0.3cm 0.3cm 0.5cm 0.3cm, clip]
         {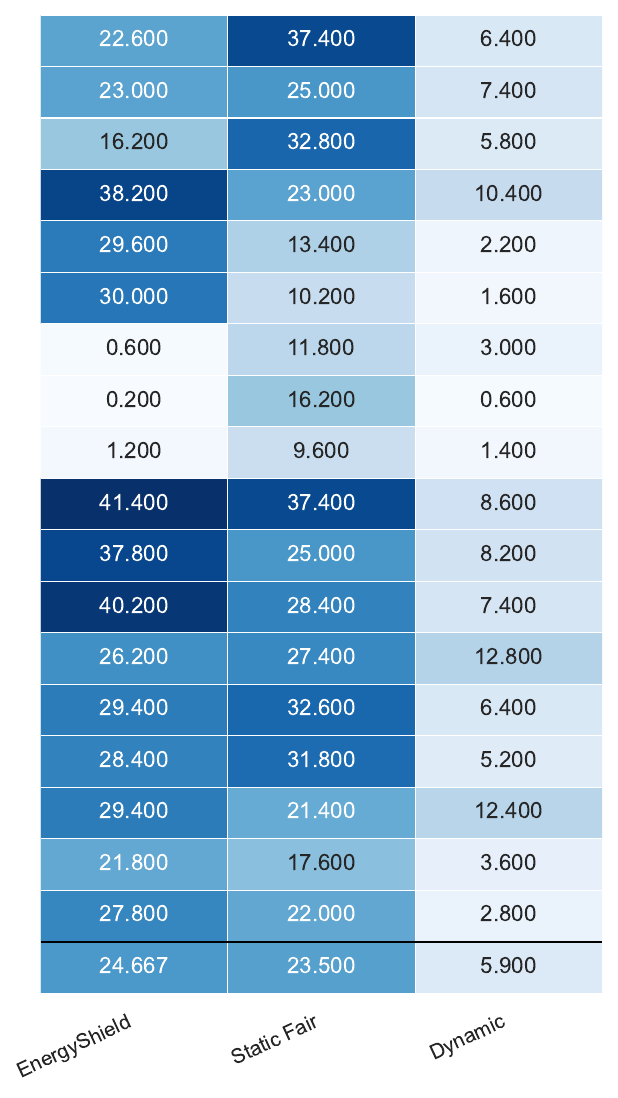}
         \caption{Number  of Interventions}
         \label{fig:intervention-2g}
     \end{subfigure}
     \caption{Group-fairness experiments on COMPAS, Adult Income, and German Credit comparing energy shields against \texttt{StaticFair} and \texttt{Dynamic} fairness shields~\cite{cano2025fairness} under the target $\safe=[-0.15,0.15]$ and $\live=[-0.075,0.075]$. \textbf{Columns (C):} (C1) demographic parity at one-third of the run, (C2) demographic parity at the end of the run, and (C3) total number of interventions.
     Darker values indicate worse performance, in terms of either demographic parity or number of interventions.}

        \label{tab:dp_end}
\end{table}

\paragraph{Comparison: General.}
In Table~\ref{tab:dp_end}, we benchmark our energy shields against the fairness shields from~\cite{cano2025fairness} in the general two-group setting of Section~\ref{sec:group_fairness}.
We consider standard benchmarks from the fairness literature, namely COMPAS~\cite{compas}, Adult Income~\cite{misc_adult_2}, and German Credit~\cite{dua2017uci}.
For each dataset, we evaluate decision makers trained with empirical risk minimization (ERM), HSIC~\cite{hsic}, and Differential Demographic Parity~\cite{fairmixup}.
At runtime, individuals are sampled sequentially from the test dataset, the trained classifier produces the original decision, and the shield may intervene by flipping the decision.
We compare against the \texttt{StaticFair} and \texttt{Dynamic} shields from~\cite{cano2025fairness}.
Following the same scale as in~\cite{cano2025fairness}, we use a time horizon of $50$ and $10$ repetitions, so each value in Table~\ref{tab:dp_end} represents the result after processing $500$ individuals.
The previous shields are configured with a demographic parity target of $0.15$.
For our energy shields, we use the same value as the running specification, i.e.,
$\safe=[-0.15,0.15]$, and impose the stricter limit specification
$\live=[-0.075,+0.075]$.
Thus, the running target specifies the tolerated short-term disparity, while the limit target specifies the desired long-term disparity.
Table~\ref{tab:dp_end} reports demographic parity after one third of the run, demographic parity at the end of the run, and the total number of interventions.
We observe that energy shields achieve better final demographic parity than their counterparts, with a number of interventions comparable to \texttt{StaticFair}.
The \texttt{Dynamic} shield often intervenes less, but it does so by keeping the process close to the boundary of the running target.
By contrast, energy shields continue to push the fairness value toward the stricter limit target, explaining their final fairness performance.

\section{Discussion}
\label{sec:conclusion}
\paragraph{Safety and liveness.}
Formal verification classifies properties over infinite traces as safety or liveness~\cite{lamport1977proving,henzinger2023quantitative}. 
A safety property asserts that ``bad'' things never occur, e.g., the car never crashes, thus its violation can be observed on a finite prefix.
A liveness property asserts that ``good'' things will eventually happen, e.g., the car reaches its goal, thus its satisfaction can only be determined in the limit, and every finite trace can be extended to an infinite satisfiable trace.
For stochastic processes one evaluates the satisfaction probability of the safety or liveness property w.r.t.\ the law of the process~\cite{baier2008principles}.
We deliberately connected short- and long-term fairness to safety and liveness respectively. 
The short-term running fairness property requires the fairness measure to remain within the running target at all times, which implies that once violated it remains violated along the infinite sequence, i.e., safety.
The long-term limit fairness property requires the limit of the fairness measure to lie in the limit target, which implies it can only be determined in the limit, i.e., liveness.
The convergence in the limit holds because the fairness measure is an average. This normalization allows the distance to $\fairtarget$ to shrink and the process to converge almost surely. This implies that almost surely the fairness measure eventually remains within the limit target. It does not imply that the fairness measure eventually remains within the limit target almost surely, i.e., $\prob(\exists \tau \in \NN \forall t\geq \tau \colon M_t\in \live)=1$ is satisfied, and $\exists \tau  \in \NN \colon  \prob(\forall t\geq \tau \colon M_t\in \live)=1 $ is not.

\paragraph{Burn-in.}
The choice of burn-in is both a value and a technical decision: from a value perspective, $\tau$ should be chosen to obtain a meaningful representation system behavior, e.g., a college admissions task force may introduce a fairness shield and wait two years before evaluating its impact.
from a technical perspective, $\tau$ has to be large enough to have some confidence on fairness measure being inside the target interval. 
For example, if we are observing individuals being accepted (1) or rejected (0), and we have a target acceptance rate between $0.6$ and $0.7$, after only four candidates, the closest acceptance rate the shield can get is either $0.5$ or $0.75$. This is because after four candidates there can be $0,1,2,3$ or $4$ accepts in total, i.e., all possible values of the fairness metric are in $\{0/4,1/4, 2/4,3/4,4/4\}$. Hence, our target can only be achieved if enough decisions have been made. The burn-in time  quantifies this concept of "enough": the shield is only required to satisfy the short-term fairness requirements after  decisions.

\section{Conclusion}
We introduced \emph{energy shields}, a physics-inspired framework for enforcing fairness at runtime. Energy shields are lightweight, probabilistic mechanisms that provide rigorous short- and long-term fairness guarantees. Utilizing a bowl-shaped energy function, enforcing fairness is reduced to an energy minimization problem, which enables gentle, adaptive interventions. We provide a synthesis algorithm based on binary search and dynamic programming to find the least-invasive shield for a given specification. The experimental validation demonstrates the practicality of our synthesis procedure and supports, in the comparison with \cite{cano2025fairness}, the claim that energy shields are the first shields to satisfy both short- and long-term guarantees.
Future work will explore multi-group fairness, properties beyond fairness, and  dynamic environments.
We establish a foundational theory of energy-based fairness shielding, extending these guarantees to multi-group fairness and properties beyond fairness remains an exciting direction.

\section*{Endmatter}

\subsection*{Generative AI Usage Statement}
The authors used ChatGPT 5.2 for the following tasks. 
For writing, ChatGPT was used to find and correct grammar and spelling mistakes, and to assess the fluency, consistency, and clarity of the text. No substantial amount of text is AI generated. 
For the theoretical results, ChatGPT was used to brainstorm proof strategies and to search the literature for potentially relevant results. 
All proofs were then independently checked and fully formalized by the authors. No generative AI was used to develop the paper's core scientific contributions.
For the experimental evaluation, ChatGPT was used to generate and refactor small portions of experimental code, specifically in the generation of plots. All code generated code was tested and validated by the authors.

\subsection*{Ethical Considerations Statement}
This work is strongly motivated by the problem of ensuring algorithmic fairness in machine learning applications. Although this is a sensitive topic, the ethical considerations for this paper are minimal. Our contributions are largely theoretical, and the datasets we employ are standard benchmarks widely used in the literature already.

\subsection*{Acknowledgements}
This work has been supported by the European Research Council under Grant No.: ERC-2020-AdG 101020093.

\bibliographystyle{ACM-Reference-Format}
\bibliography{references}

@article{cano2025fairness,
 title={Fairness Shields: Safeguarding against Biased Decision Makers},
 volume={39},
 number={15},
 journal={Proceedings of the AAAI Conference on Artificial Intelligence},
 author={Cano, Filip and Henzinger, Thomas A. and Könighofer, Bettina and Kueffner, Konstantin and Mallik, Kaushik},
 year={2025},
 pages={15659-15668}
}

@inproceedings{carr2023safe,
  title={Safe reinforcement learning via shielding under partial observability},
  author={Carr, Steven and Jansen, Nils and Junges, Sebastian and Topcu, Ufuk},
  booktitle={Proceedings of the AAAI conference on artificial intelligence},
  volume={37},
  number={12},
  pages={14748--14756},
  year={2023}
}

@inproceedings{alshiekh2018safe,
  title={Safe reinforcement learning via shielding},
  author={Alshiekh, Mohammed and Bloem, Roderick and Ehlers, Rüdiger and Könighofer, Bettina and Niekum, Scott and Topcu, Ufuk},
  booktitle={Proceedings of the AAAI conference on artificial intelligence},
  year={2018}
}

@article{karandikar2024convergence,
  title={Convergence rates for stochastic approximation: Biased noise with unbounded variance, and applications},
  author={Karandikar, Rajeeva Laxman and Vidyasagar, Mathukumalli},
  journal={Journal of Optimization Theory and Applications},
  volume={203},
  number={3},
  pages={2412--2450},
  year={2024},
  publisher={Springer}
}

@book{borkar2008stochastic,
  title={Stochastic approximation: a dynamical systems viewpoint},
  author={Borkar, Vivek S.},
  year={2008},
  publisher={Springer}
}

@incollection{robbins1971convergence,
  title={A convergence theorem for non negative almost supermartingales and some applications},
  author={Robbins, Herbert and Siegmund, David},
  booktitle={Optimizing methods in statistics},
  pages={233--257},
  year={1971},
  publisher={Elsevier}
}

@inproceedings{hardt2016equality,
  author={Hardt, Moritz and Price, Eric and Srebro, Nati},
  title={Equality of Opportunity in Supervised Learning},
  booktitle={Advances in Neural Information Processing Systems (NeurIPS)},
  pages={3315--3323},
  year={2016}
}

@misc{dua2017uci,
  author={Hofmann, Hans},
  title={{Statlog (German Credit Data)}},
  year={1994},
  howpublished={UCI Machine Learning Repository},
  note={{DOI}: https://doi.org/10.24432/C5NC77}
}

@misc{misc_adult_2,
  author={Becker, Barry and Kohavi, Ronny},
  title={{Adult}},
  year={1996},
  howpublished={UCI Machine Learning Repository},
  note={{DOI}: https://doi.org/10.24432/C5XW20}
}

@article{compas,
  author={Kirchner, Lauren and Mattu, Surya and Larson, Jeff and Angwin, Julia},
  year={2016},
  title={Machine Bias},
  journal={ProPublica},
  url={https://www.propublica.org/article/machine-bias-risk-assessments-in-criminal-sentencing},
  urldate={2024-06-04}
}

@InProceedings{pmlr-v235-alamdari24a,
  title={Remembering to Be Fair: Non-{M}arkovian Fairness in Sequential Decision Making},
  author={Alamdari, Parand A. and Klassen, Toryn Q. and Creager, Elliot and Mcilraith, Sheila A.},
  booktitle={Proceedings of the International Conference on Machine Learning (ICML)},
  pages={906--920},
  year={2024},
  volume={235}
}

@InProceedings{hsic,
  author={Pérez-Suay, Adrián and Laparra, Valero and Mateo-García, Gonzalo and Muñoz-Marí, Jordi and Gómez-Chova, Luis and Camps-Valls, Gustau},
  editor={Ceci, Michelangelo and Hollmén, Jaakko and Todorovski, Ljupčo and Vens, Celine and Džeroski, Sašo},
  title={Fair Kernel Learning},
  booktitle={Machine Learning and Knowledge Discovery in Databases (KDD)},
  year={2017},
  publisher={Springer International Publishing},
  address={Cham},
  pages={339--355}
}

@inproceedings{fairmixup,
  author={Ching-Yao Chuang and Youssef Mroueh},
  title={Fair Mixup: Fairness via Interpolation},
  booktitle={International Conference on Learning Representations (ICLR)},
  publisher={OpenReview.net},
  year={2021}
}

@article{caton2020fairness,
  title={Fairness in machine learning: A survey},
  author={Caton, Simon and Haas, Christian},
  journal={ACM Computing Surveys},
  year={2020},
  publisher={ACM New York, NY}
}

@inproceedings{wen2021algorithms,
  title={Algorithms for fairness in sequential decision making},
  author={Wen, Min and Bastani, Osbert and Topcu, Ufuk},
  booktitle={International Conference on Artificial Intelligence and Statistics (AISTATS)},
  pages={1144--1152},
  year={2021},
  organization={PMLR}
}

@inproceedings{li2023certifying,
  title={Certifying the Fairness of KNN in the Presence of Dataset Bias},
  author={Li, Yannan and Wang, Jingbo and Wang, Chao},
  booktitle={International Conference on Computer Aided Verification (CAV)},
  publisher={Springer},
  year={2023}
}

@inproceedings{baumeister2025stream,
  title={Stream-Based Monitoring of Algorithmic Fairness},
  author={Baumeister, Jan and Finkbeiner, Bernd and Scheerer, Frederik and Siber, Julian and Wagenpfeil, Tobias},
  booktitle={International Conference on Tools and Algorithms for the Construction and Analysis of Systems},
  pages={60--81},
  year={2025},
  organization={Springer}
}

@article{meyer2021certifying,
  title={Certifying Robustness to Programmable Data Bias in Decision Trees},
  author={Meyer, Anna and Albarghouthi, Aws and D'Antoni, Loris},
  journal={Advances in Neural Information Processing Systems (NeurIPS)},
  volume={34},
  pages={26276--26288},
  year={2021}
}

@inproceedings{ghosh2020justicia,
  title={Justicia: A stochastic SAT approach to formally verify fairness},
  author={Ghosh, Bishwamittra and Basu, Debabrota and Meel, Kuldeep S},
  booktitle={Proceedings of the AAAI Conference on Artificial Intelligence (AAAI)},
  volume={35},
  pages={7554--7563},
  year={2021}
}

@article{bastani2019probabilistic,
  title={Probabilistic verification of fairness properties via concentration},
  author={Bastani, Osbert and Zhang, Xin and Solar-Lezama, Armando},
  journal={Proceedings of the ACM on Programming Languages},
  volume={3},
  number={OOPSLA},
  pages={1--27},
  year={2019},
  publisher={ACM New York, NY, USA}
}

@inproceedings{henzinger2023dynamic,
  author={Henzinger, Thomas A. and Karimi, Mahyar and Kueffner, Konstantin and Mallik, Kaushik},
  title={Runtime Monitoring of Dynamic Fairness Properties},
  booktitle={Proceedings of the {ACM} Conference on Fairness, Accountability, and Transparency (FAccT)},
  pages={604--614},
  publisher={{ACM}},
  year={2023}
}

@inproceedings{feldman2015certifying,
  title={Certifying and removing disparate impact},
  author={Feldman, Michael and Friedler, Sorelle A and Moeller, John and Scheidegger, Carlos and Venkatasubramanian, Suresh},
  booktitle={Proceedings of the 21th ACM SIGKDD International Conference on Knowledge Discovery and Data Mining (KDD)},
  pages={259--268},
  year={2015}
}

@inproceedings{gordaliza2019obtaining,
  title={Obtaining fairness using optimal transport theory},
  author={Gordaliza, Paula and Del Barrio, Eustasio and Fabrice, Gamboa and Loubes, Jean-Michel},
  booktitle={International Conference on Machine Learning (ICML)},
  pages={2357--2365},
  year={2019},
  organization={PMLR}
}

@inproceedings{agarwal2018reductions,
  title={A reductions approach to fair classification},
  author={Agarwal, Alekh and Beygelzimer, Alina and Dudík, Miroslav and Langford, John and Wallach, Hanna},
  booktitle={International Conference on Machine Learning (ICML)},
  pages={60--69},
  year={2018},
  organization={PMLR}
}

@inproceedings{d2020fairness,
  title={Fairness is not static: deeper understanding of long term fairness via simulation studies},
  author={D'Amour, Alexander and Srinivasan, Hansa and Atwood, James and Baljekar, Pallavi and Sculley, David and Halpern, Yoni},
  booktitle={Proceedings of the Conference on Fairness, Accountability, and Transparency (FAccT)},
  pages={525--534},
  year={2020}
}

@inproceedings{gupta2025monitoring,
  author={Gupta, Ashutosh and Henzinger, Thomas A. and Kueffner, Konstantin and Mallik, Kaushik and Pape, David},
  title={Monitoring Robustness and Individual Fairness},
  booktitle={KDD 2025},
  year={2025}
}

@inproceedings{albarghouthi2019fairness,
  title={Fairness-aware programming},
  author={Albarghouthi, Aws and Vinitsky, Samuel},
  booktitle={Proceedings of the Conference on Fairness, Accountability, and Transparency},
  pages={211--219},
  year={2019}
}

@article{zafar2019fairness,
  title={Fairness constraints: A flexible approach for fair classification},
  author={Zafar, Muhammad Bilal and Valera, Isabel and Gomez-Rodriguez, Manuel and Gummadi, Krishna P},
  journal={The Journal of Machine Learning Research},
  volume={20},
  number={1},
  pages={2737--2778},
  year={2019},
  publisher={JMLR. org}
}

@book{baier2008principles,
  title={Principles of model checking},
  author={Baier, Christel and Katoen, Joost-Pieter},
  year={2008},
  publisher={MIT press}
}

@inproceedings{maler2004monitoring,
  title={Monitoring temporal properties of continuous signals},
  author={Maler, Oded and Nickovic, Dejan},
  booktitle={International Symposium on Formal Techniques in Real-Time and Fault-Tolerant Systems},
  pages={152--166},
  year={2004},
  organization={Springer}
}

@incollection{bartocci2018specification,
  title={Specification-based monitoring of cyber-physical systems: a survey on theory, tools and applications},
  author={Bartocci, Ezio and Deshmukh, Jyotirmoy and Donzé, Alexandre and Fainekos, Georgios and Maler, Oded and Ničković, Dejan and Sankaranarayanan, Sriram},
  booktitle={Lectures on Runtime Verification},
  pages={135--175},
  year={2018},
  publisher={Springer}
}

@article{mehrabi2021survey,
  title={A survey on bias and fairness in machine learning},
  author={Mehrabi, Ninareh and Morstatter, Fred and Saxena, Nripsuta and Lerman, Kristina and Galstyan, Aram},
  journal={ACM Computing Surveys (CSUR)},
  volume={54},
  number={6},
  pages={1--35},
  year={2021},
  publisher={ACM New York, NY, USA}
}

@inproceedings{sun2021probabilistic,
  title={Probabilistic verification of neural networks against group fairness},
  author={Sun, Bing and Sun, Jun and Dai, Ting and Zhang, Lijun},
  booktitle={International Symposium on Formal Methods},
  pages={83--102},
  year={2021},
  organization={Springer}
}

@article{albarghouthi2017fairsquare,
  title={Fairsquare: probabilistic verification of program fairness},
  author={Albarghouthi, Aws and D'Antoni, Loris and Drews, Samuel and Nori, Aditya V},
  journal={Proceedings of the ACM on Programming Languages},
  volume={1},
  number={OOPSLA},
  pages={1--30},
  year={2017},
  publisher={ACM New York, NY, USA}
}

@inproceedings{stoller2011runtime,
  title={Runtime verification with state estimation},
  author={Stoller, Scott D and Bartocci, Ezio and Seyster, Justin and Grosu, Radu and Havelund, Klaus and Smolka, Scott A and Zadok, Erez},
  booktitle={International conference on runtime verification},
  pages={193--207},
  year={2011},
  organization={Springer}
}

@inproceedings{dwork2012fairness,
  title={Fairness through awareness},
  author={Dwork, Cynthia and Hardt, Moritz and Pitassi, Toniann and Reingold, Omer and Zemel, Richard},
  booktitle={Proceedings of the 3rd innovations in theoretical computer science conference},
  pages={214--226},
  year={2012}
}

@inproceedings{cordoba2023safety,
  title={Safety shielding under delayed observation},
  author={Córdoba, Filip Cano and Palmisano, Alexander and Fränzle, Martin and Bloem, Roderick and Könighofer, Bettina},
  booktitle={Proceedings of the International Conference on Automated Planning and Scheduling},
  volume={33},
  pages={80--85},
  year={2023}
}

@inproceedings{faymonville2017real,
  title={Real-time stream-based monitoring},
  author={Faymonville, Peter and Finkbeiner, Bernd and Schwenger, Maximilian and Torfah, Hazem},
  journal={arXiv preprint arXiv:1711.03829},
  year={2017}
}

@inproceedings{donze2010robust,
  title={Robust satisfaction of temporal logic over real-valued signals},
  author={Donzé, Alexandre and Maler, Oded},
  booktitle={International Conference on Formal Modeling and Analysis of Timed Systems},
  pages={92--106},
  year={2010},
  organization={Springer}
}

@inproceedings{liu2018delayed,
  title={Delayed impact of fair machine learning},
  author={Liu, Lydia T and Dean, Sarah and Rolf, Esther and Simchowitz, Max and Hardt, Moritz},
  booktitle={International Conference on Machine Learning},
  pages={3150--3158},
  year={2018},
  organization={PMLR}
}

@article{baier2003ctmc,
  author={Baier, C. and Haverkort, B. and Hermanns, H. and Katoen, J.-P.},
  journal={IEEE Transactions on Software Engineering},
  title={Model-checking algorithms for continuous-time Markov chains},
  year={2003},
  volume={29},
  number={6},
  pages={524-541}
}

@article{lamport1977proving,
  title={Proving the correctness of multiprocess programs},
  author={Lamport, Leslie},
  journal={IEEE transactions on software engineering},
  number={2},
  pages={125--143},
  year={1977},
  publisher={IEEE}
}

@inproceedings{henzinger2023quantitative,
  title={Quantitative Safety and Liveness.},
  author={Henzinger, Thomas A and Mazzocchi, Nicolas and Sara{\c{c}}, N Ege},
  booktitle={FoSSaCS},
  pages={349--370},
  year={2023}
}

@inproceedings{cano2025algorithmic,
  title={Algorithmic Fairness: A Runtime Perspective},
  author={Cano, Filip and Henzinger, Thomas A. and Kueffner, Konstantin},
  booktitle = {Proceedings of the International Conference on Runtime Verification (RV)},
  pages={1-21},
  year={2025}
}

@inproceedings{jansenconcur2020,
  author       = {Nils Jansen and
                  Bettina K{\"{o}}nighofer and
                  Sebastian Junges and
                  Alex Serban and
                  Roderick Bloem},
  title        = {Safe Reinforcement Learning Using Probabilistic Shields (Invited Paper)},
  booktitle    = {{CONCUR}},
  series       = {LIPIcs},
  volume       = {171},
  pages        = {3:1--3:16},
  publisher    = {Schloss Dagstuhl - Leibniz-Zentrum f{\"{u}}r Informatik},
  year         = {2020}
}

@inproceedings{probshieldsIJCAI2023,
  author       = {Wen{-}Chi Yang and
                  Giuseppe Marra and
                  Gavin Rens and
                  Luc De Raedt},
  title        = {Safe Reinforcement Learning via Probabilistic Logic Shields},
  booktitle    = {{IJCAI}},
  pages        = {5739--5749},
  publisher    = {ijcai.org},
  year         = {2023}
}

@inproceedings{katoen2016probabilistic,
  title={The probabilistic model checking landscape},
  author={Katoen, Joost-Pieter},
  booktitle={Proceedings of the 31st Annual ACM/IEEE Symposium on Logic in Computer Science},
  pages={31--45},
  year={2016}
}

\appendix

\section{Detailed Proofs}

\noindent\textbf{Claim \ref{claim:process}} (Shielded decision process)\textbf{.}
\emph{
    The shielded decision process generated by the energy shield can be written as a sequence of Bernoulli random variables with evolving bias, $Z_t\sim \mathrm{Bernoulli}(p_t)$. 
The biases are defined recursively as $p_1=p$ and $p_{t+1} = f(\mu_t)$ for a given history $z_1, \dots, z_t$, where 
\begin{equation}
\label{eq:def_of_f_appendix}
f(\mu) = \begin{cases}
 p + (1-p)\zeta(\mu) & 
 \mbox{ if }\: \mu \leq \fairpivot, \\
 p\cdot \left(1-\zeta(\mu)\right) & \mbox{ if }\: \mu > \fairpivot.
    \end{cases} 
\end{equation}
Moreover, the resulting shielded fairness process can be written as 
\begin{equation}
\label{eq:fairness-process-update_appendix}
    M_t = M_{t-1} + \frac{1}{t}(Z_{t} - M_{t-1}) \quad \text{(with $M_{0}=0$)}.
\end{equation}
}
\begin{proof}
    Both equations are just simple computations.

    \textbf{Equation~\ref{eq:def_of_f_appendix}.}
    Suppose $\mu_t \leq \fairpivot$, i.e., the shield favors 1's.
    Then a 1 can be obtained either by $X_{t+1}=1$ (which happens with probability $p$), or by flipping the decision $X_{t+1}=0$ (which happens with probability $(1-p)\zeta(\mu_t)$).
    Therefore, when $\mu_t\leq \fairpivot$, $Z_{t+1}$ behaves like a Bernoulli of bias $p + (1-p)\zeta(\mu_t)$.
    Analogously, when $\mu_t > \fairpivot$, the shield favors 0's, so to obtain a 1 we need to toss $X_{t+1}=1$ and for the shield to fail to flip the decision, which happens with probability $(1-\zeta(\mu_t))$.

    \textbf{Equation~\ref{eq:fairness-process-update_appendix}.}
    \[
    M_t = \frac{\sum_{i=1}^t Z_i}{t} = 
    \frac{Z_t}{t} + \frac{t-1}{t}\cdot \frac{\sum_{i=1}^{t-1} Z_i}{t-1} = 
    \frac{Z_t}{t} + \frac{t-1}{t}M_{t-1} = 
    M_{t-1} + \frac{1}{t}(Z_{t} - M_{t-1}).
    \]
\end{proof}

\subsection{Long term guarantees}

\noindent\textbf{Lemma \ref{lem:unique-fixpoint-1d}.}
\emph{
    The function $f\colon [0,1]\to [0,1]$ defined as in Eq.~\ref{eq:def_of_f} is continuously differentiable, and has a unique point $\fairtarget\in [0,1]$ such that $f(\fairtarget) = \fairtarget$.
    Furthermore, $\fairtarget$ sits between $p$ and $\fairpivot$.
}
\begin{proof}
    We need to prove smoothness, existence of the fixpoint, and that it sits between $p$ and $\fairpivot$.
    
    \emph{Smoothness.}
    The function $f$ inherits continuous differentiability from $\zeta$ clearly at all points except maybe $\mu = \fairpivot$.
    For $\mu=\fairpivot$, 
    the assumption of the energy function being flat at the pivoting point $\fairpivot$ guarantees continuous differentiability (Def.~\ref{def:energyfunc}, item 3). 
    We first write the expression for $f'$: 
    \begin{equation}
    \label{eq:aux1_appendix}
      f'(\mu) = \begin{cases}
        -p\cdot \zeta'(\mu) & \mbox{ if } \mu > \fairpivot \\
        (1-p)\zeta'(\mu) & \mbox{ otherwise.}
    \end{cases}  
    \end{equation}
    We can check that $f$ is continuous at $\mu=\fairpivot$, as
    \[
    p\cdot(1-  \zeta(\fairpivot)) = p, \mbox{ and }\quad
    p+(1-p)\cdot\zeta(\fairpivot) = p.
    \]
    Similarly, we can check that $f'$ is continuous at $f=\fairpivot$, as
    \[
    -p\cdot \zeta'(\fairpivot) = 0, \mbox{ and }
    (1-p)\cdot\zeta'(\fairpivot) = 0.
    \]
    Note that both are necessary conditions. 
    For $f$ to be continuous at $\mu=\fairpivot$, we need
    \[
    p(1-\zeta(\fairpivot)) = p+(1-p)\zeta(\fairpivot) \iff
    -p\zeta(\fairpivot) = \zeta(\fairpivot)-p\zeta(\fairpivot) \iff \zeta(\fairpivot) = 0.
    \]
    Similarly, for $f'$ to be continuous as $\mu=\fairpivot$, we need
    \[
    -p\zeta'(\fairpivot) = (1-p)\zeta'(\fairpivot) \iff \zeta'(\fairpivot) = 0.
    \]

    \emph{Existence and uniqueness of a fixpoint.}
    Consider the function $g(\mu) = f(\mu) - \mu$.
    We have $g(0) = 0 + (1-p)\zeta(0) - 0 > 0$ and $g(1) = 1\cdot(1-\zeta(1)) -1 < 1-1=0$.
    Since $g$ is continuous, there must be a point $\fairtarget\in [0,1]$ such that $g(\fairtarget) = 0$.
    Suppose this point is not unique, i.e., there exist two point $x<\mu$ such that $g(x) = g(\mu) = 0.$
    As we can see in Eq.~\ref{eq:aux1_appendix},  $f'(\mu) \leq 0$ for all $\mu\in[0,1]$, therefore $g$ is non-increasing, so if $g(x) = g(\mu) = 0$, then $g(z) = 0$ for all $z\in [x,\mu]$.
    The interval $[x,\mu]$ contains at least two numbers that are either larger or lower than $\fairpivot$.
    Let $z_1, z_2\in [x,\mu]\cap [0,\fairpivot)$ be such numbers, with $z_1 < z_2$.
    Then we have
    $$p + (1-p)\zeta(z_1) - z_1 = p + (1-p)\zeta(z_2) - z_2,$$
    which implies that
    \[
    z_2 - z_1 = (1-p)(\zeta(z_2) - \zeta(z_1)).
    \]
    On the left-hand side we have a positive number (because $z_1 < z_2$). 
    On the right-hand side we have a non-positive number (because $\zeta$ is decreasing in the range $[0,\fairpivot)$, which is a contradiction.
    Since assuming $z_1, z_2\in [x,\mu]\cap [0,\fairpivot)$ leads to a contradiction, it must be that $z_1, z_2\in [x,\mu]\cap (\fairpivot, 1]$.
    However, with a similar argument this implies that 
    \[
    p(1-\zeta(z_1)) - z_1 = p(1-\zeta(z_2)) - z_2 \iff z_2 - z_1 = p(\zeta(z_1) - \zeta(\zeta_2)).
    \]
    Since $\zeta$ is in the decreasing range, in the right-hand side we have again a non-positive term, which is a contradiction.\\
    \emph{Placement of the fixpoint.}
    We write the case where $p \leq \fairpivot$, the case where $p > \fairpivot$ is analogous. If $p \leq \fairpivot$, we have
    \[
    g(p) = (1-p)\zeta(p) \geq 0, \quad \mbox{and} \quad 
    g(\fairpivot) = p - \fairpivot \leq 0.
    \]
    By continuity, there exists a point $\fairtarget\in [p,\fairpivot]$ such that $g(\fairtarget) = 0$.
    Note that this includes the special case where $p=\fairpivot$, in which the unique fixpoint is $\fairtarget = \fairpivot = p$.   
\end{proof}

\noindent\textbf{Lemma \ref{lem:convergence_outcome_1d}.}
\emph{
    The process described in Eq.~\ref{eq:process-update} converges almost surely to the unique fixpoint $\fairtarget$ of $f$, as defined in Eq.~\ref{eq:def_of_f}.
    Furthermore, the convergence rate of the error satisfies $(M_t-\fairtarget)^2 = o(1/t^\lambda)$ for all $\lambda\in (0,1)$ almost surely.
}
\begin{proof}
    \emph{Error construction.}
    For convenience, we will prove that the sequence converges to the unique root of $g(\mu) = f(\mu) - \mu$.
    Let $(V_t)_{t\in\mathbb{N}}$ be the sequence of squared distances to the target, i.e., 
    $V_t = (\mu_t-\fairtarget)^2$.
    We will prove that $V_t\to 0$ a.s. (almost surely).

    First we need to find a recurrence formula for $V_t$.
    Recall from Eq.~\ref{eq:robbins-monro-form} the recurrence for $\mu_t$ is
    \[
    \mu_{t} = \mu_{t-1} + \gamma_t(g(\mu_{t-1}) + \xi_t),
    \]
    where $\gamma_t = 1/t$ and $\xi_t$ has null expectation. 
    For $V_t$ we have
    \begin{align}
         V_{t} & = (\mu_{t} - \fairtarget)^2 =
         \left(
         \mu_{t-1} - \fairtarget + \gamma_t(g(\mu_{t-1}) + \xi_t)
         \right)^2 \\
         &= 
        (\mu_{t-1} - \fairtarget)^2 + 2\gamma_t(\mu_{t-1} - \fairtarget)(g(\mu_{t-1})+\xi_t) + \gamma^2_t (g(\mu_{t-1})+\xi_t)^2.
        \label{eq:aux3}
    \end{align}
    Taking conditional expectations we have
    \begin{equation}
        \EE[V_{t}\mid \mu_{t-1}] = V_{t-1}  + 2\gamma_t(\mu_{t-1}-\fairtarget)g(\mu_{t-1}) +\gamma_t^2\EE[(g(\mu_{t-1})+\xi_t)^2].
    \end{equation}
    Since $g(\mu_{t-1})$ is deterministic given  $\mu_{t-1}$ and $\EE[\xi_t\mid \mu_{t-1}]=0$, we have
    \[
    \EE[(g(\mu_{t-1})+\xi_t)^2] = g(\mu_{t-1})^2 + \sigma_t^2,
    \]
    where we define $\sigma_t^2 = \EE[\xi_t^2\mid \mu_{t-1}]$. 
    Recall that $\xi_t = Z_t - f(\mu_{t-1})$, which is always in the interval $[-1,1]$, therefore $\sigma_t^2\leq 1$.
    Plugging everything into Eq.~\ref{eq:aux3}, we have
    \begin{equation}
    \label{eq:aux4}
        \EE[V_{t}\mid \mu_{t-1}] \leq V_{t-1} + 2\gamma_t(\mu_{t-1} - \fairtarget)g(\mu_{t-1}) + \gamma^2_t (g(\mu_{t-1})^2 + \sigma_t^2).
    \end{equation}
    Let $\alpha_t = -2\gamma_t\frac{g(\mu_{t-1})}{\mu_{t-1}-\fairtarget}$ and $\beta_t = \gamma_t^2(g(\mu_{t-1})^2+\sigma_t^2)$.
    Note that $\alpha_t$ is well defined, even when $\mu_{t-1}=\fairtarget$ as, using L'H\^{o}pital's rule:
    \[\lim_{\mu
    \to \fairtarget}\frac{g(\mu)}{\mu-\fairtarget} = 
    \lim_{\mu
    \to \fairtarget}\frac{f(\mu) - \mu}{\mu-\fairtarget} =
    \lim_{\mu\to \fairtarget} \frac{g'(\mu)}{1} = g'(\fairtarget).\]
    Also, $g$ is decreasing around $\fairtarget$, so $g(\mu_{t-1})$ and $\mu_{t-1}-\fairtarget$ have opposite signs. Therfore $\alpha_t\geq 0$.
    We can then rewrite Eq.~\ref{eq:aux4} as
    \begin{equation}
        \EE[V_{t}\mid \mu_{t-1}] \leq (1-\alpha_t)V_{t-1} + \beta_t.
    \end{equation}
    \emph{Proof of convergence.}
    This is a standard form for the condition in the classical Robbins-Siegmund theorem~\cite{robbins1971convergence}, 
    which guarantees convergence of $V_t$ almost surely whenever the sequences $(\alpha_t)$ and $(\beta_t)$ are non-negative and satisfy the limiting properties that $\sum_{t=1}^\infty \alpha_t = \infty$ and $\sum_{t=1}^\infty \beta_t < \infty$ almost surely.
    We use the recent results in~\cite{karandikar2024convergence} to guarantee convergence. In particular, we use Thm. 5.1 to guarantee $V_t\to \infty$ almost surely, and Thm. 5.2 to guarantee rates of convergence.
    To apply both theorems, we need to guarantee that
    \begin{enumerate}
        \item $\sum_{t=1}^\infty \alpha_t = \infty$ almost surely, and 
        \item $\sum_{t=1}^\infty \beta_t < \infty$ almost surely.
    \end{enumerate}
    Note that both results are true in the strict sense (not only ``almost surely''). 
    In the case of $\alpha_t$, note that there exists $C>0$ such that $g(\mu_t)/(\fairtarget-\mu_t) \geq C$. Otherwise, it would imply that $g$ has a zero other than $\fairtarget$ or that $g'(\fairtarget)=0$, both which we know are not true.
    Therefore \[
    \sum_{t=1}^\infty \alpha_t \geq 2C \sum_{t=1}^\infty \gamma_t = \infty.
    \]
    Since $\gamma_t=1/t$, the sum $\sum_t\gamma_t$ diverges.
    For the case of $\beta_t$, recall by definition that $g(\mu_t)^2\leq 1$.
    We have also previously established that $\sigma_t^2\leq 1$:
    \[
    \sum_{t=1}^\infty \beta_t \leq 2\sum_{t=1}^\infty \gamma_t^2 < \infty.
    \]
    Theorem 5.1 in~\cite{karandikar2024convergence} is stated in terms of a sequence $h_t$. 
    We can use $h_t = V_t$ (since the identity is what~\cite{karandikar2024convergence} defines as a class $\mathcal B$ function), to guarantee $V_t\to 0$ almost surely.\\
    
    \noindent\emph{Convergence rates.}
    From Theorem 5.2 in~\cite{karandikar2024convergence}, we have that $V_t=o(1/t^\lambda)$ for any $\lambda$ such that $\alpha_t \geq \lambda/t$ for sufficiently large $t$.
    In our case, we can take $\lambda=2|g'(\fairtarget)|-\epsilon$ for any $\epsilon>0$.
    Therefore, $V_t=o(1/t^\lambda)$ for any $\lambda\in (0,1)\cap (0, 2|g'(\fairtarget)|) = (0,\min\{1,2|g'(\fairtarget)|\})$.

    However, we can be more fine-grained and show that actually $\min\{1,2|g(\fairtarget)|\}=1$.
    Recall from the definition of $g$ that $g'(\mu) = f'(\mu) -1$, and $f'(\mu) \leq 0$ for all $\mu\in [0,1]$.
    Therefore $|g'(\mu)| = 1+|f'(\mu)|\geq 1$.
    Altogether, we obtain the expected result of $V_t=o(1/t^\lambda)$ for all $\lambda\in(0,1)$.
\end{proof}

\noindent\textbf{Lemma \ref{lem:cost}.}
\emph{
    For the process described in Eq.~\ref{eq:process-update} ,
    the corresponding sequence of average interference $(\nu_t)_{t\in\NN}$ converges to $h(\fairtarget)$, 
    where $\fairtarget$ is the fixpoint of $f$ (Eq.~\ref{eq:def_of_f}) and 
    $h$ is as defined in Eq.~\ref{eq:def_of_h}.
}
\begin{proof}
    By definition $N_t = \frac{1}{t}\sum_{i=1}^t Y_i$.

    Since $\EE[Y_t\mid \mu_{t-1}] = h(\mu_{t-1})$, and $\mu_t\xrightarrow{t\to\infty} \fairtarget$ almost surely.
    Since $h$ is a continuous function, 
    \[
    \lim_{t\to\infty} N_t = 
    h\left(\lim_{t\to \infty} M_t \right) =
    h(\fairtarget).
    \]
\end{proof}

\textbf{Theorem \ref{thm:mainconvergence}.}
\emph{
    Let $\fairtarget\in [0,1]$.
   Given the shielded decision process as described in Eq.~\ref{eq:process-update}, with bias parameter 
    $p$
    and energy function 
    $\zeta$,
    the shielded fairness process $(M_t)_{t\in\NN}$ converges almost surely to $\fairtarget$ if and only if 
    \begin{equation}
    \label{eq:mainconvergenceeq_appendix}
    \zeta(\fairtarget) = 
    \begin{cases}
        |\fairtarget-p|/(1-p) & \mbox{ if } p < \fairtarget, \\
        |\fairtarget-p|/p & \mbox{ otherwise. }
    \end{cases}
    \end{equation}
    Furthermore, the expected interference cost $(\EE[\nu_t])_{t\in \NN}$ converges almost surely to $|\fairtarget-p|$.
}
\begin{proof}
    We know from Lemma~\ref{lem:convergence_outcome_1d} that the fairness outcome converges to a fixpoint that is always between $p$ and $\fairpivot$.
    If we want the process to converge to a desired outcome $\fairtarget$, 
    then we need to find $\fairpivot$ and $\zeta$ such that $f(\fairtarget) = \fairtarget$.

    The function $f(\mu)$ is defined by parts depending on whether $\mu$ is larger or smaller than $\fairpivot$.
    If $p \leq \fairtarget$, then we need $\fairpivot \geq \fairtarget$,
    so we will be using the expression 
    $f(\mu) = p + (1-p)\zeta(\mu)$ 
    for $\mu=\fairtarget$.
    Imposing $\fairtarget$ to be a fixpoint, we have
    \[
    \fairtarget = p + (1-p)\zeta(\fairtarget) 
    \iff 
    \zeta(\fairtarget) = \frac{\fairtarget - p}{1-p}.
    \]
    Analogously, if $p\geq \fairpivot$, then we need $\fairpivot \leq \fairtarget$, 
    so we will use the expression 
    $f(\mu) = p(1-\zeta(\mu))$. 
    If we set $\fairtarget$ to be a fixpoint, we have
    \[
    \fairtarget = 
    p(1-\zeta(\fairtarget)) 
    \iff 
    \zeta(\fairtarget) = 
    1- \frac{\fairtarget}{p}.
    \]
    To prove the expectation of cost, just recall from 
    Lemma~\ref{lem:cost} that the expected cost converges to $h(\fairtarget)$.
    If $p\leq \fairtarget$, 
    then $\fairpivot\geq \fairtarget$,
    so we have
    \[
    h(\fairtarget) = 
    (1-p)\zeta(\fairtarget), \quad \mbox{and}\quad 
    \zeta(\fairtarget) = 
    \frac{\fairtarget - p}{1-p} \implies 
    h(\fairtarget) = \fairtarget-p.
    \]
    Analogously, if $p\geq\fairtarget$, then $\fairpivot\leq\fairtarget$, so we have 
    \[
    h(\fairtarget) = p\zeta(\fairtarget), \quad \mbox{and}\quad 
    \zeta(\fairtarget) = 1-\frac{\fairtarget}{p} \implies 
    h(\fairtarget) = p-\fairtarget.
    \]
    This completes the proof, since in both cases $h(\fairtarget) = |\fairtarget-p|$.
\end{proof}

\subsection{Short-term guarantees: Upper Bounds}
\begin{lemma}
    Let $\Delta_t = \mu_t-\fairtarget$.
    Recall $\xi_t = Z_{t} - f(M_{t-1})$ from Equation~\ref{eq:robbins-monro-form}
    The following holds:
    \begin{equation}
    \label{eq:unrolled}
        \Delta_t = A_{1,t}\,\Delta_1\ + \sum_{i=1}^{t-1} w_{i+1,t}\,\xi_{i+1}, \qquad \mbox{where}
    \end{equation}
    \begin{align}
        A_{i,t} = \prod_{j=i}^{t-1}\Big(1-\frac{\alpha_j}{j+1}\Big), \qquad 
        w_{i,t} = \frac{A_{i,t}}{i},
        \qquad \mbox{and} \qquad
        \alpha_t := 1 - \frac{f(\mu_t)-\fairtarget}{\mu_t - \fairtarget} \qquad 
        \mbox{with} \quad \alpha_t=1, \: \mbox{ if } \mu_t=\fairtarget.
    \end{align}
    Furthermore, the following inequality holds:
    \begin{equation}
    \label{eq:sumw2}
    \sum_{i=1}^{t-1} w_{i+1,t}^2\ \le\ \frac{4^{1-\beta}}{t+1}
    \qquad\text{with}\qquad
    \beta = \sup_{r\in [0,1]} f'(r).    
    \end{equation}
\end{lemma}
\begin{proof}
    From the definition of $\Delta_t$ and the recursion for $\mu_t$ expressed in Equation~\ref{eq:process-update}, we have
    \begin{equation}
    \label{eq:Deltat+1}
        \Delta_{t+1}\ =\ \Big(1-\frac{\alpha_t}{t+1}\Big)\Delta_t\ +\ \frac{\xi_{t+1}}{t+1}.
    \end{equation}
    Equation~\ref{eq:unrolled} follows by induction on $t$ from Eq.~\ref{eq:Deltat+1}. 
    The base case $t=1$ is trivial, 
    since product defining 
    $A_{i,t}$ is empty, and so is the sum of $w_{i+1,t}\xi_{i+1}$.
    For the induction step, can apply the induction hypothesis to Equation~\ref{eq:Deltat+1} to obtain
    \begin{align*}
        \Delta_{t+1} & = \Big(1-\frac{\alpha_t}{t+1}\Big)\Delta_t\ +\ \frac{\xi_{t+1}}{t+1} \\
        & = 
        \Big(1-\frac{\alpha_t}{t+1}\Big)\left(A_{1,t}\,\Delta_1\ + \sum_{i=1}^{t-1} w_{i+1,t}\,\xi_{i+1}\right) + \frac{\xi_{t+1}}{t+1}.
        A_{1,t}\Big(1-\frac{\alpha_t}{t+1}\Big)\Delta_1 + 
    \end{align*}
    Note that, from the definition of $A_{i,t}$, we have 
    \[
    \Big(1-\frac{\alpha_t}{t+1}\Big)A_{i,t} = A_{i,t+1}, \qquad \mbox{and}\qquad
    \Big(1-\frac{\alpha_t}{t+1}\Big)w_{i,t} = w_{i,t+1}.
    \]
    Using the previous identities, plus the case that $w_{t+1,t+1} = \frac{1}{t+1}$
    we get
    \begin{align*}
        \Delta_{t+1} & =
        A_{1,t+1}\Delta_1 + 
        \sum_{i=1}^{t-1}w_{i+1,t+1}\xi_{i+1} + 
        \frac{\xi_{t+1}}{t+1} \\
        & = 
        A_{1,t+1}\Delta_1 + 
        \sum_{i=1}^{t}w_{i+1,t+1}\xi_{i+1},
    \end{align*}
    which finishes the induction step.

    To prove the lower bound expressed in Equation~\ref{eq:sumw2}, consider 
    $\beta = \sup_{r\in [0,1]} f'(r)$. 
    Since $f(\fairtarget) = \fairtarget$, and $f$ is differentiable on the compact interval $[0,1]$, by the mean-value theorem
    \[
    \frac{f(\mu)-\fairtarget}{\mu-\fairtarget} = \frac{f(\mu)-f(\fairtarget)}{\mu-\fairtarget}\in \{f'(z):\ z\ \text{between }\mu \text{ and }\fairtarget\}\subseteq \big(-\infty,\ \beta \big].
    \]
    Therefore, we have the following lower bound for $\alpha_t$ that holds for every $t\in\NN$:
    \begin{equation}
        \alpha_t = 1-\frac{f(\mu)-\fairtarget}{\mu-\fairtarget}\geq 1-\beta.
    \end{equation}
    For $u\in (0,1)$, we can use the bound $\log(1-u) \leq -u$. We can apply this to the definition of $A_{i,t}$
    \begin{equation}
    \label{eq:boundA_it}
    \log A_{i,t} = \sum_{j=i}^{t-1}\log\left(1-\frac{\alpha_j}{j+1}\right)
    \leq
    -\sum_{j=i}^{t-1} \frac{\alpha_j}{j+1} 
    \leq
    (\beta-1) \log\frac{t+1}{i+1} =
    \log\left(\frac{i+1}{t+1}\right)^{1-\beta}.
    \end{equation}
    The last inequality stems from using twice that $\sum_{i=1}^t 1/t\leq \log(t+1)$.
    Using $(i+1)/i \leq 2$, we have
    \[
    w^2_{i,t} = \left(\frac{A_{i,t}}{i}\right)^2
    \leq 
    \frac{1}{i^2}\left(\frac{i+1}{t+1}\right)^{2-2\beta} \leq
    \left(\frac{2}{t+1}\right)^{2-2\beta}i^{-2\beta}.
    \]
    Using this bound, we have
    \[
    \sum_{i=1}^{t-1} w_{i+1,t}^2\ \le\ \frac{4^{1-\beta}}{(t+1)^{2-2\beta}}\sum_{i=1}^{t-1} (i+1)^{-2\beta}
    \leq_{(*)}
    \frac{4^{1-\beta}}{(t+1)^{2-2\beta}} (t-1)t^{-2\beta} \leq \frac{4^{1-\beta}}{t+1},
    \]
    where in the $(*)$ inequality we are bounding each element of the sum by tha largest one, which is $t^{-2\beta}$, and that there are $t-1$ summands.
    This is valid because, given our definition of $f$ as always non-increasing, $\beta\leq 0$.
\end{proof}

\begin{lemma}
    \label{lem:azuma}
    Let $\delta > 0$.
    Let $\tau=\frac{2^{1+\frac{1}{1-\beta}}}{\delta^{1-\beta}}-1$, and
     $K = (1/32)\cdot 4^{\beta}$.
    For all $t \geq \tau$, we have
    \begin{equation}
        \prob[ \Delta_t > \delta] \leq \exp\left(-K\delta^2 t\right), \qquad
        \prob[ \Delta_t < -\delta] \leq \exp\left(-K\delta^2 t\right).
    \end{equation}
\end{lemma}

\begin{proof}
    The idea is to use the expression in Equation~\ref{eq:unrolled} to bound $\Delta_t$. Of the two summands, $A_{1,t}\Delta_1$ will be bounded by bounding $A_{1,t}$ for large enough $t$.
    For the second summand, we will show it is a martingale and use a concentration inequality to bound its value.
    We show the case for the first inequality. The second inequality follows the same argument, with a symmetric use of the Azuma-Hoeffding's inequality.

    Fix $t\in\NN$ and consider
    $E_{i,t} = \sum_{j=1}^{i} w_{j+1,t}\xi_{j+1}$, for $i < t$.
    When conditioning over the decisions up to time $i$, we have $\EE[\xi_{i+1}\mid \mu_i]=0$, and therefore 
    \[
    \EE[E_{i,t}\mid \mu_i] = E_{i-1,t} + w_{i+1,t}\EE[\xi_{i+1}\mid \mu_i] = E_{i-1,t}. 
    \]
    So $(E_{i,t})_{i=1}^{t-1}$ is a martingale.
    Its increments are bounded by $|E_{i,t} - E_{i-1,t}| = |w_{i+1,t}\xi_{i+1}| \leq |w_{i+1,t}|$.
    Applying Azuma-Hoeffding's inequality to $E_{t-1,t}$ we have, for any $\delta>0$
    \begin{equation}
    \label{eq:azuma_for_M}
        \prob\left[
        E_{t-1,t} > \delta 
        \right] \leq
        \exp\left( -\frac{\delta^2}{2\sum_{i=1}^{t-1} w_{i+1,t}^2} \right) \leq
        \exp\left( -\frac{(t+1)\delta^2}{2\cdot 4^{1-\beta}}\right)
    \end{equation}
    From Equation~\ref{eq:boundA_it}, applied to $i=1$, we  have
    \[
    A_{1,t} \leq \left(\frac{2}{t+1}\right)^{1-\beta}.
    \]
    We can apply the triangle inequality to Equation~\ref{eq:unrolled} to get
    \begin{equation}
    \label{eq:boundA}
    |\Delta_t | \leq |A_{1,t}\Delta_1|+|E_{t-1,t}|.    
    \end{equation}
    So we can apply the bound in Equation~\ref{eq:azuma_for_M} with $\delta/2$ to bound $\prob[\Delta_t > \delta]$ as long as $t$ is large enough to guarantee $|A_{1,t}\Delta_1|\leq \delta/2$.
    In such case
    \begin{equation}
        \prob[\Delta_t >\delta] \leq 
        \prob[E_{t-1,t}>\delta/2] 
     \leq 
        \exp\left( -\frac{(t+1)(\delta/2)^2}{2\cdot 4^{1-\beta}}\right) = 
        \exp\left( -K(t+1)\delta^2\right),
    \end{equation}
    with $K=(8\cdot 4^{1-\beta})^{-1}=(1/32)\cdot 4^{\beta}$.
    For the previous bound to hold, we need $t$ large enough so that $|A_{1,t}\Delta_1|\leq \delta/2$.
    Using $|\Delta_1|\leq 1$ and Equation~\ref{eq:boundA}, we have
    \begin{equation}
        |A_{1,t}\Delta_1|\leq \left(\frac{2}{t+1}\right)^{1-\beta},
    \end{equation}
    which is smaller than $\delta/2$ for $t\geq \frac{2^{1+\frac{1}{1-\beta}}}{\delta^{1-\beta}}-1$.
\end{proof}

\textbf{Theorem \ref{thm:tail_bound}.}
\emph{
    Let $I=[L,U]$ such that $\fairpivot, p\in I$.
    Let $\tau= 4/\min(|L-\fairtarget|, |U-\fairtarget|)$.
    Then for every $t\geq \tau$ we have 
    \begin{equation}
         \prob(M_t\notin I)\ \leq  \exp\left(-Kt|L-\fairtarget|^2\right) + \exp\left(-Kt|U-\fairtarget|^2\right),
    \end{equation}
    where $K$ is a positive constant defined as 
    $K=  (1/32)\cdot 4^\beta$ and $\beta = \sup_{r\in [0,1]} f'(r)$.
}
\begin{proof}
    This is just a matter of unpacking the results from Lemma~\ref{lem:azuma} into the case of having a concrete interval. First note that 
    \[
    \prob[M_t\notin I] = \prob[M_t<L] + \prob[M_t>U] = 
    \prob\Big[\Delta_t < -|L-\fairtarget|\Big] + 
    \prob\Big[\Delta_t > |U-\fairtarget|\Big].
    \]
    We bound each of the summands using Lemma~\ref{lem:azuma}, which is guaranteed to hold as long as 
    $t \geq \frac{2^{1+\frac{1}{1-\beta}}}{\delta^{1-\beta}}-1$.
    Since $f$ is always a decreasing function with $f'(\kappa)=0$, the value of $\beta$ is always $0$, so we have that 
    \[
    \frac{2^{1+\frac{1}{1-\beta}}}{\delta^{1-\beta}}-1 \leq \frac4\delta -1 \leq 4/\delta,
    \]
    so if $t\geq 4/\delta$, Lemma~\ref{lem:azuma} holds. 
    This is exactly the condition that
    $t\geq \tau$, for $\tau= 4/\min(|L-\fairtarget|,|U-\fairtarget|)$.
\end{proof}
\textbf{Corollary \ref{cor:tail_bound}.}
\emph{
    Let $r_-=\exp(-K|L-\fairtarget|^2)$, 
    $r_+ = \exp(-K|U-\fairtarget|^2)$.
    For every $T,T'\in \NN$ such that $\tau \leq T < T'$ we have:
    \begin{equation*}
        \expectedviolation_{\safe}(M_{T:T'}) \leq \sum_{t=T}^{T'}( r_-^t + r_+^t)
          \qquad \mbox{ and }
         \qquad 
         \expectedviolation_{\safe}(M_{T:\infty}) \leq \frac{r_-^{T}}{1-r_-} + \frac{r_+^{T}}{1-r_+}.
    \end{equation*}
    Moreover, this provides the upper bound $\probviolation_{\safe}(M_{T:T'})  \leq \expectedviolation_{\safe}(M_{T:T'}) $ for $T'\in \NN\cup\{\infty\}$ s.t.\ $T'\geq T$.
}
\begin{proof}
    This corollary is just an application of Boole's inequality (also known as union bound) on the results from Theorem~\ref{thm:tail_bound}.
\end{proof}

\subsection{Short-term guarantees: Monotonicity}

To prove Theorem~\ref{thm:monotoniciy}, we will use an inductive argument and the tower property of the expectation.

\begin{definition}[Tower operator]
    \label{def:tower-operator}
    Let $\mathcal U=\{u\colon[0,1]\to\RR\}$ be the space of all measurable functions from $[0,1]$ to $\RR$.
    Let $t\in\NN$, and $\zeta$ be an energy function.
    The \emph{tower operator}
    is the function $T_{\zeta,t}\colon \mathcal U\to\mathcal U$, that takes a function $u$ and returns $T_{\zeta,t}u$ defined as
    \[
    (T_{\zeta,t} u)(y)\;:=\;f_\zeta(y)\,u\!\bigl(y_t^+(y)\bigr)
+\bigl(1-f_\zeta(y)\bigr)\,u\!\bigl(y_t^-(y)\bigr),
    \]
    where 
    $y_t^+(y):=y+\frac{1-y}{t+1}$ and $y_t^-(y):=y-\frac{y}{t+1}$.
    
    The \emph{iterated tower operator} is 
    $T_\zeta^{(t)} = T_{\zeta, 1}\circ T_{\zeta,2} \circ \dots \circ T_{\zeta, t}  $
\end{definition}
\begin{lemma}
\label{lem:properties-tower}
    Let $u,v\colon [0,1]\to \RR$, $y\in [0,1]$, $t\in\NN$, $\zeta$ an energy function and $(y_k)$ be a stochastic process generated following Eq.~\ref{eq:process-update}.
    The following properties hold:
    \begin{itemize}
        \item \emph{Expectation: } 
        $\EE_\zeta[u(y_t)] = \EE_\zeta[T_{\zeta,t-1}u(y_{t-1})]$
        \item \emph{Iterated expectation: }
        $\EE_\zeta[u(y_t)] = \EE\left[T_\zeta^{(t-1)}u(y_1)\right]$.
        \item \emph{Monotonicity: }
        If $u(y)\leq v(y)$, then 
        $(T_{\zeta,t}u)(y) \leq (T_{\zeta,t}v)(y)$.
    \end{itemize}
\end{lemma}
\begin{proof}
    The first property follows from a simple computation of the expectation, using that $\prob[y_{t+1}=y_t^+\mid y_{t}] = f_\zeta(y_t)$. 
    In fact, the tower operator is defined  with the expectation property in mind. 
    Applying the expectation property consecutively leads to the next property
    \[
    \EE_\zeta[u(y_t)] = \EE_\zeta[T_\zeta^{(t-1)}u(y_1)].
    \]
    Monotonicity follows directly from the definition, noting that both $f_\zeta(y^+)$ and $1-f_\zeta(y^-)$ are always non-negative.
\end{proof}

The iterated expectation property is useful because it lets us write the expectation of $u(y_m)$, 
which depends on the distribution of $y_m$ after $m$ steps of the process, 
in terms of the expectation of a function depending only on the much simpler distribution of $y_1$.
In particular, the distribution over  $y_1$ does not depend on $\zeta$.
If we want to prove a result of the form $\EE_{\zeta_1}[u(y_t)] \leq \EE_{\zeta_2}[u(y_t)]$ for some $u, \zeta_1$, and $\zeta_2$, it is equivalent to prove that
$\EE[T_{\zeta_1}^{(t-1)}u(y_1)] \leq \EE[T_{\zeta_2}^{(t-1)}u(y_1)]$, so it suffices to prove that 
\begin{equation}
    \label{eq:aux10}
    T_{\zeta_1}^{(t-1)}u(y) \leq T_{\zeta_2}^{(t-1)}u(y),\qquad  \mbox{for}\quad y\in \{0,1\}.
\end{equation}
This observation lets us go from a local comparison to a global result.
We study two properties on the short term: the probability of violating point fairness, and the expected number of point fairness violations.
For the first one, we can take $u(y) = \indi\{y\notin (L, U\}$.
For the second one, we can take $u(y) = (L - y_k)_+ + (y_k - U)_+$.

\begin{lemma}
    Consider the functions $y_t^+, y_t^-$ as defined in Def.~\ref{def:tower-operator}.
    Then
    \begin{enumerate}
        \item If $y\le \kappa$ (start below):
        \begin{enumerate}
            \item If $y_t^+(y)\le \kappa$ (no overshoot), then
            $|y_t^+(y)-\kappa|<|y_t^-(y)-\kappa|$.
            \item If $y_t^+(y)>\kappa$ (overshoot), then
            \begin{equation}
            \label{eq:aux11}
                |y_t^+(y)-\kappa| \leq \frac{1-\kappa}{t+1}, \qquad 
                |y_t^-(y)-\kappa|\leq \frac{1}{t+1}.
            \end{equation}
        \end{enumerate}
        \item If $y > \kappa$ (start above):
        \begin{enumerate}
            \item If $y_t^-(y)\geq \kappa$ (no overshoot), then
            $|y_t^+(y)-\kappa|>|y_t^-(y)-\kappa|$,
            \item If $y_t^-(y)< \kappa$ (overshoot), then
            \begin{equation}
                \label{eq:aux12}
                |y_t^-(y)-\kappa| \leq \frac{\kappa}{t+1}, \qquad 
                |y_t^+(y)-\kappa|\leq \frac{1}{t+1}.
            \end{equation}
        \end{enumerate}
    \end{enumerate}
\end{lemma}

\begin{proof}
We follow the proof case by case.
\begin{itemize}
    \item\emph{Case 1(a)}. 
    $|y_t^-(y)-\kappa|-|y_t^+(y)-\kappa|= y_t^+(y) - y_t^-(y) = 1/(t+1) > 0$.

    \item\emph{Case 1(b)(left)}.
    \[
    |y_t^+(y)-\kappa| = \frac{1-y}{t+1} - (\kappa-y) = \frac{1-\kappa t - \kappa +yt}{t+1}.
    \]
    Using $y\leq \kappa$ in the previous equation, we can cancel $-\kappa t$ with $\kappa t$, and get the result.

    \item\emph{Case 1(b)(right)}.
    Using $y_t^-(y)\geq 0$, we have 
    $|y_t^-(y)-\kappa| = \kappa - y_t^-(y)  < \kappa$.
    
    \item\emph{Case 2(a)}. 
    $|y_t^+(y)-\kappa|-|y_t^-(y)-\kappa|= y_t^+(y) - y_t^-(y) = 1/(t+1) > 0$.

    \item\emph{Case 2(b)(left)}.
    \[
    |y_t^-(y)-\kappa| = \frac{y}{t+1} - (y-\kappa) = \frac{y- yt - y +\kappa t + \kappa}{t+1}.
    \]
    We can cancel $y$ and $-y$, and using $\kappa\leq y$, we can cancel $yt$ with $-yt$, and get the result.

    \item\emph{Case 2(b)(right)}.
    Simply note that $|y^+_t(y)-\kappa| \leq |y^+_t(y)-y_t^-(y)| = 1/(t+1)$.
\end{itemize}
\end{proof}

\begin{definition}[Unimodal function]
    Let $D\subseteq \mathbb R$ be some domain.
    We say that a function $u\colon D\to\mathbb{R}$ is $\kappa$-unimodal if $\kappa\in D$, it is non-increasing on $\{x\::\: x \leq \kappa\}$, and non-decreasing on $\{x\::\:  \kappa \leq x\}$.
\end{definition}
Note that if a function $u$ is $\kappa$-unimodal, then $u(\kappa)$ is its absolute minimum.

\begin{lemma}
\label{lem:monotonicity}
    Let $\zeta_1\succeq \zeta_2$. 
    Let $T\in\NN$.
    Let $u$ be a $\kappa$-unimodal function, and with $u(y)=0$ for $y\in [a,b]$, 
    where 
    \[
    a\leq \kappa - \frac{1}{T+1}, \qquad 
    b \geq \kappa + \frac{1}{T+1}.
    \]
    Then for all $t\geq T$, we have
    for $y\in \{0,1\}$ that
    \[
    T^{(t)}_{\zeta_1}u(y) \leq
    T^{(t)}_{\zeta_2}u(y).
    \]
\end{lemma}

\begin{proof}
    \emph{Step 1: a pointwise one-step inequality.}
    We first prove that for any $t\in\NN$, $y\in [0,1]$, and any
    $\kappa$-unimodal function $v$ that is null on $[a,b]$, we have 
    \begin{equation}\label{eq:one-step}
    (T_{\zeta_1,t}v)(y)\ \leq\ (T_{\zeta_2,t}v)(y).
    \end{equation}

    Using the definition of the tower operator, we have
    \begin{align}
        (T_{\zeta_1,t}v)(y) - (T_{\zeta_2,t}v)(y) & =
        \Big( f_{\zeta_1}(y)v(y_t^+) + (1-f_{\zeta_1}(y)) v(y_t^-)  \Big) -
        \Big( f_{\zeta_2}(y)v(y_t^+) + (1-f_{\zeta_2}(y)) v(y_t^-)  \Big)\nonumber\\
        & =
        f_{\zeta_1}(y)\Big( v(y_t^+) - v(y_t^-) \Big) -
        f_{\zeta_2}(y)\Big( v(y_t^+) - v(y_t^-) \Big) \nonumber\\
        & = 
        \Big(
        f_{\zeta_1}(y) - f_{\zeta_2}(y)
        \Big)\cdot \Big(
        v(y_t^+) - v(y_t^-)
        \Big).
        \label{eq:aux14}
    \end{align}
    Therefore, to check $(T_{\zeta_1,t}v)(y) -(T_{\zeta_2,t}v)(y)\leq 0$ we just need to prove that both factors of the previous equation are of different sign, or one of them is zero.

    We check both sides of $\kappa$ separatedly. 
    
    \underline{Case $y\leq \kappa$.}
    Since $\zeta_1\succeq \zeta_2$, we have $f_{\zeta_1}(y) \geq f_{\zeta_2}(y)$.
    If we are in the no-overshoot regime $(y^+_t(y)\leq \kappa)$, we have $v$ in its descending mode in the whole interval $[y_t^-(y),y_t^+(y)]$, so $v(y_t^+(y)) \leq v(y_t^-(y))$.
    Therefore  $f_{\zeta_1}(y) - f_{\zeta_2}(y)$ and $v(y_t^+) - v(y_t^-)$  have opposed signs.

    If we are in the overshoot regime $(y^+_t(y)\leq \kappa)$,
    we want to make sure that $v(y_t^+) = v(y_t^-)=0$. 
    Using Eq.~\ref{eq:aux11}, we can guarantee it with 
    \[
    a \leq \kappa - \frac{1}{t+1}, \qquad \mbox{and}\qquad 
    b \geq \kappa + \frac{1-\kappa}{t+1},
    \]
    which is guaranteed for $t\geq T$ in the hypothesis of the lemma.

    \underline{Case $y > \kappa$.}
    This case is analogous, following the same argument as in $y \leq \kappa$, taking into account the switch in signs.

    \emph{Step 2: iterate the one-step inequality.}
     For $r=0,1,\dots,t$ define
\[
\Phi_r := B_1B_2\cdots B_r\,A_{r+1}A_{r+2}\cdots A_t\,u,
\]
with the conventions
\[
A_s := T_{\zeta_1,s},\qquad B_s := T_{\zeta_2,s},\qquad
\Phi_0 = A_1A_2\cdots A_tu = T^{(t)}_{\zeta_1}u,\qquad
\Phi_t = B_1B_2\cdots B_tu = T^{(t)}_{\zeta_2}u.
\]

Our goal is to show the chain of inequalities
\begin{equation}
    \label{eq:aux13}
    \Phi_0(y)\ \le\ \Phi_1(y)\ \le\ \cdots\ \le\ \Phi_t(y),
\qquad\text{for }y\in\{0,1\},
\end{equation}
which implies $T^{(t)}_{\zeta_1}u(y)\le T^{(t)}_{\zeta_2}u(y)$ at the endpoints.

\emph{Telescoping principle.}
For each $r=1,\dots,t$, set
\[
v^{(r)} := A_{r+1}A_{r+2}\cdots A_tu.
\]
Then
\[
\Phi_{r-1}=B_1\cdots B_{r-1}\,(A_r v^{(r)}),\qquad
\Phi_r     =B_1\cdots B_{r-1}\,(B_r v^{(r)}).
\]
Thus, if we can show
\begin{equation}\label{eq:one-step-replacement}
A_r v^{(r)}(y)\ \le\ B_r v^{(r)}(y),\qquad y\in\{0,1\},
\end{equation}
then applying the prefix operator $B_1\cdots B_{r-1}$ (which is order-preserving because of the monotonicity property in Lemmas~\ref{lem:properties-tower}) yields
\[
\Phi_{r-1}(y)\ \le\ \Phi_r(y)\qquad\text{for }y\in\{0,1\}.
\]
Chaining over $r=1,\dots,t$ establishes the desired inequality (Eq.\ref{eq:aux13}).

\emph{Verification of \ref{eq:one-step-replacement}.}
Fix $r$ and write $y^+:=y_r^+(y)$, $y^-:=y_r^-(y)$ for brevity. By the algebraic identity from Step~1 (Eq.\ref{eq:aux14}),
\[
(A_r v^{(r)}-B_r v^{(r)})(y) = (f_{\zeta_1}(y)-f_{\zeta_2}(y))\,(v^{(r)}(y^+)-v^{(r)}(y^-)).
\]

At the endpoints $y\in\{0,1\}$, the sign of $f_{\zeta_1}-f_{\zeta_2}$ is fixed by the assumption $\zeta_1\succeq\zeta_2$,
while the sign of $v^{(r)}(y^+)-v^{(r)}(y^-)$ is determined by the geometry of the successors and the plateau assumption:

- For $y=0$: we have $f_{\zeta_1}(0)\ge f_{\zeta_2}(0)$.  
  If $y^+\le \kappa$ (no overshoot), then $|y^+-\kappa|<|y^--\kappa|$, hence
  $v^{(r)}(y^+)\le v^{(r)}(y^-)$.  
  If $y^+>\kappa$ (overshoot), then by the plateau condition both $y^\pm\in[a,b]$, so $v^{(r)}(y^\pm)=0$.  
  In both cases, $v^{(r)}(y^+)-v^{(r)}(y^-)\le 0$, so the product is $\le 0$.

- For $y=1$: we have $f_{\zeta_1}(1)\le f_{\zeta_2}(1)$.  
  If $y^-\ge \kappa$ (no overshoot), then $|y^--\kappa|<|y^+-\kappa|$, hence
  $v^{(r)}(y^-)\le v^{(r)}(y^+)$, i.e.\ $v^{(r)}(y^+)-v^{(r)}(y^-)\ge 0$.  
  If $y^-<\kappa$ (overshoot), then $y^\pm\in[a,b]$, so $v^{(r)}(y^\pm)=0$.  
  In both cases, $v^{(r)}(y^+)-v^{(r)}(y^-)\ge 0$, so the product is $\le 0$.

Thus $(A_r v^{(r)}-B_r v^{(r)})(y)\le 0$ for $y\in\{0,1\}$, proving \ref{eq:one-step-replacement}, which in turn proves Eq.~\ref{eq:aux13}, as we wanted. 
\end{proof}

\textbf{Theorem \ref{thm:monotoniciy}.}
\emph{Let $\zeta_1\succeq \zeta_2$ be two energy functions, 
    with a common minimum at $\kappa$.
    Let $\safe=[L, U]$. 
    Let $M^{\zeta_1}$ and $M^{\zeta_2}$ be the shielded fairness process 
    generated by enforcing the decision process of $p\in [0,1]$ with $\zeta_1$ and $\zeta_2$, respectively.
    Let $\tau = \lceil\max\{1/|\kappa-L|, 1/|\kappa-U|\}\rceil$,
    for all $T\in\NN$ and $T'\in \NN\cup \{\infty\}$ such that $T<T'$ we have 
    \[
    \expectedviolation_{\safe}(M^{\zeta_1}_{T:T'}) \leq
    \expectedviolation_{\safe}(M^{\zeta_2}_{T:T'}), \qquad
    \mbox{ and } \qquad
        \probviolation_{\safe}(M^{\zeta_1}_{T:T'}) \leq
    \probviolation_{\safe}(M^{\zeta_2}_{T:T'}).
    \]}
\begin{proof}
    This is a direct consequence of Lemma~\ref{lem:monotonicity}. 
    The condition on $\tau$ follows from it, and the condition on the values of either probability or expectation of point fairness violation come from the characterization of both of the safety measures as the expectation with certain functions $u$. 
    In particular, for the first one, we can take $u(y) = \indi\{y\notin (L, U)\}$,
    and for the second one, we can take $u(y) = (L - y)_+ + (y - U)_+$.
    The expectation of the outcome at a certain timestep corresponds to the expectation of the iterated tower operator at the first timestep because of Lemma~\ref{lem:properties-tower}.
\end{proof}

\subsection{Extension to Group Fairness}
For $g\in\{A,B\}$ define the group counts and shielded acceptance sums
\[
N_{g,t} := \sum_{i=1}^t \indi[G_i=g],\qquad
S_{g,t} := \sum_{i=1}^t Z_i\,\indi[G_i=g],
\]
and the empirical group acceptance rates (when $N_{g,t}\ge 1$)
\[
M_{g,t}:=\frac{S_{g,t}}{N_{g,t}},\qquad
M_t := M_{A,t}-M_{B,t}\in[-1,1].
\]
Write $r_A=\PP[G_t=A]$, $r_B=1-r_A$, and $r_{\min}=\min\{r_A,r_B\}$.

\paragraph{Energy-shielded decision rule in the 2-group setting.}
Fix pivot $\kappa\in[-1,1]$ and energy function $\zeta:[-1,1]\to[0,1]$ with minimum at $\kappa$ (same conditions as Def.~\ref{def:energyfunc}, with domain $[-1,1]$).
At time $t+1$ the shield observes $(G_{t+1},X_{t+1})$ and the current fairness $M_t$.

- If $M_t\le \kappa$, the shield attempts to \emph{increase} $M_t$:
  it favors $Z_{t+1}=1$ when $G_{t+1}=A$, and favors $Z_{t+1}=0$ when $G_{t+1}=B$,
  flipping the unfavorable outcome with probability $\zeta(M_t)$.

- If $M_t>\kappa$, the shield attempts to \emph{decrease} $M_t$:
  it favors $Z_{t+1}=0$ when $G_{t+1}=A$, and favors $Z_{t+1}=1$ when $G_{t+1}=B$,
  flipping the unfavorable outcome with probability $\zeta(M_t)$.

\paragraph{Induced conditional biases.}
Define the baseline disparity $d:=p_A-p_B\in(-1,1)$ and the two coefficients
\[
c_- := (1-p_A)+p_B = 1-d,\qquad
c_+ := p_A+(1-p_B) = 1+d.
\]
Conditioned on $M_t=\mu$, the shielded decision has the following conditional acceptance probabilities:
\[
\PP[Z_{t+1}=1\mid M_t=\mu,G_{t+1}=A] =
\begin{cases}
p_A+(1-p_A)\zeta(\mu) & \mu\le \kappa,\\
p_A(1-\zeta(\mu)) & \mu>\kappa,
\end{cases}
\]
\[
\PP[Z_{t+1}=1\mid M_t=\mu,G_{t+1}=B] =
\begin{cases}
p_B(1-\zeta(\mu)) & \mu\le \kappa,\\
p_B+(1-p_B)\zeta(\mu) & \mu>\kappa.
\end{cases}
\]
Consequently, the \emph{one-dimensional drift map} for the fairness difference is
\begin{equation}
\label{eq:def_of_f_2g}
f(\mu):=
\begin{cases}
d + c_-\,\zeta(\mu) & \mu\le \kappa,\\
d - c_+\,\zeta(\mu) & \mu>\kappa.
\end{cases}
\end{equation}

\paragraph{Fairness recursion (random step sizes).}
If $G_{t+1}=A$ then $M_{B,t+1}=M_{B,t}$ and
\[
M_{A,t+1} = M_{A,t} + \frac{1}{N_{A,t+1}}\bigl(Z_{t+1}-M_{A,t}\bigr),
\]
hence
\begin{equation}
\label{eq:update_A}
M_{t+1}=M_t+\frac{1}{N_{A,t+1}}\bigl(Z_{t+1}-M_{A,t}\bigr)\qquad\text{if }G_{t+1}=A.
\end{equation}
If $G_{t+1}=B$ then $M_{A,t+1}=M_{A,t}$ and
\[
M_{B,t+1}=M_{B,t}+\frac{1}{N_{B,t+1}}\bigl(Z_{t+1}-M_{B,t}\bigr),
\]
hence
\begin{equation}
\label{eq:update_B}
M_{t+1}=M_t-\frac{1}{N_{B,t+1}}\bigl(Z_{t+1}-M_{B,t}\bigr)\qquad\text{if }G_{t+1}=B.
\end{equation}

\subsubsection{Long-term guarantees in the 2-group setting}

\begin{lemma}[Unique fixpoint of the 2-group drift map]
\label{lem:unique-fixpoint-2g}
Let $f:[-1,1]\to[-1,1]$ be as in Eq.~\ref{eq:def_of_f_2g}.
Then $f$ is continuously differentiable and has a unique fixpoint $\mu^*\in[-1,1]$ satisfying $f(\mu^*)=\mu^*$.
Moreover, $\mu^*$ lies between $d=p_A-p_B$ and the pivot $\kappa$.
\end{lemma}

\begin{proof}
\emph{Smoothness.}
Away from $\kappa$, $f$ inherits differentiability from $\zeta$.
At $\kappa$, the conditions $\zeta(\kappa)=0$ and $\zeta'(\kappa)=0$ imply
\[
\lim_{\mu\uparrow \kappa} f(\mu) = d + c_-\,\zeta(\kappa)=d = d - c_+\,\zeta(\kappa)=\lim_{\mu\downarrow\kappa} f(\mu),
\]
and similarly for the derivative since both side-derivatives equal $c_-\,\zeta'(\kappa)= -c_+\,\zeta'(\kappa)=0$.
Hence $f$ is continuously differentiable.

\emph{Existence.}
Let $g(\mu):=f(\mu)-\mu$.
We have $g(-1)=f(-1)+1\ge -1+1=0$ and $g(1)=f(1)-1\le 1-1=0$,
so by continuity there exists $\mu^*\in[-1,1]$ with $g(\mu^*)=0$.

\emph{Uniqueness.}
By the energy-function shape, $\zeta$ is non-increasing on $(-1,\kappa]$ and non-decreasing on $[\kappa,1)$.
Therefore $f$ is non-increasing on all of $[-1,1]$:
on $\mu\le\kappa$ we have $f'(\mu)=c_-\,\zeta'(\mu)\le 0$,
and on $\mu>\kappa$ we have $f'(\mu)=-c_+\,\zeta'(\mu)\le 0$.
Hence $g(\mu)=f(\mu)-\mu$ is strictly decreasing wherever $f'(\mu)>-1$, and in any case is decreasing overall.
A decreasing continuous function can have at most one zero, so the fixpoint is unique.

\emph{Placement between $d$ and $\kappa$.}
Assume $d\le \kappa$ (the other case is symmetric).
Then $g(d)=f(d)-d=c_-\,\zeta(d)\ge 0$ and $g(\kappa)=f(\kappa)-\kappa=d-\kappa\le 0$.
By continuity, the unique root $\mu^*$ lies in $[d,\kappa]$.
\end{proof}

\begin{lemma}[Almost sure convergence of $M_t$ to the fixpoint]
\label{lem:convergence-2g}
Let $(M_t)$ be generated by the 2-group energy shield and assume $r_A,r_B\in(0,1)$.
Then $M_t\to \mu^*$ almost surely, where $\mu^*$ is the unique fixpoint from Lemma~\ref{lem:unique-fixpoint-2g}.
Moreover, $(M_t-\mu^*)^2=o(1/t^\lambda)$ for all $\lambda\in(0,1)$ almost surely.
\end{lemma}

\begin{proof}
We follow the same Robbins--Monro / Robbins--Siegmund template as in Lemma~\ref{lem:convergence_outcome_1d}, noting that the only difference is that the step size is random but of order $1/t$.

Define $g(\mu):=f(\mu)-\mu$ and $\Delta_t:=M_t-\mu^*$.
Using the update rules \ref{eq:update_A}--\ref{eq:update_B} we can rewrite
\[
M_{t+1}=M_t+\gamma_{t+1}\bigl(g(M_t)+\xi_{t+1}\bigr),
\]
where
\[
\gamma_{t+1}:=\begin{cases}
1/N_{A,t+1} & \text{if }G_{t+1}=A,\\
1/N_{B,t+1} & \text{if }G_{t+1}=B,
\end{cases}
\qquad
\xi_{t+1}:=\bigl(M_{t+1}-M_t\bigr)/\gamma_{t+1}-g(M_t).
\]
By construction, $\EE[\xi_{t+1}\mid \mathcal F_t]=0$ and $|\xi_{t+1}|\le 2$ (since all involved quantities lie in $[-1,1]$).
Moreover, standard law of large numbers gives $N_{g,t}/t\to r_g$ a.s., hence $\gamma_t\asymp 1/t$ a.s. and
\[
\sum_{t\ge 1}\gamma_t=\infty,\qquad \sum_{t\ge 1}\gamma_t^2<\infty\qquad\text{a.s.}
\]
Since $f$ is decreasing and has a unique fixpoint, $g(\mu)$ has opposite sign to $(\mu-\mu^*)$, implying the same “negative drift in squared error” inequality as in Eq.~\ref{eq:aux4} (with $\gamma_t$ instead of $1/t$ and bounded noise).
Thus Robbins--Siegmund applies and yields $\Delta_t\to 0$ almost surely.
Rates follow by the same argument as in Lemma~\ref{lem:convergence_outcome_1d}, using that $\gamma_t\asymp 1/t$ a.s.
\end{proof}

\paragraph{Convergence to a desired target.}
Exactly as in the 1D case, the process converges to the (unique) fixpoint of $f$.
Therefore, to enforce a desired long-term target $\fairtarget\in[-1,1]$, it is necessary and sufficient to choose $\zeta$ so that $f(\fairtarget)=\fairtarget$.

\noindent\textbf{Theorem~\ref{thm:convergence2g}}
\emph{
Let $\fairtarget\in[-1,1]$ and $d=p_A-p_B$.
The 2-group shielded fairness process $(M_t)$ converges almost surely to $\fairtarget$ if and only if
\begin{equation}
\label{eq:mainconvergenceeq_2g_final}
\zeta(\fairtarget)=
\begin{cases}
(\fairtarget-d)/(1-d) & \text{if } d\leq \fairtarget,\\
(d-\fairtarget)/(1+d) & \text{if } d>\fairtarget.
\end{cases}
\end{equation}
}

\begin{proof}
By Lemma~\ref{lem:convergence-2g}, $M_t$ converges to the unique fixpoint of $f$.
Hence $M_t\to \fairtarget$ a.s. iff $f(\fairtarget)=\fairtarget$.

If $\fairtarget\le \kappa$ then $f(\fairtarget)= d + (1-d)\zeta(\fairtarget)$ (since $c_-=1-d$), giving
\[
\fairtarget=d+(1-d)\zeta(\fairtarget)\iff \zeta(\fairtarget)=\frac{\fairtarget-d}{1-d}.
\]
This corresponds exactly to the case $d<\fairtarget$ (because the fixpoint lies between $d$ and $\kappa$).

If $\fairtarget>\kappa$ then $f(\fairtarget)= d - (1+d)\zeta(\fairtarget)$ (since $c_+=1+d$), giving
\[
\fairtarget=d-(1+d)\zeta(\fairtarget)\iff \zeta(\fairtarget)=\frac{d-\fairtarget}{1+d},
\]
corresponding to $d>\fairtarget$.
The case $d=\fairtarget$ forces $\zeta(\fairtarget)=0$.
\end{proof}

\paragraph{Interference cost.}
Let $\nu_t=\frac1t\sum_{i=1}^t Y_i$ be the average intervention rate.
Conditioned on $M_t=\mu$, an intervention happens iff the raw decision is unfavorable and the shield flips it with probability $\zeta(\mu)$.
Thus define
\begin{equation}
\label{eq:def_of_h_2g}
h(\mu):=
\begin{cases}
r_A(1-p_A)\zeta(\mu)+r_B\,p_B\,\zeta(\mu) & \mu\le \kappa,\\
r_A\,p_A\,\zeta(\mu)+r_B(1-p_B)\zeta(\mu) & \mu>\kappa.
\end{cases}
\end{equation}

\begin{lemma}[2-group analogue of Lemma~\ref{lem:cost}]
\label{lem:cost-2g}
We have $\nu_t\to h(\mu^*)$ almost surely, where $\mu^*$ is the fixpoint of $f$.
In particular, if the shield is tuned so that $\mu^*=\fairtarget$, then $\nu_t\to h(\fairtarget)$ almost surely.
\end{lemma}

\begin{proof}
By definition, $\EE[Y_{t+1}\mid \mathcal F_t]=h(M_t)$ and $h$ is continuous.
Since $M_t\to \mu^*$ a.s. (Lemma~\ref{lem:convergence-2g}), we obtain $h(M_t)\to h(\mu^*)$ a.s.
A standard martingale SLLN argument for bounded increments yields $\nu_t-\frac1t\sum_{i=1}^t h(M_{i-1})\to 0$ a.s.,
hence $\nu_t\to h(\mu^*)$ a.s.
\end{proof}

\noindent\textbf{Lemma~\ref{lem:group-counts}}
\emph{For every $\eta\in(0,1)$, let $\tau_G(\eta):= \frac{8}{r_{\min}}\log\frac{4}{\eta}$. Then for all $T\geq \tau_G(\eta)$, it holds that $\PP(E_T)\ge 1-\eta$.}

\begin{proof}
    In particular, we prove that 
    \[
    \tau_G(\eta):= \frac{8}{r_{\min}}\log\frac{4}{\eta}.
    \]
    For each fixed $t$, Chernoff bounds for binomials give
    \[
    \PP\!\left[N_{A,t}<\tfrac{r_A}{2}t\right]\le \exp\!\left(-\tfrac{r_A}{8}t\right),\qquad
    \PP\!\left[N_{B,t}<\tfrac{r_B}{2}t\right]\le \exp\!\left(-\tfrac{r_B}{8}t\right).
    \]
    Union bound over $t\ge T$ and over both groups yields
    \[
    \PP(E_T^c)\le 2\sum_{t\ge T}\exp\!\left(-\tfrac{r_{\min}}{8}t\right)
    \le 2\cdot \frac{\exp(-\tfrac{r_{\min}}{8}T)}{1-\exp(-\tfrac{r_{\min}}{8})}.
    \]
    Choosing $T=\tau_G(\eta)$ makes the RHS at most $\eta$ (up to constant slack absorbed by the chosen $8\log(4/\eta)$ form).
\end{proof}

\noindent\textbf{Theorem~\ref{thm:tail_bound_2g}}.
\emph{
Fix an interval $I=[L,U]\subseteq[-1,1]$ such that $\kappa,d\in I$ and let $\eta\in(0,1)$.
Let $K=\min\{r_A,r_B\}/32$, then for all $t\geq \tau_G(\eta)$
\begin{equation}
\label{eq:tail_2g_appendix}
\PP(M_t\notin I)\ \le\ \eta\ +\ 
\exp\!\bigl(-K\,t\,|L-\fairtarget|^2\bigr)\ +\ 
\exp\!\bigl(-K\,t\,|U-\fairtarget|^2\bigr).
\end{equation}
}

\begin{proof}
    Define
    \[
    \varepsilon_L:=|\mu^*-L|,\qquad \varepsilon_U:=|U-\mu^*|,
    \]
    and
    \[
    K_2 := \frac{r_{\min}}{32}\cdot 4^{\beta_2}
    \qquad\text{with}\qquad
    \beta_2:=\sup_{\mu\in[-1,1]} f'(\mu)\le 0.
    \]

    On the event $E_\tau$, for all $t\ge \tau$ we have
    \[
    \gamma_{t+1}\le \max\Big\{\frac{2}{r_A(t+1)},\frac{2}{r_B(t+1)}\Big\}\le \frac{2}{r_{\min}(t+1)}.
    \]
    Thus, conditioning on $E_\tau$, the recursion for $\Delta_t$ becomes a Robbins--Monro recursion with deterministic upper bounds on step sizes, and the same “unrolled martingale + Azuma” argument as in Lemmas~\ref{lem:azuma}--\ref{thm:tail_bound} applies verbatim with the substitution
    \[
    \frac{1}{t+1}\ \leadsto\ \frac{2}{r_{\min}(t+1)}.
    \]
    This scales the variance proxy by a factor $4/r_{\min}^2$, yielding an exponential tail with constant $K_2=\Theta(r_{\min})$ (as stated).
    Finally, decondition:
    \[
    \PP(M_t\notin I)\le \PP(E_\tau^c)+\PP(M_t\notin I\mid E_\tau)\le \eta + \text{(Azuma tails)}.
    \]
\end{proof}

\noindent\textbf{Theorem~\ref{thm:monotonicity_2g}}.
\emph{
Let $\zeta_1\succeq \zeta_2$ be two energy functions on $[-1,1]$ with common minimum at $\kappa$.
Fix $I=[L,U]\subseteq[-1,1]$ and let $M^{\zeta_1}$, $M^{\zeta_2}$ be the corresponding 2-group shielded fairness processes.
Then for every $\eta\in(0,1)$, for all $T<T'$ with $T\ge \tau_G(\eta)$,
\[
\expectedviolation_I\!\left(M^{\zeta_1}_{T:T'}\mid E_T\right)\ \le\
\expectedviolation_I\!\left(M^{\zeta_2}_{T:T'}\mid E_T\right), \quad \mbox{and}\quad
\probviolation_I\!\left(M^{\zeta_1}_{T:T'}\mid E_T\right)\ \le\
\probviolation_I\!\left(M^{\zeta_2}_{T:T'}\mid E_T\right).
\]
}

\begin{proof}
    Condition on $E_T$ so that for all $t\ge T$, both denominators satisfy
    $N_{g,t}\ge (r_g/2)t$, hence the one-step updates \ref{eq:update_A}--\ref{eq:update_B}
    move the state by at most $O(1/t)$.
    On each side of $\kappa$, the drift map $f$ from Eq.~\ref{eq:def_of_f_2g} is affine in $\zeta$ with a \emph{nonnegative} coefficient:
    for $\mu\le \kappa$, $f(\mu)=d+(1-d)\zeta(\mu)$ is increasing in $\zeta(\mu)$,
    and for $\mu>\kappa$, $f(\mu)=d-(1+d)\zeta(\mu)$ is decreasing in $\zeta(\mu)$.
    Thus, exactly as in Eq.~\ref{eq:aux14}, the one-step tower difference factors as
    \[
    (T_{\zeta_1,t}v - T_{\zeta_2,t}v)(\mu) = \bigl(f_{\zeta_1}(\mu)-f_{\zeta_2}(\mu)\bigr)\cdot \bigl(v(\mu^+)-v(\mu^-)\bigr),
    \]
    where $\mu^\pm$ denote the two possible next fairness values (accept/reject), and where the sign of $f_{\zeta_1}-f_{\zeta_2}$
    is fixed by $\zeta_1\succeq \zeta_2$ and the side of $\kappa$.
    For any $\kappa$-unimodal test function $v$ that is $0$ on an interval containing $\kappa\pm O(1/t)$ (the same plateau condition as in Lemma~\ref{lem:monotonicity}),
    the term $v(\mu^+)-v(\mu^-)$ has the opposite sign, implying $(T_{\zeta_1,t}v)(\mu)\le (T_{\zeta_2,t}v)(\mu)$.
    Iterating this inequality exactly as in Lemma~\ref{lem:monotonicity} yields the claimed ordering of expectations for
    the indicator test function defining $\probviolation$ and for the additive test function defining $\expectedviolation$.
\end{proof}

\subsection{Shield Synthesis}
\label{sec:shield-synthesis-appendix}

\paragraph{Markov chain construction.}
We construct the acyclic Markov chain $\mathcal{M}=(S,P,s_0)$ consisting of a set of states $S$, a Markov transition kernel $P\colon S\times S \to [0,1]$, and an initial state $s_0=(1,0)$. Intuitively, the set of states is stratified into time steps $S= \bigcup_{t\in [0;T]} S_t$ where $S_t=\{(t, m) \mid m\in [t]\}$ for all $t\in [1;T]$. Intuitively, a state $s=(t, c)\in S$ is labeled with a point in time $t$, the positive decision counter $c$ clearly upper bound by the point in time $t$. 
The Markov transition kernel $P$ is defined as follows:
for every state $s=(t,c)$ with $t\in [T-1]$ we transition to state $s= (t+1,c+1)$ with probability $f(c/t)$,
and to state $s'=(t+1,c)$ with probability $1-f(c/t)$.

\paragraph{Dynamic programming.}
We compute the value of both the probabilistic and expected violation measure for the interval $[1;T]$ using dynamic programming on the Markov chain $(S,P,s_0)$. Let $\gamma(s)=\indi[c/t\notin \safe]$, then for every $s=(t,c)\in S$ we define its value $V_{\expectedviolation}$ for the expected violation measure as
\begin{align*}
   V_{\expectedviolation}(s) &=
  \gamma(s)+ (f(c/t)  V_{\expectedviolation}(t+1, c+1 ) + (1-f(c/t)) V_{\expectedviolation}(t+1, c)) &\quad \text{if $t\in [T-1]$}\\
   V_{\expectedviolation}(s) &=
  \gamma(s) &\quad \text{if $t=T$}
\end{align*}
Moreover, we define its value $V_{\probviolation}$ for the expected violation measure as
\begin{align*}
   V_{\probviolation}(s) &= \max\Big(\gamma(s), f(c/t)  V_{\probviolation}(t+1, c+1 ) + (1-f(c/t)) V_{\probviolation}(t+1, c)\Big) &\quad \text{if $t\in [T-1]$}\\
   V_{\probviolation}(s) &=
  \gamma(s) &\quad \text{if $t=T$}
\end{align*}
Notice that if $\gamma(c,t)=1$ no further computation recursion is required.

\paragraph{Intuition.}
The value function $V_{\expectedviolation}(s)$ represents the expected number of violations incurred from state $s=(t,c)$ onward, including a possible violation at time $t$. Formally, our process $M$ takes on the value $M_t=c/t$ at time $t$, then  
\[
V_{\expectedviolation}(t,c) \;=\; \indi[M_t \notin \safe] + \expe\!\left[\sum_{i=t+1}^T \indi[M_i \notin \safe] \;\middle|\; M_t\right].
\]  
Hence, the value at the initial state $s_0=(0,0)$ equals the total expected number of violations up to time $T$:  
\[
V_{\expectedviolation}(s_0) \;=\; \expe\!\left[\sum_{i=1}^T \indi[M_i \notin \rho(i)] \right].
\]  

The value function $V_{\probviolation}(s)$ instead captures the probability of encountering at least one violation from state $s$ onward. This probability equals $1$ if $M_t \notin \safe$, and otherwise it coincides with the probability of observing a violation at some later time:  
\[
V_{\probviolation}(t,c) \;=\; \mathbb{P}\!\left[\max_{i \in [t;T]} \indi[M_i \notin \safe] = 1 \;\middle|\; M_t\right].
\]  
In particular, $V_{\probviolation}(s_0)$ gives the probability of observing any violation within the time horizon $[1,T]$.

\section{Families of energy functions of interest}
\label{sec:families_of_energy_funcs}
In this section we propose two families of general-purpose energy shields, and one family of energy shields specifically designed to fit a given specification.
Note that some of these functions have isolated non-smooth points, where the function is continuous but not differentiable. 
These do not pose a theoretical issue, since we can always ``glue'' the non-smooth endpoints of the two smooth pieces using infinitely differentiable bump functions\footnote{This follows the same rationale as why most theoretical results in convergence of machine learning algorithms work with the non-smooth ReLU activation function.}.

\subsection{Polynomial}

\[
\zeta^{Pol}_{\kappa, \alpha,\beta} (x) = \alpha |x-\kappa|^{\beta},
\]
where $\beta\in (1, \infty)$, $\kappa\in (0,1)$ and $\alpha\in \left(0,\frac{1}{\max\{\kappa, 1-\kappa\}^\beta}\right)$.
In this family, $\kappa$ marks the pivoting point, and $\alpha$ and $\beta$ control the shape, with larger $\alpha$ and $\beta$ producing steeper functions.

\subsection{Exponential}
\[
\zeta^{Exp}_{\kappa, \alpha,\beta}(x) = 
\alpha\left( 1-e^{-\beta (x-\kappa)^2} \right),
\]
where $\beta\in (0,\infty)$, $\kappa\in (0,1)$, and $\alpha\in \left(0, \frac{1}{1-e^{-\beta(\min\{\kappa, 1-\kappa\})^2}}\right)$.

\subsection{Construction of a monotonic family of energy functions}
Let $p\in (0,1)$, $\fair=(\tau, \safe, \live)$ be a specification, with $\safe = [L_\safe, U_\safe]$ and $\live = [L_\live, U_\live]$, $\live\subset \safe$.
We give the construction for $\zeta^{\mathrm{Mon}}_{r; p, \safe,\live}$.
For ease of notation, within this section we will call it just $\zeta_r$.
The family of monotonic function is built differently depending on the relative position of $p$ and $\live$.

\begin{itemize}
    \item \textbf{If $p < L_\live$.}
The family of energy functions is $(\zeta_r)_{r\in R}$, with $R= (0,1)$, and is built as follows. Let $\kappa=(U_\live+U_\safe)/2$, 
$a_r = (1-r)L_\live+rU_\live$, $C_r = (a_r-p)/(1-p)$, $\alpha_r = (1-r)/r$.
\begin{equation}
    \zeta_r(x) = \begin{cases}
        \zeta_r^1(x) & \mbox{ if } x < a_r, \\
        \zeta_r^2(x) & \mbox{ if } a_r \leq x \leq  \kappa, \\
        \zeta_r^3(x) & \mbox{ otherwise. }
    \end{cases}
\end{equation}
With the following definitions:
\begin{align*}
\zeta_r^1(x)&=C_r+(1-C_r)\Bigl(1-e^{\tfrac{x-a_r}{\alpha_r}}\Bigr),\\
\zeta_r^2(x)&=C_r\left(1-\frac{x-a_r}{\kappa-a_r}\right)^{\alpha_r},\\
\zeta_r^3(x)&=1-\exp\!\left(-\Bigl(\tfrac{x-\kappa}{\alpha_r}\Bigr)^{m}\right).
\end{align*}
where $m$ is a fixed integer $m\geq 2$, we choose $m = 2$.

\item \textbf{If $p > U_\live$.} 
We follow a symmetric construction as the previous case.
The family of energy functions is $(\zeta_r)_{r\in R}$, with $R= (0,1)$, and is built as follows. Let $\kappa=(L_\safe+L_\live)/2$, 
$a_r = rL_\live+(1-r)U_\live$, $C_r = (p-a_r)/p$, $\alpha_r = (1-r)/r$.
\begin{equation}
    \zeta_r(x) = \begin{cases}
        \zeta_r^1(x) & \mbox{ if } x < \kappa, \\
        \zeta_r^2(x) & \mbox{ if }  \kappa \leq x \leq  a_r, \\
        \zeta_r^3(x) & \mbox{ otherwise. }
    \end{cases}
\end{equation}
With the following definitions:
\begin{align*}
    \zeta_r^1(x) &= 1-e^{-\left(\tfrac{x-\kappa}{\alpha_r}\right)^{m}} \\
    \zeta_r^2(x)&=C_r\left(1-\frac{a_r-x}{a_r-\kappa}\right)^{\alpha_r},\\
    \zeta_r^3(x)&=C_r+(1-C_r)\Bigl(1-e^{\tfrac{a_r-x}{\alpha_r}}\Bigr).
\end{align*}

\item \textbf{If $p\in [L_\live,U_\live]$.} In this case, the natural bias of the process already aligns with the short-term requirements, so we can choose $\kappa = p$ and use a family of either polynomial or exponential functions. 
We show here a family of modified polynomial functions that satisfy the monotonicity requirements.
The family of energy functions is $(\zeta_r)_{r\in R}$, with $R=(0,1)$, built as follows. 
Let $\alpha_r = r/(1-r)$, and $m\geq 2$ fixed, 
$l_r = \kappa - \frac{1}{\alpha_r^{1/m}}$, $u_r = \kappa +\frac{1}{\alpha_r^{1/m}}$.
Then
\[
\zeta_r(x) = \begin{cases}
    \alpha_r|x-\kappa|^{m} & \mbox{ if } x\in [l_r,u_r] \\
    1 & \mbox{ otherwise. }
\end{cases}
\]
\end{itemize}

\begin{figure}[b]
     \centering
     \begin{subfigure}[b]{0.3\textwidth}
         \centering
         \includegraphics[width=\textwidth]{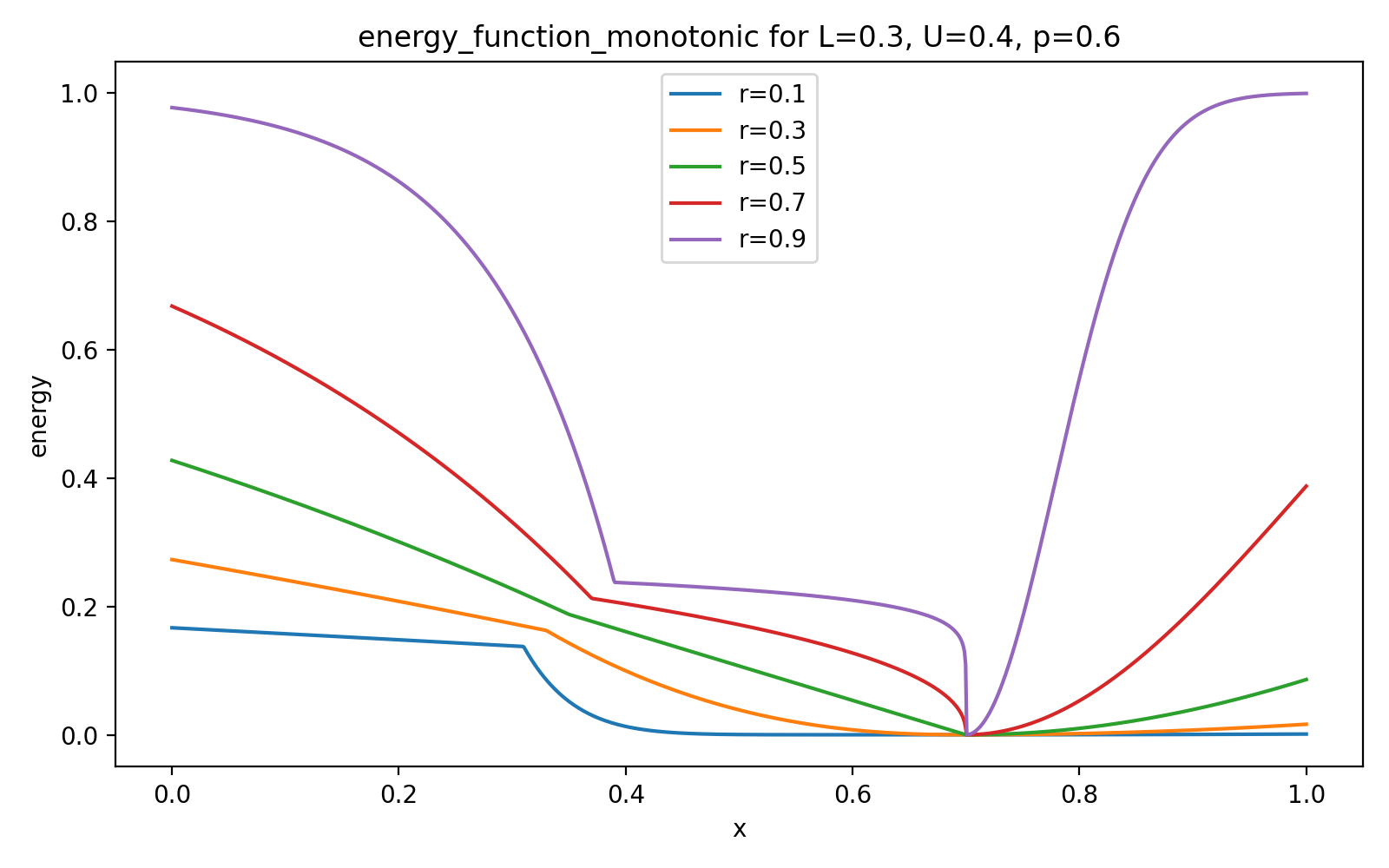}
     \end{subfigure}
     \hfill
     \begin{subfigure}[b]{0.3\textwidth}
         \centering
         \includegraphics[width=\textwidth]{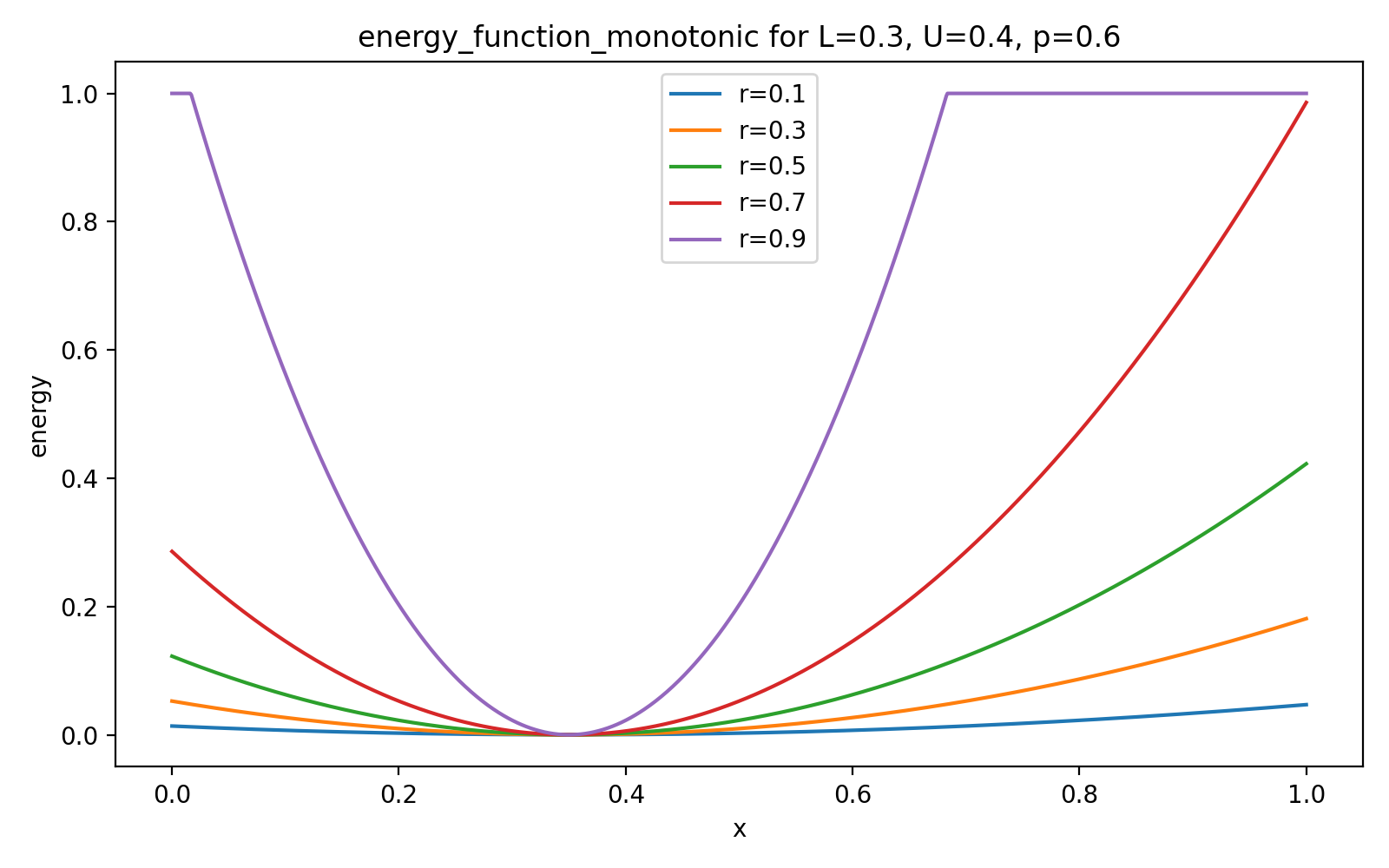}
     \end{subfigure}
     \hfill
     \begin{subfigure}[b]{0.3\textwidth}
         \centering
         \includegraphics[width=\textwidth]{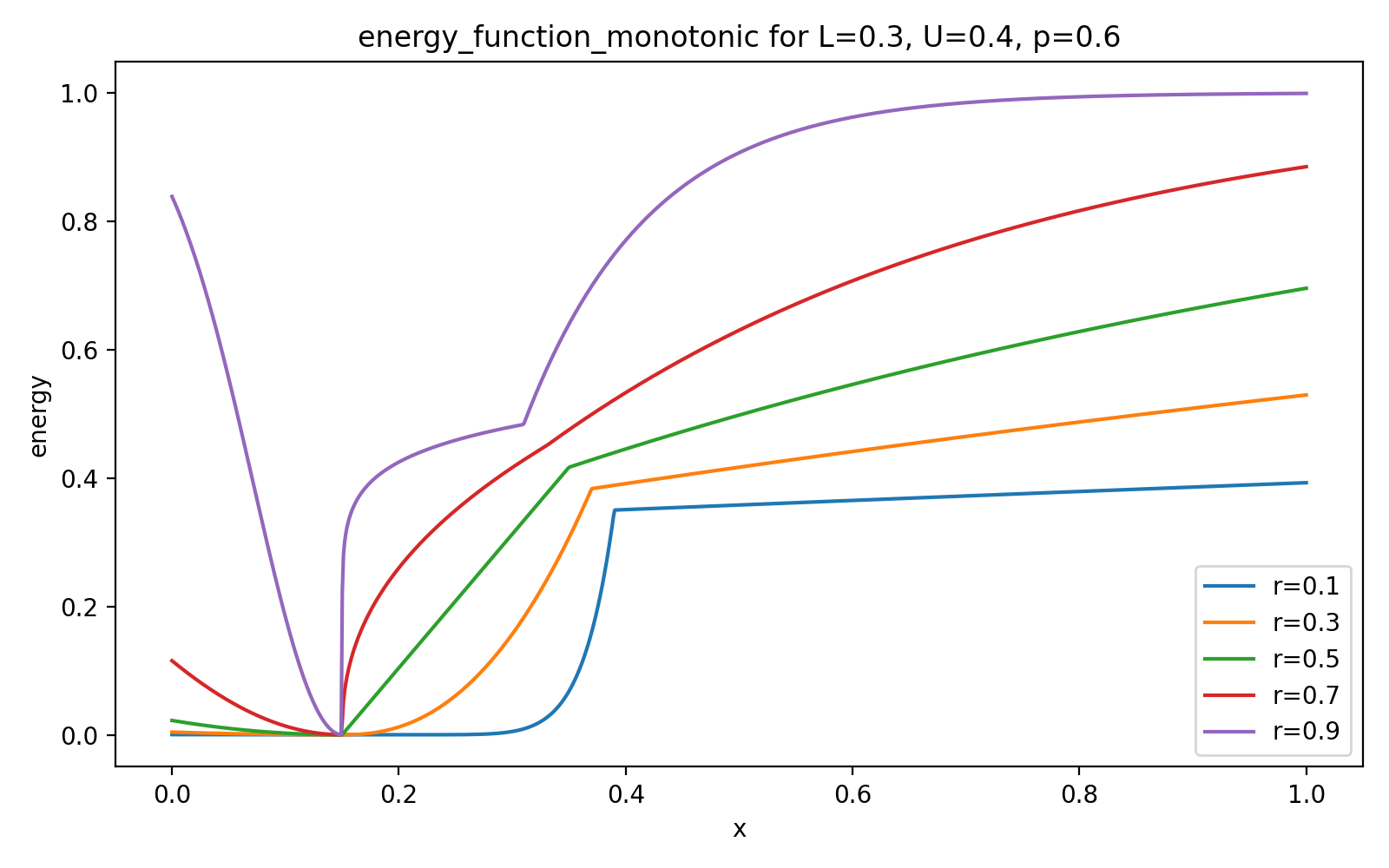}
     \end{subfigure}
        \caption{Monotonic families of functions}
        \label{fig:monotonic_example}
\end{figure}

\begin{claim}
    The family of functions $(\zeta_r)_{r\in R}$ satisfies the following properties:
    \begin{itemize}
        \item It is monotonic in $r$,  i.e., for each $x\in [0,1]$ and each pair $s,r\in (0,1)$, if $s\leq r$, then $\zeta_s(x)\leq \zeta_r(x)$.
        \item The corresponding characteristic function $f_r$ (as defined in Eq.~\ref{eq:def_of_f}) has a fixpoint at $\mu = a_r$.
        \item If $p < L_\live$, for all $x\neq \kappa$:
        \[\lim_{r\to 0} \zeta_r(x) = \begin{cases}
            (L_\live-p)/(1-p) & \mbox{ if } x \leq L_\live \\
            0 & \mbox{ otherwise. }
        \end{cases},
        \quad 
        \lim_{r\to 1} \zeta_r(x) = \begin{cases}
            (L_\live-p)/(1-p) & \mbox{ if } x\in [U_\live, \kappa] \\
            1 & \mbox{ otherwise. }
        \end{cases}
        \]       
        \item If $p > U_\live$, for all $x\neq \kappa$:
        \[\lim_{r\to 0} \zeta_r(x) = \begin{cases}
            (p-L_\live)/p & \mbox{ if } x \geq U_\live \\
            0 & \mbox{ otherwise. }
        \end{cases},
        \quad 
        \lim_{r\to 1} \zeta_r(x) = \begin{cases}
            (p-L_\live)/p & \mbox{ if } x\in [\kappa, L_\live] \\
            1 & \mbox{ otherwise. }
        \end{cases}
        \]       
        \item If $p\in \live$, for all $x\neq \kappa$ and all $l\in \{0,1\}$, 
        $\lim_{r\to l}\zeta_r(x) = l$.
    \end{itemize}
\end{claim}
All the properties are satisfied by construction.
Fig.~\ref{fig:monotonic_example} illustrates them with $m=2$.

\end{document}